%% file: v7_vcglearn.tex
\documentclass[twoside,11pt]{article}

\usepackage{jmlr2e}

\usepackage{hyperref}       
\usepackage{url}            
\usepackage{booktabs}       
\usepackage{amsfonts}       
\usepackage{nicefrac}       
\usepackage{microtype}      
\usepackage{times}
\usepackage{calc}
\usepackage{mathtools}
\usepackage{color}
\usepackage{graphicx} 

\usepackage{DejaVuSans}
\usepackage{parskip}
\usepackage{epsfig}
\usepackage{subfig}
\usepackage{verbatim}
\usepackage{xcolor}
\usepackage{colortbl}
\usepackage{listings}
\usepackage{algorithm}
\usepackage{algorithmic}
\usepackage{booktabs}
\usepackage{enumitem}
\usepackage{bbm}
\usepackage{multirow}
\usepackage{pifont}
\usepackage{wrapfig}
\usepackage{longtable}
\usepackage{eqparbox}
\usepackage{xspace}
\usepackage{kk}
\usepackage{vlDefns}

\newcommand{\toworkon}[1]{}

\jmlrheading{-}{2020}{-}{-}{-}{-}{
  Kirthevasan Kandasamy,
  Joseph E Gonzalez,
  Michael I Jordan,
  Ion Stoica
}

\ShortHeadings{VCG Mechanism Design under Stochastic Bandit Feedback}
{Kandasamy, Gonzalez, Jordan, Stoica}
\firstpageno{1}

\begin{document}

\title{VCG Mechanism Design with Unknown Agent Values \\ under Stochastic Bandit Feedback}

\author{\name Kirthevasan Kandasamy \email kandasamy@eecs.berkeley.edu \\
       \name Joseph E Gonzalez \email jegonzal@eecs.berkeley.edu \\
       \name Michael I Jordan \email jordan@cs.berkeley.edu \\
       \name Ion Stoica \email stoica@cs.berkeley.edu  \\
       \addr 
       University of California, Berkeley,
       CA 94723, USA \\
      }

\editor{-}

\maketitle

\input{tables}

\input{figures}

\input{abstract}

\input{intro}

\input{related}

\input{setup}

\input{lowerbound}

\input{vcglearn}

\input{analysis}

\input{experiments}

\input{lb_app}

\input{proofs_vcglearn}

\input{conclusion}

\paragraph{Acknowledgements:}
We would like to thank Matthew Wright for providing feedback on an initial draft of this manuscript.

{
\bibliography{kky,bib_vcglearn}
}

\end{document}

%% file: tables.tex
\newcommand{\insertAlgoVCGLearnTwo}{

\begin{algorithm}[t]
\vspace{0.02in}

\begin{algorithmic}[1]
\REQUIRE $\estmethod\in\{\estee,\estopt\}$, $\pricemethod\in\{\priceage,\pricesel\}$.
\COMMENT{$\estmethod, \pricemethod$ are hyperparameters. See lines~\ref{lin:pricing}
and~\ref{lin:estmethodif}.}
\STATE Collect bids $\{\bidi\}_i$ from agents participating by bids.
\STATE $t\leftarrow0$. \COMMENT{$t$ indexes rounds}
\label{lin:vcgrequire}
\FOR[$q$ indexes brackets]{$q=1, 2, \dots$,}
    \STATE \emph{Explore phase:}
            Assign each allocation $s\in\Scal$ to each agents at least once  and
           charge them price $0$ on each round. Collect reported rewards $\{\Yit\}_{i,t}$ from agents
           participating by rewards. \hspace{-0.14in}
    \STATE $t\leftarrow t+K$.
            \COMMENT{$K$ is the number of rounds required for the explore phase. See
Sec.~\ref{sec:setup}}
    \FOR{$r = 1,\dots, \lfloor \frac{5}{6} K q^{1/2}\rfloor$,}
        \STATE  $t\leftarrow t+1$.
        \STATE Choose outcome $\alloct \leftarrow \argmax_{\alloc} \Vubt(\alloc)$. \label{lin:alloct}
                \COMMENT{$\Vubt$ is a UCB on the welfare.
        See~\eqref{eqn:learnterms},\eqref{eqn:bidterms},\eqref{eqn:Vubt}}
        \STATE Charge each agent $i$, $\;\priceit \leftarrow \max_\alloc\Fit(\alloc) - \Git(\alloct)$.
            \label{lin:pricing}
                \COMMENT{$\pricemethod$ determines computation of
$\Fit$ and $\Git$ respectively. See~\eqref{eqn:fgterms}}
    \IF{$\estmethod=\estopt$} \label{lin:estmethodif}
        \STATE Collect reported reward $\{\Yit\}_i$ 
        from each agent $i$ participating by rewards.
    \ENDIF
    \ENDFOR
\ENDFOR
\end{algorithmic}
\caption{$\;$\vcglearn \label{alg:vcglearn}}
\end{algorithm}
%
}

\newcommand{\insertAlgoVCGLearn}{
\begin{algorithm}[t]
\vspace{0.02in}

\begin{algorithmic}[1]
\REQUIRE $\estmethod\in\{\estee,\estopt\}$, $f\in\{\vlb,\vbar,\vub\}$, $g\in\{\vlb,\vbar,\vub\}$,
    Bids $\bidi$ for agents participating by bids.
\label{lin:vcgrequire}
\STATE $t\leftarrow0$.
\FOR[brackets]{$q=1, 2, \dots$,}
    \STATE Assign all outcomes $s\in\Scal$ to all agents at least once over $K$ rounds and
           charge them price $0$.
           Collect one realised value per allocation from each agent
           participating by realisations.
            \COMMENT{explore-phase}
    \STATE $t\leftarrow t+K$.
    \FOR{$r = 1,\dots, \lfloor \frac{5}{6} K q^{1/2}\rfloor$,}
        \STATE  $t\leftarrow t+1$.
        \STATE $\alloct \leftarrow \argmax_{\alloc} \Vubt(\alloc)$. \label{lin:alloct}
                \COMMENT{See~\eqref{eqn:learnterms},\eqref{eqn:bidterms},\eqref{eqn:Vubt}}
        \STATE $\forall\,i,\;\,\priceit \leftarrow \max_\alloc\Fmit(\alloc) - \Gmit(\alloct)$.
            \label{lin:pricing}
                \COMMENT{See~\eqref{eqn:fgterms}}
        \STATE Set allocation $\alloct$ and charge the agents $\{\priceit\}_i$.
        \STATE Collect realised values $\{\Yit\}_i$ from the agents.
    \ENDFOR
\ENDFOR
\end{algorithmic}
\caption{$\;$\vcglearn \label{alg:vcglearn}}
\end{algorithm}
}

\newcommand{\insertAlgoVCGUCB}{
\begin{algorithm}[t]
\vspace{0.02in}

\begin{algorithmic}[1]
\REQUIRE $f\in\{\vub,\vlb\}$, Bids $\bidi$ for agents playing by bids.
\STATE $t\leftarrow0$.
\FOR[brackets]{$q=1, 2, \dots$,}
    \STATE Assign all $s\in\Scal$ to all agents over $K$ rounds and
           charge them $0$. Collect realised values $\{\Yit\}_{i,t}$.
    \STATE $t\leftarrow t+K$.
    \FOR[UCB-rounds]{$r = 1,\dots, \lfloor \frac{5}{6} K q^{1/2}\rfloor$,}
        \STATE  $t\leftarrow t+1$.
        \STATE Choose $\alloct \leftarrow \argmax_{\alloc} \Vubt(\alloc)$.
                \COMMENT{See~\eqref{eqn:ucbterms},~\eqref{eqn:welfaredefns}}
        \STATE Charge agent $i$, $\;\priceit \leftarrow \max_\alloc\Fit(\alloc) - \Vubit(\alloct)$.
        \STATE Collect realised values $\{\Yit\}_i$.
    \ENDFOR
\ENDFOR
\end{algorithmic}
\caption{$\;$\vcgucb \label{alg:vcgucb}}
\end{algorithm}
}

\newcommand{\insertTableResultsSummary}{
\newcommand{\insertbetweentablegaps}{\multicolumn{4}{l}{} \\[-0.05in]}
\begin{table}[t]
\centering
\begin{tabular}{l|c|c|c}
\toprule
 &\begin{tabular}{@{}c@{}} Truthfulness \\ $\EE[\UpiiT - \UiT] \in \bigOtilde(\textrm{?})$
    \\[0.05in] (Theorem~\ref{thm:truthfulness})
\end{tabular} &
 \begin{tabular}{@{}c@{}} Individual rationality \\ $-\EE[\UiT] \in \bigOtilde(\textrm{?}) $
    \\[0.05in] (Theorem~\ref{thm:ir})
    \end{tabular} &
 \begin{tabular}{@{}c@{}} VCG regret \\ $\EE[\RmaxT] \in \bigOtilde(\textrm{?}) $
    \\[0.05in] (Theorem~\ref{thm:value})
    \end{tabular} 
    \\
\midrule
$\estmethod=\estee$ &
 $\KotTtt$ &
 $\nstarone(\pricemethod) \KotTtt$ &
 $n^2\KotTtt$
 \\
\midrule
$\estmethod=\estopt$ &
 $n \KotTtt$, NIC &
 $\nstartwo(\pricemethod) \sqrtST$ &
 \;$n^2\KotTtt$\;
 \\
\bottomrule
\end{tabular}
\caption{%
Summary of asymptotic rates for Algorithm~\ref{alg:vcglearn}
for truthfulness, individual rationality, and the VCG regret $\RmaxT$~\eqref{eqn:regret}.
Above, $n$ denotes the number of agents, $T$ the number of rounds,
$|\Scal|$ the number of different allocations, $K$ is a problem-specific parameter
defined in Section~\ref{sec:setup},
$\UiT$ is the total utility of agent $i$ after $T$ rounds when being truthful,
and
$\UpiiT$ is the total utility of agent $i$ after $T$ rounds when following any other (non-truthful)
strategy  $\pi$.
Algorithm~\ref{alg:vcglearn} has two (binary) hyperparameters
$\estmethod\in\{\estee,\estopt\}$, and $\pricemethod\in\{\priceage,\pricesel\}$,
defined in Section~\ref{sec:vcglearn}.
While 
the \valueregret{} and asymptotic truthfulness rates do not depend on $\pricemethod$,
the rates for asymptotic individual rationality do, via the $\nstarone$ term,
where, we have $\nstarone(\priceage) = 1$ and $\nstarone(\pricesel) = n$.
NIC indicates that the mechanism is asymptotically Nash incentive-compatible.
When $\estmethod=\estee$, we have better truthfulness guarantees, but weak rates for individual
rationality, and vice versa when $\estmethod=\estopt$.
\label{tb:resultssummary}
}
\end{table}
}

%% file: figures.tex
\newcommand{\imtextspace}{-0.05in}
\newcommand{\imfourwidth}{1.3in}
\newcommand{\imfourhspace}{0.1in}
\newcommand{\imthreewidth}{1.7in}
\newcommand{\imthreehspace}{0.1in}

\newcommand{\imarrwtwo}{2.88in}
\newcommand{\imhsptwo}{-0.00in}
\newcommand{\imarrwthree}{1.99in}
\newcommand{\imhspthree}{-0.03in}
\newcommand{\imleftspace}{-0.10in}
\newcommand{\imsinglecol}{2.495in}
\newcommand{\imcaptionspace}{-0.05in}

\newcommand{\dflablationfigwidth}{2.14in}
\newcommand{\dflablationfigspace}{-0.25in}
\newcommand{\insertFigSimulation}{
\begin{figure}
\centering
\hspace{-0.15in}%
  \includegraphics[width=\dflablationfigwidth]{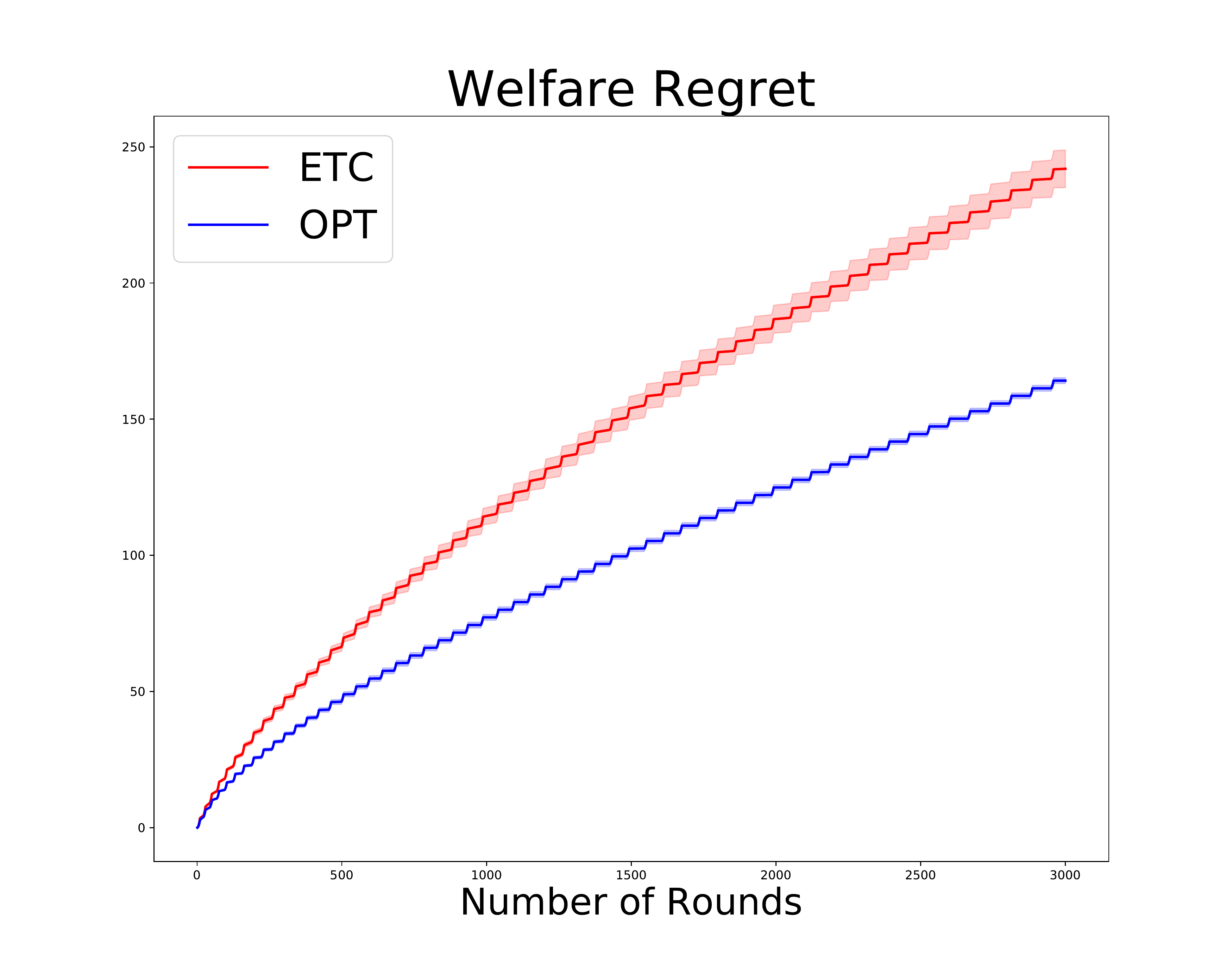}
\hspace{\dflablationfigspace}
  \includegraphics[width=\dflablationfigwidth]{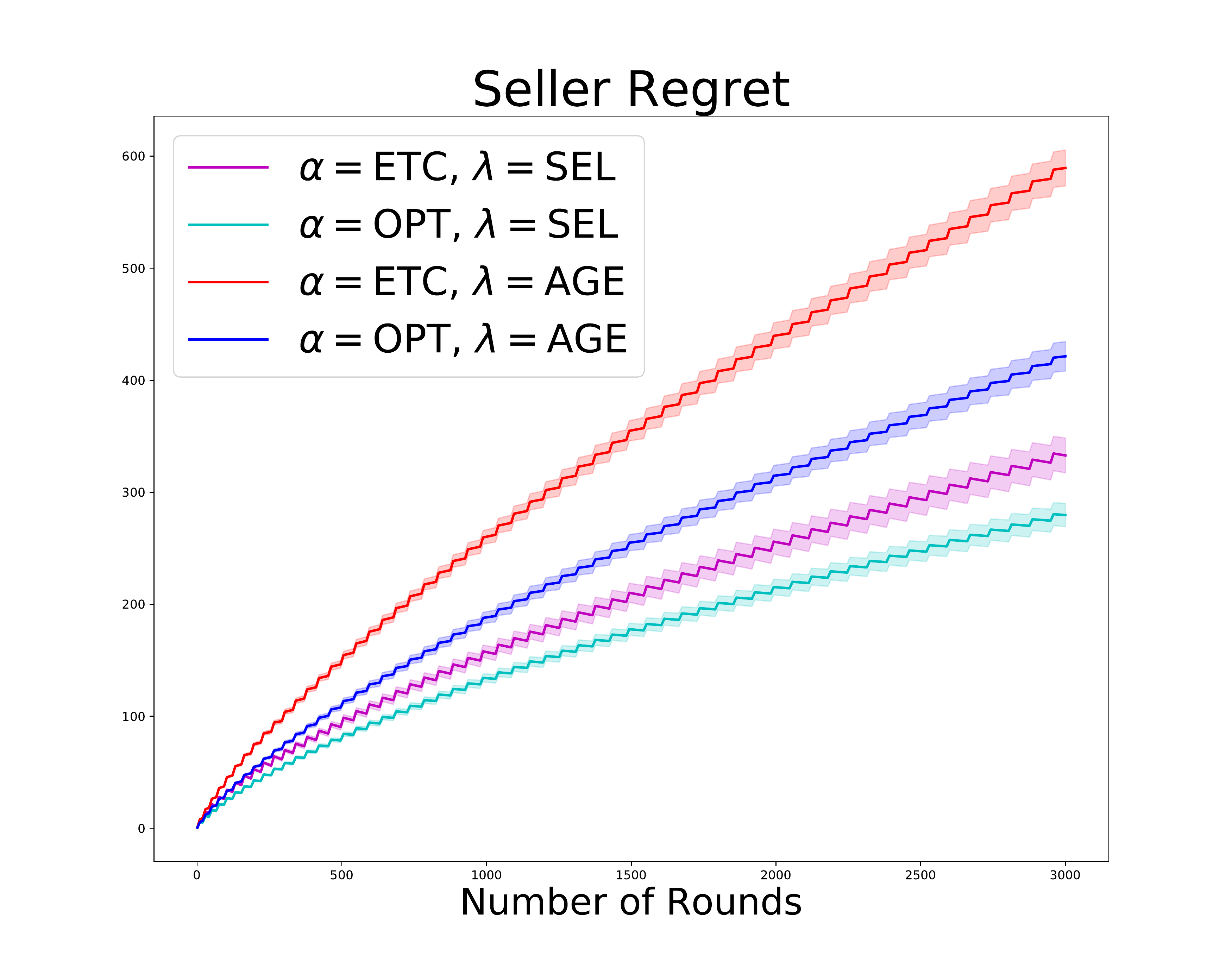}
\hspace{\dflablationfigspace}
  \includegraphics[width=\dflablationfigwidth]{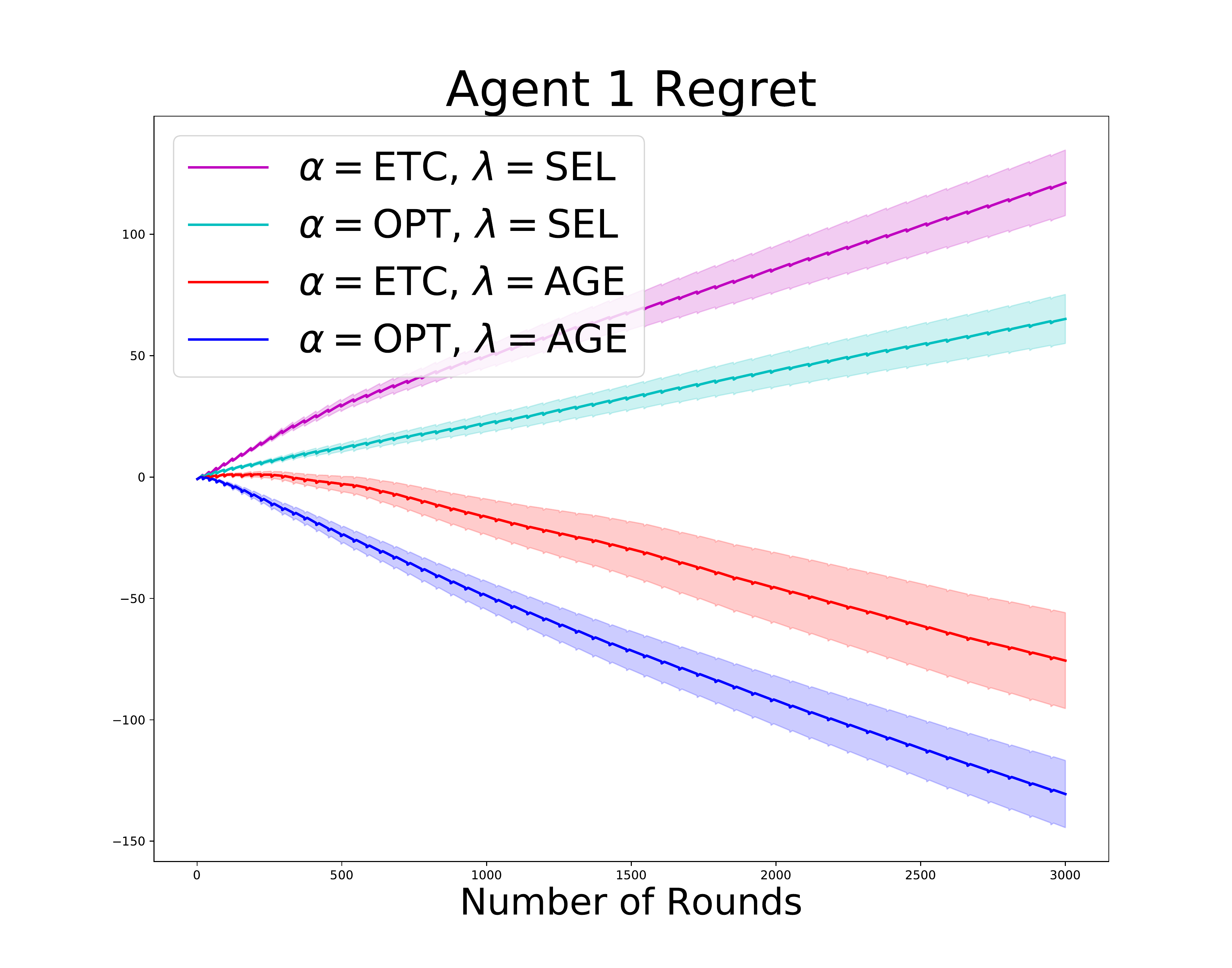}
\hspace{\dflablationfigspace}
\\%
\hspace{-0.1in}%
\hspace{\dflablationfigspace}%
\includegraphics[width=\dflablationfigwidth]{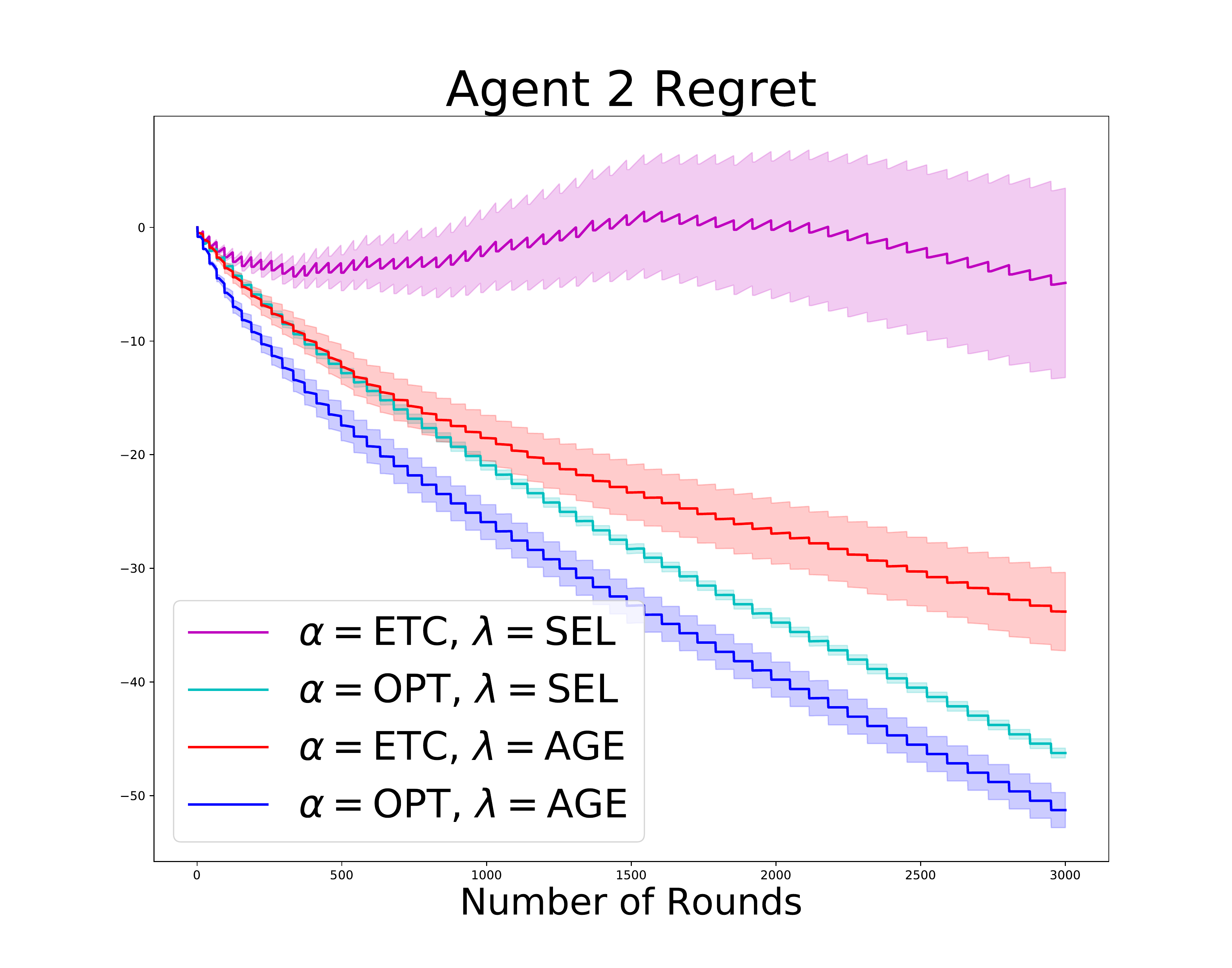}
\hspace{\dflablationfigspace}
  \includegraphics[width=\dflablationfigwidth]{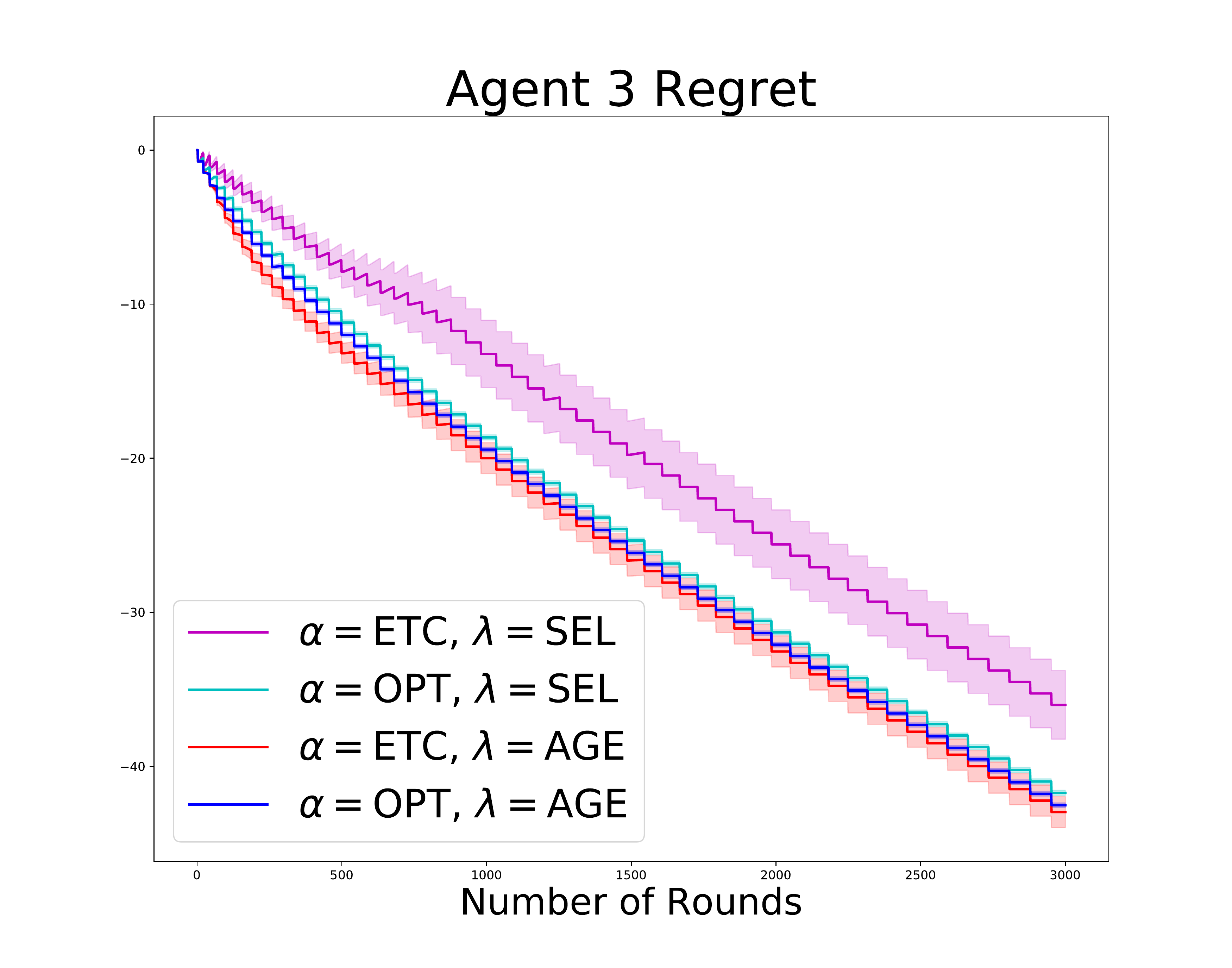}
\hspace{\dflablationfigspace}
  \includegraphics[width=\dflablationfigwidth]{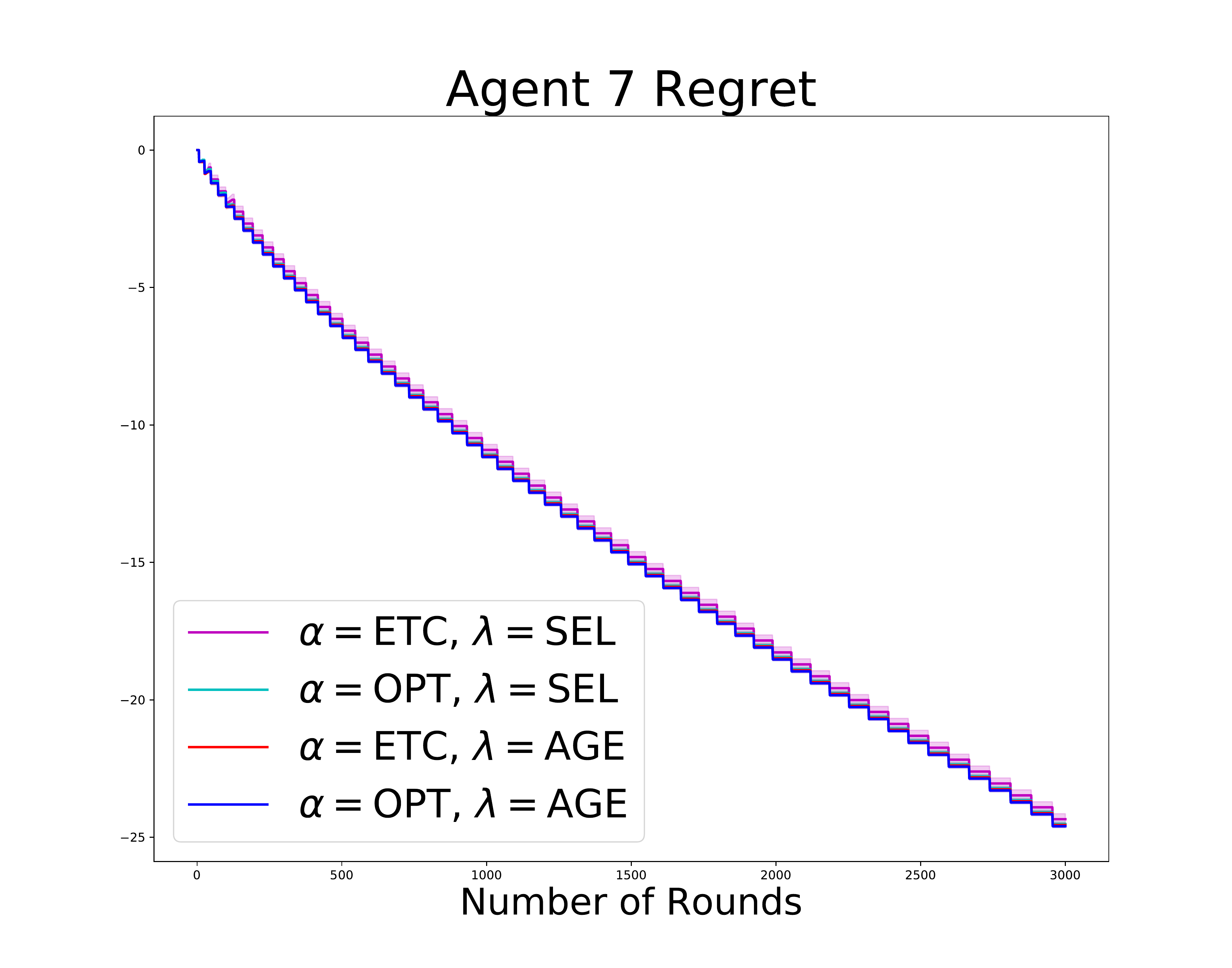}
\hspace{\dflablationfigspace}
\vspace{\imcaptionspace}
\caption{
Results for a single-item simulation study.
We show the welfare regret, the seller regret, and the regret of four agents for the four
difference choices for $(\estmethod, \pricemethod)$ over $3000$ rounds.
Lower is better in all plots.
Agent $1$ has the highest value, and if the values were known, the item would be assigned to agent
$1$.
The figures were obtained by averaging over 50 independent runs and the shaded regions
represent two standard errors.
The jagged shape of the curves is due to the periodic exploration phase in the algorithm.
Some of the curves overlap in the last two plots for agents 3 and 10.
\label{fig:figsimulation}
\vspace{\imtextspace}
}
\end{figure}
}

%% file: abstract.tex
\begin{abstract}
We study a multi-round welfare-maximising mechanism design problem in instances where
agents do not know their values.
On each round, a
mechanism first assigns an allocation each to a set of agents and charges them a price;
at the end of the round,
the agents provide (stochastic) feedback to the mechanism for the allocation they received.
This setting is motivated by applications in cloud markets and online advertising
where an agent may know her value for an allocation only after experiencing it.
Therefore, the mechanism needs to explore different allocations for each agent so that it can
learn their values,
while simultaneously attempting to find the socially optimal set of allocations.
Our focus is on truthful and individually rational mechanisms which imitate the
classical VCG mechanism in the long run.
To that end, we first define three notions of regret for the welfare, the
individual utilities of each agent 
and that of the mechanism.
We show that these three terms are interdependent via an $\bigOmega(T^{\nicefrac{2}{3}})$
lower bound for the maximum of these three terms after $T$ rounds of allocations,
and describe an algorithm which essentially achieves this rate.
Our framework also provides flexibility to control the pricing scheme so as to trade-off
between the agent and seller regrets.
Next, we define asymptotic variants for the truthfulness and individual rationality requirements
and provide asymptotic rates to quantify the degree to which both properties are satisfied
by the proposed algorithm.

\end{abstract}









%% file: intro.tex
\section{Introduction}
\label{sec:intro}

Mechanism design is one of the most important problems in economics
and computer science~\citep{nisan2001algorithmic}.
A mechanism chooses \emph{allocations} for multiple rational agents with possibly conflicting
goals and charges them a price.
It is necessary to find an \emph{outcome} (an allocation to each agent)
that is as beneficial as possible to all agents and the mechanism designer.
Agents who act in their own self interest might choose to misrepresent their values
in order to obtain an advantageous allocation.
Mechanism design aims to elicit values from agents, such that the
agents are incentivised to report truthfully (truthfulness),
while ensuring that they
are not worse off than if they had not participated in the mechanism (individual rationality).

As a motivating example,
consider a Platform-as-a-Service (PaaS) provider
who serves multiple customers using the same compute cluster.
The service provider (seller) chooses a service level (allocation) for each
customer (agent) and charges them accordingly.
The service level determines the resources allocated to the customer, and consequently
her value for that service, which could be tied to her own revenue.
A customer's experience of a service level at a given instant
is affected by exogenous stochastic factors such as traffic, machine failures, etc., which is
beyond the control of the customer.
The celebrated VCG mechanism~\citep{vickrey1961auction,groves1979efficient,clarke1971multipart}
provides a means to find outcomes which maximise the social welfare (sum of agent and seller
utilities) in such situations while satisfying
truthfulness and individual rationality.
For instance, if one customer's application is memory intensive and another's is compute intensive,
they can be co-located on the same set of machines instead of using separate machines.
This might be a better outcome for the service provider as she can serve both customers at a cheaper
cost,
and for the customers, since the service provider can now charge them less and they achieve the
same end result.
The VCG mechanism requires customers to submit bids for each service level and encourages truthful
behavior;
i.e., the dominant strategy for each customer is to submit their true value as the bid.

A crucial shortcoming of most mechanism design work, which limits its usage in practice,
 is that it assumes agents know their own values for each allocation.
For instance, the VCG mechanism
requires that customers submit bids representing these values.
This may not be true  in many real world situations, especially when there are
many unsophisticated agents and/or when
the number of allocations is very large.
However, having experienced an allocation, it is often the case that a customer can
provide feedback based on their experience.
She can either measure this directly via the impact on her own revenue, such as in online
advertising where an ad impression might lead to a click and then a purchase,
or gauge it from performance metrics, such as in the PaaS example where the
service level affects the fraction of queries completed on time, which in turn affects
her revenue.


\textbf{Setting:}
In a departure from prior work, 
we study mechanisms where agents do not know their values a priori.
However, the mechanism can learn them over multiple rounds of allocations and
feedback,
while simultaneously finding the socially optimal outcome.
At the beginning of each round, the mechanism chooses an outcome, i.e. an allocation for each agent,
and charges each agent a price.
At the end of the round, the agents report 
stochastic  feedback on their experience in using the given allocation, which we will call
\emph{reward}.
When choosing an outcome for a given round, the mechanism may use the rewards reported by the
agents in previous rounds.



This problem ushers in the classical explore-exploit dilemma encountered in bandit settings.
Provided that all agents report their rewards truthfully,
choosing the outcome that appears to be the best according to feedback provided by agents
up to the current round will likely have large welfare.
However, exploring other outcomes might improve the estimate of the
best outcome for future rounds.

As is the case in prior mechanism design work,
we assume that agents are strategic and rational, which necessitate the truthfulness
and individual rationality requirements.
A strategic agent wishes to maximise her total utility after $T$ rounds,
which is simply the sum of her instantaneous utilities (value of the allocation received minus
price).
An individually rational agent wishes to ensure that her utility after $T$ rounds is non-negative,
so that she stands to gain by participating in the mechanism.
Both these requirements are more challenging in our setting.
A mechanism cannot learn agent values 
if she does not report back truthfully.
Since she reports a reward at the end of each round, a strategic agent has significantly more
opportunity to manipulate outcomes in her favour, than in typical mechanism design settings where
she submits a single bid once.
In particular, she may be strategic over multiple
rounds, say, by incurring losses in early rounds in order to gain in the long run.
Additionally, since an agent's true values cannot be exactly known, the mechanism runs the
risk of overcharging them, which might cause her to withdraw from the mechanism.

We design an algorithm that accounts for the above considerations.
Applications such as PaaS or online advertising, where there are repeated agent-mechanism
interactions and where values can be reported back in an automated way, are suitable for
such methods.

\insertTableResultsSummary

\paragraph{Our Contributions:}
Our contribution in this work is threefold:
\begin{enumerate}
\item
First, we formalise mechanism design with bandit feedback
for settings where agents do not know their values,
but the mechanism is repeated for several rounds.
In order to quantify how close the mechanism is to the VCG mechanism,
our formalism defines the \emph{\valueregret{}}; this is derived
via regret terms for the welfare, the seller, and the agents relative
to the VCG mechanism.
Additionally,
given the above challenges in achieving truthfulness and individual rationality exactly,
we define asymptotic variants to make the problem tractable.
\item
Second, we establish a hardness result 
via a $\bigOmega(\Ttwth)$ lower bound after $T$ rounds for the \valueregret{}
even under truthful reporting from agents.
\item
Third, we describe \vcglearn, an algorithm whose behaviour is determined by two binary hyperparameters.
For all values of these hyperparameters, the algorithm is asymptotically individually rational and
truthful,
and moreover match the above lower bound on the \valueregret{} up to factors that are
polylogarithmic in $T$ and polynomial in other problem-dependent terms.
However, the asymptotic rates and the regrets of the agent and the seller are affected
by the choice of these hyperparameters.
Table~\ref{tb:resultssummary} summarises the results for Algorithm~\ref{alg:vcglearn}.
\end{enumerate}

This manuscript is organised as follows.
First, in 
Section~\ref{sec:related}, we discuss related work.
In Section~\ref{sec:setup}, we briefly review the VCG mechanism and
describe our formalism.
Section~\ref{sec:lowerbound} presents the lower bound on the \valueregret.
Section~\ref{sec:vcglearn} present our algorithm, \vcglearn,
and Section~\ref{sec:analysis} presents the
theoretical results for \vcglearn.
Section~\ref{sec:experiments} presents some simulation results.
The proof of the lower bound is given in Section~\ref{sec:lb_app} and
the proofs of results in Section~\ref{sec:analysis} are given in Section~\ref{sec:vcglearnproofs}.


%% file: related.tex
\section{Related Work}
\label{sec:related}

Bandit problems were first studied by~\citet{thompson33sampling} and have since become an
attractive framework to study exploration-exploitation trade-offs that arise in
online decision-making.
Optimistic methods, which usually choose an arm on a given round by maximising an upper confidence
bound on the mean rewards, 
are known to be minimax optimal in a variety of stochastic optimisation
settings~\citep{lai1985asymptotically,auer03ucb,bubeck2011x}.
Explore-then-commit strategies use separate rounds for exploration and exploitation.
While they are provably
sub-optimal~\citep{garivier2016explore}, they separate exploration from exploitation
facilitating a cleaner analysis when we need to combine optimisation
with other side objectives, such as in our problem, where we need to provide truthfulness guarantees
and compute the prices.

Mechanism design has historically been one of the core areas of research in the economics and
game theory literature
with applications in kidney exchange~\citep{roth2004kidney},
matching markets~\citep{roth1986allocation}, and
fair division~\citep{procaccia2013cake}.
Our work is on auction-like settings for mechanism design.
In addition to a rich history of research on this topic,
there has also been a recent flurry of work due to the rise in popularity of
sponsored search markets~\citep{lahaie2007sponsored,mehta2007adwords,aggarwal2006truthful},
wireless spectrum auctions~\citep{cramton2013spectrum,milgrom2017discovering},
and cloud spot markets~\citep{toosi2016auction}.

There is a long history of work in the intersection of machine learning and mechanism design.
Some examples include online learning
formulations~\citep{dudik2017oracle,amin2013learning,kakade2010optimal},
learning bidder values from past
observations~\citep{balcan2016sample,blum2015learning,balcan2008reducing},
and learning in other settings with truthfulness
constraints~\citep{mansour2015bayesian}.
Some work in this space study settings where individual agents may learn to bid in a repeated
auction.
Here, an agent may not know her value at the beginning, but needs to submit bids at the
\emph{beginning}
of each round.
The agent may calibrate her bid based on past rewards.
In this line of work,~\citet{weed2016online} and~\citet{feng2018learning} study a setting where the
behaviour of the mechanism is fixed over multiple rounds,
while~\citet{nedelec2019learning} study a setting where the mechanism may adapt its behaviour
so as to maximize revenue.
In a similar vein, ~\citet{liu2019competing} develop bandit methods where agents on one side of a
matching market learn to bid for arms on the other side.
In the above work, the regret is defined for the agent in question, defined relative to an oracle
which knows the true values.
In contrast to these works, in our setting, learning happens entirely on the mechanism side
and the role of each agent is very simple: submit the reward at the \emph{end} of the round.
This imposes minimal burden on agents who, while being strategic and rational,
may not be very sophisticated.



A body of work study multi-armed bandit formalisms for auctions with canonical
use cases in online advertising~\citep{babaioff2015truthful,babaioff2014characterizing,devanur2009price}.
In the above works, there is a single item (ad slot) with different and unknown click-through
rates for each agent $c_{i}\in\RR$.
The agent has a \emph{known} private value $v_{i}\in\RR$ for each click
and she submits a bid $b_i\in\RR$ once ahead of time representing this value.
On each round, the mechanism chooses one of the agents for the slot and observes the number of
clicks $c$; if agent $i$ was chosen, then $\EE[c] = c_{i}$
The agent's reward for this round is $c\cdot v_i$.
They formalise this problem where the agents are viewed as the arms and define regret with respect to
the optimal arm, i.e. the agent with the highest expected reward $v_i\times c_i$.
Importantly, the stochastic component  $c$ of the reward is observed by \emph{both} the agent and the
mechanism.
In both works, truthfulness means that the agent is incentivised to submit a bid $b_i=v_i$
at the beginning of all rounds.
There are a number of differences between these works and our setting.
First, while they formalise each agent as a different arm competing for the item,
in our setting, the allocations are viewed as arms with multiple agents being able to experience
different arms simultaneously.
We do not believe their results, and their lower bound in particular,
 can be straightforwardly extended to settings where
multiple agents might receive an allocation.
Second, in these works the agent can only submit a single bid and the
stochastic component of the reward (number of clicks) is observed by the mechanism on each round.
In contrast, in our setting, the reward on each round is only revealed to the agent, and she may
misreport this reward to the mechanism on each round;
therefore, she has significantly more opportunity
to manipulate outcomes in her favour.
Due to these differences, their results are not comparable to ours.
We will elaborate in other differences between our results and theirs in further
detail at the end of Sections~\ref{sec:lowerbound} and~\ref{sec:analysis}.


In other related work,~\citet{braverman2019multi} consider a setting where a seller chooses one
of $n$ agents to receive an item on each round of a repeated auction.
The agents submit a payment at the end of the round to the seller based on the reward they observed.
They study mechanisms that allow the seller to extract as much payment as possible from the agents
who themselves are trying to maximise their long term utility.
~\citet{nazerzadeh2016sell} study a multi-round setting where the seller chooses an agent and a
price on each round; the agent may choose to purchase the item at the price in which case
the seller receives some revenue.
Their goal is to maximise the revenue for the seller over a finite horizon of $T$ rounds.
~\citet{gatti2012truthful} study an online advertising setting when there are multiple ad slots
with different click-through rates which are the same for all agents, and design a mechanism which
charges the agents only when an ad is clicked. 
Finally, some work on dynamic 
auctions~\citep{bergemann2006efficient,athey2013efficient,kakade2013optimal} study settings
where agent values are unknown at the beginning but there is a known prior on the agent
value. Over time, she receives side information and the mechanism needs to
incentivise truth telling so as to update the posterior.

Perhaps the closest work to ours
is~\citet{nazerzadeh2008dynamic}, who study a single item
auction with a feedback method similar to ours: agents report rewards at the end of the round and
the learning happens at the mechanism.
While they consider asymptotic efficiency, truthfulness, and individual rationality
(with definitions that differ from ours), they do not provide rates, establish lower bounds,
or study the regrets of the agents and seller.


%% file: setup.tex
\section{Problem Description}
\label{sec:setup}
\label{sec:prelims}

\subsection{A brief review of the VCG mechanism}
\label{sec:vcgreview}

We begin with a brief review of mechanism design adapted to our setting.
There are $n$ agents (customers) $\{1,\dots,n\}$,
a mechanism (seller), and a set of possible outcomes $\Alloc$.
The mechanism chooses an \emph{outcome} $\alloc\in\Alloc$
and charges a price $\pricei$ to each agent.
For agent $i$, there exists a function $\mapfuni:\Alloc\rightarrow\Scal$
which maps outcomes to \emph{allocations} relevant to the agent;
i.e., different outcomes $\alloc, \alloc'$ might yield the same allocation to the agent,
$\mapfuni(\alloc) = \mapfuni(\alloc')$.
In this work, $|\Scal|<\infty$.
$\Scal$ could be as large as $\Alloc$, but could be much smaller in some applications.
This distinction between $\Scal$ and $\Alloc$ will be important when we consider the learning
problem in Section~\ref{sec:vcglearnsetup}; as we will see shortly,
our regret bounds will scale with $|\Scal|$ and not $|\Alloc|$.

Agent $i$ has a \emph{value function},
$\vali:\Scal\rightarrow[0, 1]$, where $\vali(s)$ represents her
private independent value for the allocation $s$.
For an outcome $\alloc\in\Alloc$, we will overload notation and write $\vali(\alloc) =
\vali(\mapfuni(\alloc))$.
After the agent experiences an allocation, she realises a \emph{reward} $X_i$  drawn from a
$\sigma$ sub-Gaussian distribution with mean $\vali(s)$.
We let $\valmech:\Alloc\rightarrow\RR$ denote the value function of the mechanism designer.
In the PaaS example,
$\valmech(\alloc)$ may denote the cost to the service provider for
providing the service where the allocations are
as specified in $\alloc$.
%
For an outcome $\alloc$ and prices $\{\pricei\}_{i=1}^n$, the \emph{utility} of
agent $i$ is $\utili = \EE[X_i] - \pricei = \vali(\alloc) - \pricei$.
The utility of the seller (which may represent profit) is
$\utilmech = \valmech(\alloc) + \sum_{i=1}^n\pricei$.
The \emph{welfare} $\Val(\alloc)$ is the sum of
the agent and seller values
$\Val(\alloc) = \sum_{i=1}^n \vali(\alloc) + \valmech(\alloc)$, which is also the
sum of all utilities regardless of the prices $\{\pricei\}_{i=1}^n$.

The expectations above are taken with respect to the rewards, i.e. the exogenous
stochasticity arising when agents experience their allocation.
In applications of interest, the agent does not have control over nor is able to predict this
stochasticity.

\textbf{The VCG Mechanism:}
Assume that the agents know their value functions $\vali$ and submit them truthfully as
bids to the seller.
The VCG mechanism stipulates that we choose the outcome $\allocopt$
which maximises the welfare. We then charge agent $i$ an amount $\priceiopt$,
which is the loss her presence causes to the others.
Precisely, denoting 
$\Valmi(\alloc) = \valmech(\alloc) + \sum_{j\neq i} \valj(\alloc)$, we have
\begin{align*}
\numberthis
\label{eqn:vcg}
\allocopt &= \argmax_{\alloc\in\Alloc} \; \Val(\alloc),
\hspace{0.4in}
\priceiopt = \max_{\alloc\in\Alloc} \Valmi(\alloc) - \Valmi(\allocopt).
\end{align*}
In general, an agent may submit a bid $\bidi:\Scal\rightarrow[0,1]$ (not necessarily truthfully),
and the mechanism computes the outcomes and prices by replacing $\vali$ with $\bidi$ above.
The VCG mechanism satisfies the following three fundamental desiderata in mechanism
design~\citep{karlin2017game}:
\begin{enumerate}
\item \emph{Truthfulness}:
A mechanism is  truthful or dominant strategy incentive-compatible if,
regardless of the bids submitted by other agents,
the utility $\utili$ of agent $i$ is maximised
when bidding truthfully, i.e. $\bidi=\vali$.
\item \emph{Individual rationality:} 
A mechanism is individually rational if it does not charge an agent more than
her bid for an allocation. Thus, if she bids truthfully, her
utility is nonnegative.
\item \emph{Efficiency:} If all agents bid truthfully,
a mechanism is efficient
if it maximises welfare.
\end{enumerate}

Since the agents cannot control the exogenous stochasticity,
it is meaningful for agents to submit bids based on their expected rewards, i.e. their value.
This is different from Bayesian formalisms for mechanism design
where agent values are drawn from a \emph{known} prior and she may submit bids based on this value.
(A Bayesian formulation of this setting would assume priors over the values, i.e.
the expected rewards, themselves.)
The following examples illustrate our motivations.

\insertprethmspacing
\begin{example}[PaaS]
\label{exm:paas}
\emph{
In the PaaS example from Section~\ref{sec:intro}, $\Scal$ are the
service levels (allocations) available to a customer.
$\Omega=\Scal^n$ are the possible outcomes.
 $-\valmech(\omega)$ is the cost for providing the service as specified in $\alloc$.
An agent's reward $X_i$ for a service level $s$ could denote her instantaneous revenue, which
is affected by exogenous stochastic factors such as traffic, machine failures,
etc., but it concentrates around her expected revenue, i.e. her value, $\vali(s)$.
A strategic who agent cannot control such stochastic effects would hence submit bids so as to
maximise her utility (expected reward minus price).
Such PaaS services can take place in a competitive market or internally
within an organisation where the provider is one team providing a service to other teams.
}
\end{example}

\insertprethmspacing
\begin{example}[Online Advertising]
\label{exm:advertising}
\emph{
A publisher (mechanism)
has a set of advertising slots $\Scal$ and must assign them to $n$ advertisers (agents).
Typically, $|\Scal|\ll n$ and there exists $\emptyset\in\Scal$ indicating no assignment.
When a slot is assigned to an advertiser, her reward is her instantaneous revenue which is
simply the number of people who clicked the ad and then purchased the product.
Consequently, it is a random quantity.
Different agents could have different values for different slots.
$\Omega$ is the set of possible ways in which the mechanism can assign slots to advertisers.
}
\end{example}

Henceforth,
when we say that an agent is truthful, we mean that she reports her values truthfully,
whereas when we say that a mechanism is truthful, we mean that it incentivises
truthful behaviour from the agents.
We are now ready to describe the learning problem when agents do not know their values,
but when the mechanism is repeated for multiple rounds.

\subsection{Learning a VCG mechanism under bandit feedback from agents}
\label{sec:vcglearnsetup}


In the multi-round setting, agent and seller values $\{\vali\}_{i=0}^{n}$
remain fixed throughout all the rounds.
On round $t$, the  mechanism chooses an outcome $\alloct\in\Alloc$
and sets prices $\{\priceit\}_{i=1}^n$ for the agents.
Then, agent $i$ realises a stochastic reward $\Xit$ which has expectation $\vali(\alloct)$.
At the end of the round, she reports a reward $\Yit$; if she is being truthful, she would
report $\Yit=\Xit$, but she may also choose to misreport the reward.
While the agent does not know her values $\vali$, by reporting the reward at the end of each
round, the mechanism could learn these values over multiple rounds.

%
%

While the primary focus of this work are for agents who do not know their values, our mechanism
can also accommodate agents who know their values up front.
Hence, we will also permit agents to submit bids $\bidi:\Scal\rightarrow[0, 1]$ 
(not necessarily $\bidi=\vali$) which represent her values for all rounds;
she may do so once before the first round.
We will refer to agents who submit rewards at the end of each round as those
participating \emph{by rewards}, and those who submit bids once at the beginning as those
participating \emph{by bids}.
As we will see shortly, stronger results are possible for agents who participate by bids as their
values need not be learned.

When choosing outcome $\alloct$ on round $t$,
the mechanism may use the information gathered from previous
rewards $\{\Xil\}_{\ell=1}^{t-1}$ for agents participating by rewards and the bids
$\bidi$ for agents participating by bids.
The utility of
agent $i$ on round $t$ is $\utilit = \EE[\Xit] - \priceit = \vali(\alloct) - \priceit$,
where the expectation is only with respect to the rewards (exogenous stochasticity) at round $t$.
The utility of the seller is $\utilmecht = \valmech(\alloct) + \sum_{i=1}^n\priceit$.
Let $\UiT, \UmechT$, defined below,
 denote the sum of utilities of agent $i$ and the mechanism respectively over $T$ rounds.
We have:
\begin{align*}
\label{eqn:agentutil}
\numberthis
\UiT = \sum_{t=1}^T \utilit, \hspace{1.1in}
\UmechT = \sum_{t=1}^T \utilmecht.
\end{align*}

Our goal is to design an anytime algorithm which imitates the VCG mechanism over time.
To that end,
we quantify the performance of an algorithm via the following regret terms,
defined relative to the VCG mechanism~\eqref{eqn:vcg},
after $T$ rounds of interactions:
\vspace{-0.05in}
\begingroup
\allowdisplaybreaks
\begin{align*}
&\RT = T\Val(\allocopt) - \sum_{t=1}^T\Val(\alloct),
\hspace{0.5in}
\RiT = T\utiliopt - \UiT,
\hspace{0.5in}
\RaT = \sum_{i=1}^n \RiT,
\\
&\RmechT = T\utilmechopt - \UmechT,
\hspace{0.5in}
\RmaxT = \max(n\RT, \RaT, \RmechT).
\numberthis
\label{eqn:regret}
\end{align*}
\endgroup
Here $\allocopt$ is the optimal outcome~\eqref{eqn:vcg}, which we will assume is unique.
Moreover, 
$\utilmechopt = \valmech(\allocopt) + \sum_i\priceiopt$ and
$\utiliopt = \vali(\allocopt) - \priceiopt$
are the utilities of the seller and agent $i$ respectively in the VCG mechanism.
$\RT$ is the welfare regret over $T$ rounds; it measures the welfare of the chosen outcomes
$\alloct$ relative to $\allocopt$.
$\RiT$ is the regret of agent $i$ and $\RmechT$ is the regret of the seller, both defined relative
to the VCG mechanism.
$\RaT$ is the sum of all agents' regrets.
Finally, we also define the \emph{\valueregret{}} $\RmaxT$.
In~\eqref{eqn:agentutil} and~\eqref{eqn:regret},
we have followed pseudo-regret convention, which takes an expectation with respect to
the rewards at the current round.
\toworkon{Better term to use here than \valueregret?}

Our goal is to imitate the VCG mechanism over time, and $\RmaxT$
captures how well the welfare, and all agent/seller utilities converge uniformly to their
VCG values.
As we will see shortly, $\RmaxT$ will be a fundamental quantity in this problem,
and we will use it to establish a hardness result.
We focus on the VCG mechanism because it is one of the well-studied
paradigms in multi-parameter mechanism design and is therefore a natural starting point.
Moreover, even in competitive markets, sellers may be motivated to maximise welfare for long-term
customer retention.
This is similar in spirit to~\citet{devanur2009price} who study a seller's regret,
and~\citet{weed2016online} who study an agent's regret when the agent bids in
a repeated single item auction---in both cases, the regret is defined relative to the Vickrey auction.

A truthful agent simply reports her rewards at the end of each round.
In general though, a strategic agent follows some strategy $\pi$ so as to maximise her sum of
utilities over several rounds.
If she is participating by rewards, $\pi$ is a map from her
past information $\{(\mapfuni(\allocl), \priceil, \Xil, \Yil)\}_{\ell=1}^{t-1}$
and current allocation, price, and reward $(\mapfuni(\alloct),\priceit,\Xit)$
to a (possibly random) scalar to report as $\Yit$.
In particular, the agent may adopt a non-truthful strategy $\pi$, where
she misreports her reward at the end of the current round so as to manipulate the allocations she
may receive in future rounds, with the intent of maximising her long-term utility $\UpiiT$
for large $T$.
We also mention that
if an agent is participating by bids, $\pi$ is simply the bid that she submits
ahead of time.

In addition to obtaining sublinear VCG regret~\eqref{eqn:regret},
we would like to achieve the three
desiderata for mechanism design given in Section~\ref{sec:vcgreview}.
Here we define variants of those desiderata in order to precisely delineate
the extent to which they can be achieved in our setting.
\begin{enumerate}
\item \emph{Truthfulness}:
\label{pro:asytruth}
Let $\UiT$ and $\UpiiT$ respectively denote the sum of utilities of agent $i$ when she is being truthful and
when she is following any other (non-truthful) strategy $\pi$.
A mechanism is \emph{truthful}, if,
for all $\pi,T$, $\UpiiT \leq \UiT$ almost surely (a.s), regardless of the behaviour of other agents.
It is \emph{asymptotically truthful} if, for all $\pi, T$,
$\EE[\UpiiT - \UiT] \in \littleO(T)$, regardless of the behaviour of other agents.
A mechanism is \emph{asymptotically Nash incentive-compatible} (NIC) if, for all $\pi, T$,
$\EE[\UpiiT - \UiT] \in \littleO(T)$, when the other agents are behaving truthfully.
\item \emph{Individual rationality:}
Assume that agent $i$ is truthful.
A mechanism is \emph{individually rational} if, for all $T$,
$\UiT \geq 0$ a.s, regardless of the behaviour of other agents.
It is \emph{asymptotically individually rational} if 
$\lim_{T\rightarrow\infty} \frac{1}{T} \EE[\UiT] \geq 0$, regardless of the behaviour of other agents.
\item \emph{Efficiency:}
A mechanism is \emph{asymptotically efficient} if $\EE[\RT] \in \littleO(T)$
when all agents are reporting truthfully.
\end{enumerate}
To undestand the difference between the almost-sure and in-expectation definitions above,
recall that 
$\utilit = \vali(\alloct) - \priceit$ contains an  expectation with
respect to the reward at round $t$, but is a random quantity
as the outcome $\alloct$ and price $\priceit$ depend on
the rewards realised/reported by all agents in previous rounds.
In our almost sure definitions above, the statements should hold
regardless of this randomness, whereas in our in-expectation definitions, they need to hold
in expectation over the past exogenous randomness.
\toworkon{I had previously called it Bayes-Nash incentive compatible and not Nash Incentive
comaptible since the former is a more standard term used in mechanism design while the latter is
not. However, there is nothing Bayesian about this property that I have defined. Let me know your
thoughts on this.}

While achieving dominant-strategy incentive-compatibility is a desirable goal, it can be
difficult, especially in multi-round
mechanisms~\citep{babaioff2014characterizing,babaioff2013multi}.
A common approach to sidestep this difficulty is to adopt a Bayesian formalism which assumes
that agent values are drawn from \emph{known} prior beliefs and consider
\emph{ex ante} or \emph{ex interim} versions of incentive-compatibility.
However, Bayesian assumptions can be strong~\citep{schummer2004almost} as it may not be possible to
know the prior distributions ahead of time.
In contrast, we do not make such distributional assumptions,
but rely on  asymptotic notions of truthfulness to make the problem tractable.
If a mechanism is asymptotically truthful, the maximum value an agent may gain by not
being truthful vanishes over time.
In many applications it is reasonable to assume that agents would be truthful if the benefit
of deviating is negligible, especially in settings where they may not know their value.
It is worth pointing out that prior work has explored similar ideas of approximate
incentive-compatibility in various
contexts~\citep{nazerzadeh2008dynamic,lipton2003playing,kojima2010incentives,roberts1976incentives,feder2007approximating,daskalakis2006note}.



%
Finally, we will define two problem-dependent terms for what follows.
First, let $K$ be the minimum number of rounds necessary to assign all allocations to all agents.
In Example~\ref{exm:paas}, we can do this in $K=|\Scal|$ rounds,
provided that there are no constraints on assigning different service
levels to different agents.
In Example~\ref{exm:advertising}, this can be done in $K=n$ rounds if $|\Scal|<n$.
Second, let $\Vmax$ be an upper bound on the expected welfare.
Since $\vali(s)\in[0, 1]$, $\Vmax$ could be as large as $\bigO(n)$.
However, it can be small in settings such as Example~\ref{exm:advertising} where it
is $\bigO(1)$ if there is only one ad slot.

We make two observations before we proceed.
First, we consider a fairly unadorned version of this problem
as it provides the simplest platform to study how truthfulness and individual rationality
constraints affect learning in this setting.
One could study richer models which assume structure between the allocations in $\Scal$
or that the values $\vali$ change on each round.
For instance,  we may assume that
$\Scal\subset\RR^d$ and that $\vali$ is either linear or smooth in these $d$ attributes.
We may also consider variations which incorporate changing values and/or contextual information.
While these settings are beyond the scope of this work,
we believe the analysis techniques and intuitions developed in this work would be useful in
analysing such settings.
Second, while our feedback model requires agents to share their observed reward at the end of each
round,
this is not too dissimilar from agents sharing their values in mechanism design.
For instance, in Example~\ref{exm:advertising}, in usual truthful mehanisms, the agents
would share their expected revenue from an ad slot when its known, whereas in our setting they
would submit their instantaneous revenue on each round so that its expectation can be learned.


%% file: lowerbound.tex
\section{A Hardness Result}
\label{sec:lowerbound}
\label{sec:lb}

We first establish a lower bound on the \valueregret{}, defined in~\eqref{eqn:regret},
even when all agents are truthful.
To formalise this,
let $\Theta$ be the class of problems with $n$ agents,
and $\Acal$ be the class of algorithms for this setting.
Note that
the regret terms in~\eqref{eqn:regret} depend on the specific problem in $\Theta$ and algorithm in
$\Acal$.

\insertprethmspacing
\begin{theorem}
\label{thm:lowerbound}
Let $n\geq 2$ and assume all agents are truthful.
Let the \valueregret{} $\RmaxT$ be as defined in~\eqref{eqn:regret}.
Then, for $T\geq 128 n$,
\[
\inf_\Acal \sup_\Theta \;\EE\left[\RmaxT\right]
\geq\,
\frac{1}{50} (n-1)^{\nicefrac{4}{3}} T^{\nicefrac{2}{3}}.
\]
\end{theorem}

The above theorem captures the following intuition: regardless of the chosen outcome,
the seller can achieve small regret by demanding
large payments from the agents; however, this will result in large agent regret.
Hence, there is a natural trade-off between agent and seller regrets, which is determined by how
the prices $\{\priceit\}_{i,t}$ are set when the VCG prices $\{\priceiopt\}_{i}$  are unknown but
can only be estimated from data.
In fact, we will also see this phenomenon manifest in our algorithm, where, 
while there is flexibility to handle this trade-off
in a way  that  is favourable to either the agent or the seller,
the maximum of $\RaT$ and $\RmechT$ is always large (Proposition~\ref{prop:agentseller}).

It is necessary to study $\max(n\RT, \RaT, \RmechT)$ instead of $\max(\RaT, \RmechT)$ as we need
to account for the fact that the value being shared by the agents and the mechanism is
constrained by the total welfare generated, which is factored into the welfare regret $\RwT$.
Precisely, the total welfare generated over $T$ rounds is $T\Val(\allocopt) - \RwT$ and not
the maximum achievable $T\Val(\allocopt)$ when the values are known.

It is instructive to compare this result with prior 
lower bounds in similar settings where learning happens on the mechanism
side~\citep{babaioff2014characterizing,babaioff2015truthful,devanur2009price}.
In the online advertising setting described in Section~\ref{sec:related},
~\citet{babaioff2014characterizing} show that $\bigOmega(\Ttwth)$ welfare regret is unavoidable
 for deterministic a.s. truthful mechanisms.
~\citet{devanur2009price} establish a similar lower bound for the seller regret, defined relative to the
seller's revenue in a Vickrey auction for online advertising.
Both hardness results rely on a necessary and sufficient condition for truthfulness
in single parameter auctions~\citep{archer2001truthful,myerson1981optimal}.
In contrast, our result for the \valueregret{} is obtained by studying the estimation error of the
prices and applies to the maximum of the welfare, agent, and seller regrets, even
when agents are reporting truthfully.
Moreover, while our result applies to the VCG regret for general mechanisms,
their results apply to the welfare/seller regrets only in the online advertising use case described in
Section~\ref{sec:related}.
~\citet{babaioff2015truthful} design a randomised multi-round mechanism for this online advertising
use case which is truthful in expectation and achieves
$\bigO(\sqrt{T})$ welfare regret.
This result does not contradict our result above which, as described in Section~\ref{sec:related},
 considers a different feedback model and additionally accounts for
the agent and seller regrets along with the welfare regret.

\paragraph{Proof sketch of Theorem~\ref{thm:lowerbound}:}
Minimising all regret terms requires that we estimate the VCG prices~\eqref{eqn:vcg} correctly,
which is the main bottleneck as the best outcome omitting any given agent might be very different from
the optimal outcome.
We first use a series of manipulations to lower bound the VCG regret via
$\RmaxT \geq n\EE\RT + \EE\WT$ where $\RT$ is the welfare regret and
$\WT$ captures how well we have estimated the prices.
These two terms are conflicting---minimising one will cause the other to be large.
We reduce the task of minimising the supremum of this sum over $\Theta$ to a binary hypothesis
testing problem between two carefully chosen problems in $\Theta$.
We then apply a high probability version of Fano's inequality to obtain the result.
The complete proof is given in Section~\ref{sec:lb_app}.

%% file: vcglearn.tex
\section{Algorithm}
\label{sec:vcgee}
\label{sec:vcglearn}

We now describe our algorithm for this setting, called \vcglearn,
which is outlined in Algorithm~\ref{alg:vcglearn}.
The algorithm has two binary hyperparameters $\estmethod, \pricemethod$,
which control the trade-offs between the agent
and seller regrets and properties such as truthfulness and individual rationality. 
We will first describe the algorithm then explain how these hyperparameters may be used
to control the above properties.

\vcglearn{} proceeds over a sequence of brackets.
Brackets are indexed by $q$ and rounds by $t$.
Each bracket begins with an explore-phase of $K$ rounds where the mechanism assigns all
allocations in $\Scal$ to all agents at least once.
It does not charge the agents during this phase but collects their realised rewards.
This is then followed by  
$\lfloor \frac{5}{6}Kq^{\nicefrac{1}{2}{}}\rfloor$
rounds, during which the mechanism sets the outcome and prices based on the
rewards collected thus far.
The outcomes in the latter phase are chosen depdendent on hyperparameter $\estmethod$;
if $\estmethod=\estee$ (explore-then-commit),
we only use rewards from the explore phase to determine the outcomes,
whereas when $\estmethod=\estopt$ (optimistic), we use rewards from all rounds thus far.
By proceeding in brackets in the above manner,
we are able to optimally control the time spent in the different phases.
As
the length of the latter phase increases with each bracket, we spend more rounds in this
phase than in the explore phase as the mechanism is repeated for longer.

\insertAlgoVCGLearnTwo

To describe how we compute the outcome, we first define the following three quantities,
$\vbarit,\vubit,\vlbit:\Scal\rightarrow[0,1]$.
For an agent participating by rewards,
$\vbarit(s)$ is the sample mean of the rewards when agent $i$ was assigned outcome $s\in\Scal$,
which serves as an estimate for $\vali(s)$.
Next, $\vubit(s)$ and $\vlbit(s)$ are upper and lower confidence bounds respectively for $\vali(s)$.
They are computed as shown in~\eqref{eqn:learnterms}.
Below, $D_t$ denotes the round indices belonging to explore phases up to round $t-1$
when $\estmethod=\estee$ and
$D_t = \{1,\dots,t-1\}$ when $\estmethod=\estopt$.
$\Nit(s)$ denotes the number of observations from agent $i$ for allocation $s$ in the first $t-1$ rounds
that are used in the computation for $\vbarit,\vubit,\vlbit$.
$\sigma$ is the sub-Gaussian constant for the reward distributions (see
Section~\ref{sec:vcgreview}),
The $\vbarit,\vubit,\vlbit$ quantities are first computed when $t>K$ and $\Nit(s)\geq 1$ so they
are well defined.
We have:
\begingroup
\allowdisplaybreaks
\begin{align*}
\Nit(s) &= \displaystyle\sum_{\ell\in D_t}
         \indfone\big(\mapfuni(\allocl)=s\big),
\hspace{0.2in}
%
\vbarit(s) = {\rm clip}\left(\frac{1}{\Nit(s)}\sum_{\ell\in D_t} \Xiitt{i}{\ell}\;
                \indfone\big(\mapfuni(\allocl)=s), \;\, 0, \,1\right),
\\[0.02in]
\vubit(s) &=\vbarit(s) +
        \sigma\sqrt{\frac{5\log(t-qK + 1)
                    + 2\log(|\Scal|)}{\Nit(s)}},
\\
\vlbit(s) &= \vbarit(s) -
        \sigma\sqrt{\frac{5\log(t-qK + 1)
                    + 2\log(|\Scal|)}{\Nit(s)}}.
\label{eqn:learnterms}
\numberthis
\end{align*} 
\endgroup
Since $\vali(s)\in[0,1]$, we clip the initial estimate between $0$ and $1$ to obtain $\vbarit$.
We will assume that each agent
experiences each allocation in $\Scal$ \emph{exactly once} during the exploration phase at the
beginning of each bracket.
If an agent was assigned the same allocation multiple times, we will use the reported
value of only one of them, picked arbitrarily.
For an agent $i$ who participates by bidding $\bidi$, we simply set, for all $s\in\Scal$,
\begin{align*}
\label{eqn:bidterms}
\numberthis
\vbarit(s) = \vubit(s) = \vlbit(s) = \bidi(s).
\end{align*}

We now define $\Vubt$, an upper confidence bound on the welfare
at time $t$.
In line~\ref{lin:alloct}, the algorithm chooses the outcome which maximises
$\Vubt$ in round $t$:
\begin{align*}
\label{eqn:Vubt}
\numberthis
\Vubt(\alloc) = \valmech(\alloc) + \sum\nolimits_{i=1}^n\vubit(\alloc).
\end{align*}

Finally, we describe how the prices are computed in line~\ref{lin:pricing}, which depend on the
hyperparameter $\pricemethod\in\{\priceage,\pricesel\}$.
First define the functions $\fit,\git:\Scal\rightarrow[0,1]$ for all $i,t$ as follows:
if $\pricemethod=\priceage$ (agent favourable pricing), set
$\fit=\vlbit$ and $\git=\vubit$;
if $\pricemethod=\pricesel$ (seller favourable pricing), set
$\fit=\vubit$ and $\git=\vlbit$.
Then define $\Fmit, \Gmit$ as follows:
\begingroup
\allowdisplaybreaks
\begin{align*}
\label{eqn:fgterms}
\numberthis
\Fmit(\alloc) &= \valmech(\alloc) + \sum_{j\neq i}\fjt(\mapfunj(\alloc)),
\hspace{0.5in}
\Gmit(\alloc) = \valmech(\alloc) + \sum_{j\neq i}\gjt(\mapfunj(\alloc)).
\end{align*}
As described in line~\ref{lin:pricing}, we charge price
$\priceit=\max_{\alloc\in\Alloc}\Fmit(\alloc) - \Gmit(\alloc)$ from agent $i$ on rounds
$t$ that are not in the exploration phase.

This completes the description of the algorithm.
To warm us up for the theoretical analysis in the next section,
we discuss the implications of the hyperparameter choices $\estmethod,\pricemethod$.
First, when $\estmethod=\estee$, Algorithm~\ref{alg:vcglearn}
behaves similarly to explore-then-commit-style bandit algorithms~\citep{perchet2013multi}.
It first explores all options at the beginning of each bracket.
It then switches to an exploit phase for the remainder of the bracket
during which it commits to the best outcome found during
previous explore phases\footnote{%
As all agents experience all allocations exactly once during each explore phase,
when $\estmethod=\estee$,
the maximiser of the upper confidence bound $\Vubt$ and the mean
$\Vbart = \valmech + \sum_i\vbarit$ coincide.
}.
The main advantage of this two-phase strategy is a clean separation between preference learning and
outcome/pricing selection which gives rise to strong truthfulness guarantees.
When $\estmethod=\estopt$, the procedure is reminiscent of optimistic
strategies~\citep{lai1985asymptotically}
which maximise an upper confidence bound using rewards from all rounds.
Not only is this empirically sample-efficient as it uses rewards from all rounds, but it
also enjoys better welfare regret and individual rationality properties
over $\estmethod=\estee$ as we shall demonstrate shortly.
Unfortunately, this comes at the cost of weaker guarantees on truthfulness.
We will elucidate this in Section~\ref{sec:analysis}.
While optimistic strategies do not usually require an explore phase,
this is necessary in our problem to accurately
estimate the prices and to guarantee asymptotic NIC.
Consequently, our bounds on the welfare, agent, seller regrets are worse than the typical $\sqrt{T}$
rates one comes to expect of optimistic strategies in stochastic bandit problems.

Next, consider $\pricemethod$, which is used in computing the $\Fmit,\Gmit$
quantities~\eqref{eqn:fgterms}, and consequently
determine the pricing calculation in line~\eqref{lin:pricing}
of Algorithm~\ref{alg:vcglearn}.
While $\pricemethod$ does not affect the outcome and the welfare generated on each round,
it determines how this welfare is shared between the agents and the seller;
therefore, it
affects the agent and seller regrets $\RiT,\RmechT$.
For instance, suppose we choose $\pricemethod=\pricesel$.
In line~\eqref{lin:pricing}, this uses the most optimistic estimate 
of the maximum welfare omitting agent $i$ in $\Fmit$, 
and the most pessimistic estimate of the values of the current outcome
for the other agents in $\Gmit$.
This results in large payments and consequently is the most favourable pricing scheme to the seller,
while still ensuring asymptotic truthfulness, individual rationality, and sublinear agent regret.
Similarly, when $\pricemethod=\priceage$, the pricing is favourable to the agents.
We will illustrate these trade-offs, along with their effects on individual rationality and
truthfulness, in the next section.
These options give a practitioner a fair amount of flexibility when applying
Algorithm~\ref{alg:vcglearn} for their specific use case.

%% file: analysis.tex
\section{Main Results for Algorithm~\ref{alg:vcglearn}}
\label{sec:analysis}

We now present our main theoretical results for \vcglearn, providing rates for
asymptotic truthfulness, individual rationality, and \valueregret{} in Sections
~\ref{sec:theorytruthful},
~\ref{sec:theoryir},
and
~\ref{sec:theoryvalue} respectively.
In Section~\ref{sec:theoryagent} we also provide bounds on the agent, seller, and welfare
regrets defined in~\eqref{eqn:regret}.
We wish to remind the reader that Table~\ref{tb:resultssummary} summarises the main results
(Theorems~\ref{thm:truthfulness},~\ref{thm:ir}, and~\ref{thm:value}) of this section.
The proofs of all results are given in Section~\ref{sec:vcglearnproofs}.


\subsection{Asymptotic truthfulness}
\label{sec:theorytruthful}

We first state the truthfulness/NIC properties of the proposed algorithm.
Theorem~\ref{thm:truthfulness} establishes that Algorithm~\ref{alg:vcglearn} is asymptotically
truthful when $\estmethod=\estee$ and is asymptotically NIC when $\estmethod=\estopt$.
In fact, we will state a slightly stronger result for the $\estmethod=\estopt$ case.
For this, we say that
a strategy $\pi$ by agent $j$ is \emph{stationary} if she either participates by bids,
or if participating by rewards, when assigned an allocation $s\in\Scal$, she reports a sample from some
fixed distribution dependent on $s$.
Any other strategy is \emph{non-stationary}.
Intuitively, when an agent participates by rewards, if we view the rewards reported for
any allocation as a time series, the strategy is stationary if this time
series is stationary.

While truthfulness implies stationarity, a non-truthful player can be either stationary or
non-stationary.
For example, when participating by rewards, an agent may choose to report
$\Yit = \phi_s(\Xit)$ when assigned an allocation $s$, where
the functions $\{\phi_s\}_s\in\Scal$ may be designed to squash or amplify rewards for certain
allocations, say, so as to discourage or encourage the mechanism from assigning said allocation to the
agent in the future.
Such reports, while non-truthful, come from a stationary distribution.
An agent is also stationarily non-truthful if, when participating by bids, she submits false values.
We have the following theorem.

\insertprethmspacing
\begin{theorem}
\label{thm:truthfulness}
Let $\pi$ be any non-truthful strategy for agent $i$.
Fix the strategies adopted by the other agents.
Let
$\UpiiT,\UiT$ be the sum of agent $i$'s utilities when she follows $\pi$ and when being truthful
respectively.
The following statements hold for any \emph{$\pricemethod\in (\priceage,\pricesel)$} for
all $T>2K$.

\textbf{1.}
First let \emph{$\estmethod=\estee$}.
If an agent participates by bids, then regardless of the behaviour of others,
then
$\UpiiT - \UiT \leq 0$ a.s;
i.e. Algorithm~\ref{alg:vcglearn} is truthful.
If the agent participates by rewards,
then, regardless of the behaviour of the others, we have,
\[
\EE[\UpiiT - \UiT] \leq 10\sigma\sqrtlogST\, K^{\nicefrac{1}{3}} T^{\nicefrac{2}{3}} + 4
\; \in \; \bigOtilde \left( K^{\nicefrac{1}{3}} T^{\nicefrac{2}{3}}\right).
\]
\textbf{2.}
Next, let \emph{$\estmethod=\estopt$} and
assume that all agents other than $i$ adopt stationary policies.
Then, for any (stationary or non-stationary) strategy $\pi$ for agent $i$,
\[
\EE[\UpiiT - \UiT] \leq 10\sigma(6n + 2) \sqrtlogST\, K^{1/3}T^{2/3} + 12n
\; \in \; \bigOtilde \left(n\KotTtt\right).
\]
The above imply that Algorithm~\ref{alg:vcglearn} is asymptotically truthful when
\emph{$\estmethod=\estee$} and asymptotically Nash incentive-compatible when
\emph{$\estmethod=\estopt$}
for an agent participating by rewards.
\end{theorem}

The guarantees when $\estmethod=\estopt$
is weak when compared to $\estmethod=\estee$ in two regards.
Not only does the asymptotic bound scale with $n$,
but it also holds only when the other agents are adopting stationary policies.
However, since truthfulness implies stationarity,
it does imply an asymptotic Nash equilibrium;
that is, when $\estmethod=\estopt$,
if all other agents are truthful, then the amount by which an agent stands to gain
by misreporting her rewards vanishes over multiple rounds.

However, as we will see shortly, when $\estmethod=\estopt$, we have better empirical results and
theoretical bounds on the welfare regret, individual rationality and the agent and seller regrets
since we use data from all rounds.
Using only a small fraction of the data
can be wasteful if we do not expect agents to be very strategic.
The $\estmethod=\estopt$ option is primarily motivated by this practical consideration.
It allows us to efficiently learn in such
environments, while  providing some weak protection against agents who
might try to manipulate the mechanism with
``simple'' methods, such as squashing/amplifying their rewards for certain allocations.

\paragraph{Proof sketch:}
We write the instantaneous difference in the utilities as
$\utilpiit - \utilit = (\utilpiit - \utiltildeit) + (\utiltildeit - \utilit)$.
Here, $\utilpiit$ is the utility on round $t$ when the agent reports according to
strategy $\pi$ up to round $t-1$,
$\utiltildeit$ is the utility on round $t$ when she follows $\pi$ up to round $t-2$ and
switches to truth-telling on round $t-1$,
and $\utilit$ is the utility on round $t$ when she is truthful on all rounds.
The first term captures the benefit of misreporting in the current round;
this can be bound using proof techniques for truthfulness of the VCG mechanism.
The latter term captures the benefit of misreporting in previous rounds;
this can be large, since, an agent's false reports will have affected the
outcomes and prices chosen by the
mechanism not just in the current round
but in previous rounds as well.
To control this term, we use properties of
our algorithm to show that the agent's past actions cannot have changed the outcomes by too much.

\subsection{Asymptotic individual rationality}
\label{sec:theoryir}

Our next theorem establishes the asymptotic individual rationality properties of
Algorithm~\ref{alg:vcglearn}.

\insertprethmspacing
\begin{theorem}
\label{thm:ir}
Consider any agent $i$. 
Let $\UiT$ be the sum of her utilities after $T$ rounds when she participates
truthfully (while others may not).
The following statements are true for all $T>2K$.

\textbf{1.}
First let \emph{$\estmethod=\estee$}.
When \emph{$\pricemethod=\priceage$},
$\UiT \geq 0$ a.s for all $T$ for an agent participating by bids;
i.e, Algorithm~\ref{alg:vcglearn} is (almost surely) individually rational.
For an agent participating by rewards,
\[
\EE[\UiT] \geq -10 \sigma \sqrtlogST\, K^{\nicefrac{1}{3}} T^{\nicefrac{2}{3}} - 4,
\hspace{0.3in} \text{i.e. } \;\EE[-\UiT]
\; \in \; \bigOtilde \left( K^{\nicefrac{1}{3}} T^{\nicefrac{2}{3}}\right)
\]
That is, Algorithm~\ref{alg:vcglearn} is asymptotically individually rational.
Similarly, when \emph{$\pricemethod=\pricesel$}, 
\[
\EE[\UiT] \geq  -  10 \sigma n \sqrtlogST\, K^{\nicefrac{1}{3}} T^{\nicefrac{2}{3}} - 4,
\hspace{0.3in} \text{i.e. } \;\EE[-\UiT]
\in \bigOtilde \left( n K^{\nicefrac{1}{3}} T^{\nicefrac{2}{3}}\right)
\]

\textbf{2.}
Next, let \emph{$\estmethod=\estopt$}.
When \emph{$\pricemethod=\priceage$}, for all agents,
$\UiT \geq 0$ a.s for all $T$ for an agent participating by bids;
i.e, Algorithm~\ref{alg:vcglearn} is individually rational.
Moreover, for an agent participating by rewards,
\[
\EE[\UiT] \geq -  9\sigma \sqrt{|\Scal|T \logST} - 6,
\hspace{0.3in} \text{i.e. } \;\EE[-\UiT]
\; \in \; \bigOtilde \left( \sqrtST\right)
\]
That is, Algorithm~\ref{alg:vcglearn} is asymptotically individually rational.
Similarly, when \emph{$\pricemethod=\pricesel$}, 
\[
\EE[\UiT] \geq -  9\sigma n \sqrt{|\Scal|T \logST}  - 6,
\hspace{0.3in} \text{i.e. } \;\EE[-\UiT]
\; \in \; \bigOtilde \left( n\sqrtST\right).
\]
\end{theorem}

While the above theorem implies asymptotic individual rationality for all $\estmethod,\pricemethod$
values,
let us consider how the different hyperparameter choices affect the rates in the theorem.
We see that when $\estmethod=\estopt$, the $\bigO(\sqrtT)$ rates are better
than when $\estmethod=\estee$, where the rate is $\bigO(\Ttwth)$.
This demonstrates the first trade-off determined by the $\estmethod$ hyperparameter:
when $\estmethod=\estee$, we have stronger truthfulness guarantees but weaker individual rationality
guarantees than when $\estmethod=\estopt$.
The stronger rates are possible in the latter case 
because we use all data to learn an agent's preferences.
Next, when $\pricemethod=\pricesel$, the asymptotic rates for individual rationality have an
additional $n$ dependence than when $\pricemethod=\priceage$.
In the former case, the agents bear the brunt of the uncertainty in the price
estimation leading to worse rates;
we will see this manifest in the agent and seller regret bounds as well in
Section~\ref{sec:theoryagent}.
Finally, we also see that when $\pricemethod=\priceage$,
the individual rationality holds exactly and almost surely for agents
participating by bids while it only does so asymptotically for agents participating by rewards.
Hence, if an agent knows her values, she is better off submitting them as bids up front.

It is also worth highlighting that the above bounds above have dependence on the size of
the allocation set $|\Scal|$ and not the size of the outcomes $|\Alloc|$
(recall from Section~\ref{sec:setup} that $K$ also may depend on $|\Scal|$, but not $|\Alloc|$).
While $|\Alloc|$ can be quite large, possibly as large as $|\Scal|^n$,
the rates scale with $|\Scal|$ since the updates to the means and confidence intervals for
one agent occur independent of the rewards observed by the others~\eqref{eqn:learnterms}.

\paragraph{Proof sketch:}
All agents have non-negative utility in the exploration phase so we can restrict our attention
to rounds not in the exploration phase.
We first show that we can decompose the utility of agent $i$ on round $t$ as
$\utilit = c_t + d_t$ where
$c_t = \vali(\mapfuni(\alloct)) - \git(\mapfuni(\alloct))$
and
$d_t = \Gt(\alloct) - \max_{\alloc}\Fit(\alloc)$.
Intuitively, $c_t$, if negative, can be viewed as negative utility that an agent may accrue due to
the mechanism mis-estimating her values, and
$d_t$, if negative, can be viewed as negative utility that an agent may accrue due to the mechanism
mis-estimating the values of the other agents, and consequently the prices.
When $\pricemethod=\pricesel$, we show that $c_t$ is small (its sum can be bound by a constant)
but $d_t$ is large; in this case, the
agents bear most of the effects of uncertainty in values leading to large asymptotic rates which
scale with $n$.
When $\pricemethod=\priceage$, $d_t$ is small and $c_t$ is large;
however, as the seller bears the effects of the uncertainty, it does not scale with $n$.
In the remainder of the analysis we show that when $\estmethod=\estee$, the information obtained
in the explore phase rounds lead to a $\Ttwth$ rate, whereas when $\estmethod=\estopt$,
the information obtained in all rounds lead to a $\sqrtT$ rate.


\subsection{Bounding the \valueregret{}}
\label{sec:theoryvalue}

Finally, we will upper bound the \valueregret{} $\RmaxT$
for Algorithm~\ref{alg:vcglearn}.
Recall that $\RmaxT$ captures how well the welfare, all agent utilities and the seller
utility converge uniformly to the VCG values.

\insertprethmspacing
\begin{theorem}
\label{thm:value}
Assume all agents are truthful.
Let $\RmaxT$ be as defined in~\eqref{eqn:regret}.
The following statements hold for any \emph{$\pricemethod\in (\priceage,\pricesel)$}.
Whe \emph{$\estmethod=\estee$},
for all $T>2K$,
\begin{align*}
\EE[\RmaxT] &\leq 
        \Big(3\Vmaxit (n+3)  + 10(5n^2 + n)  \sqrtlogST \Big) K^{\nicefrac{1}{3}} T^{\nicefrac{2}{3}}
   + 4\Vmaxit (n^2 + 3n)
\\
&\hspace{0.1in}\in \; \bigOtilde \left( n^2 K^{\nicefrac{1}{3}} T^{\nicefrac{2}{3}}\right).
\end{align*}
Next, when \emph{$\estmethod=\estopt$}, we have
that for all $T>2K$,
\begin{align*}
\EE[\RmaxT] \leq\; &
    9\sigma(3n^2 + n)\sqrt{|\Scal|T\logST}
        +\, \Big(3\Vmaxit (n+3)  + 20\sigma n^2  \sqrtlogST \Big) K^{\nicefrac{1}{3}} T^{\nicefrac{2}{3}}
\\
    & +\, 6\Vmaxit(n^2 + 3n)
\hspace{0.5in} \in \; \bigOtilde \left(n^2\KotTtt\right).
\end{align*}
\end{theorem}

We find that for both choices of $\estmethod$, we have an $\bigOtilde(n^2\KotTtt)$  upper bound
on the \valueregret{} $\RmaxT$.
It is worth noting that since we use data from all rounds,
the constants in the higher order $n^2\KotTtt$ terms are smaller when
$\estmethod=\estopt$ than when $\estmethod=\estee$.
While both upper bounds differ by a $\poly(n)$ factor
from the lower bound in
Theorem~\ref{thm:lowerbound}, it achieves the $\Ttwth$ rate.
This establishes minimax optimality for \vcglearn.

\paragraph{Proof sketch:}
We first decompose the \valueregret{} as follows:
\[
\RmaxT
\,\lesssim\,
    n\RT +
    \sum_{i=1}^n \sum_{t} \EE[|\Amit|\,|\Et] +
    n\sum_{t} \EE[|\Bt|\,|\Et],
\]
where,
$\Amit = \Valmi(\allocmiopt) - \Fmit(\allocmiopt)$,
$\Bt = \Gt(\alloct) - \Val(\allocopt)$, and
 the $\lesssim$ notation ignores lower order terms.
Here, $\Bt$ is due to the error in estimating the optimum outcome and
$\Amit$  is due to the error in estimating the optimum without agent $i$.
This decomposition bounds the \valueregret{} in terms of the difference between the true values of
the agents and their upper or lower confidence bounds.
To bound the $\sum_t\Amit$ terms, we use the fact that the information obtained
in the explore phase rounds lead to a $n^2\Ttwth$ rate for both $\estmethod$ choices.
For the $\RwT$ and $\sum_t \Bt$ terms, we similarly obtain a $n\Ttwth$ rate when
$\estmethod=\estee$ and a $n\sqrtT$ rate when $\estmethod=\estopt$.

\subsection{Bounding the welfare, agent, and seller regrets}
\label{sec:theoryagent}

Finally, in this section, to better understand the behaviour of the algorithm under various
hyperparameter choices, we individually bound the
welfare, agent, and seller regrets defined in~\eqref{eqn:regret}.
While the \valueregret{} provides a bound on the welfare and seller regrets, we find that
tighter bounds are possible based on the different hyperparameters.
First, in Proposition~\ref{prop:welfare}, we bound the welfare regret.

\insertprethmspacing
\begin{proposition}
\label{prop:welfare}
Assume all agents are truthful.
Let $\RT$ be as defined in~\eqref{eqn:regret}.
The following statements hold for any \emph{$\pricemethod\in (\priceage,\pricesel)$}.
When \emph{$\estmethod=\estee$},
for all $T>2K$, 
\[\hspace{-0.05in}\EE[\RT] \leq 
        \Big(3\Vmaxit  + 10 n  \sqrtlogST \Big) K^{\nicefrac{1}{3}} T^{\nicefrac{2}{3}}
        + 4\Vmaxit n
\; \in \; \bigOtilde \Big( n K^{\nicefrac{1}{3}} T^{\nicefrac{2}{3}}\Big).
\]
Moreover, when \emph{$\estmethod=\estopt$},
for all $T>2K$, the welfare regret satisfies,
\begin{align*}
\EE[\RT] \leq\;&  9 n \sqrt{|\Scal|T\logST}
    + 3\Vmaxit K^{\nicefrac{1}{3}}T^{\nicefrac{2}{3}} + 6\Vmaxit n
\; \in \; \bigOtilde \left( n \sqrtST + \Vmaxit\KotTtt\right).
\end{align*}
The above results imply that in both cases Algorithm~\ref{alg:vcglearn} is asymptotically efficient.
\end{proposition}

When $\estmethod=\estopt$, the welfare regret is $\bigOtilde(n\sqrtST + \Vmax\KotTtt)$
whereas, when $\estmethod=\estee$, it is $\bigOtilde(n\KotTtt)$.
When $\Vmax \in \littleO(n)$, the former is better
(recall from Section~\ref{sec:setup}, that the maximum welfare $\Vmax\in\bigO(n)$,
but could be much smaller).
More precisely, there are two factors contributing to the welfare regret: first, the rounds
spent in the exploration phase during which the instantaneous regret may be arbitrarily
bad; second,
the effects of the estimation errors of the values.
For both choices of $\estmethod$, the former can be bound by  $\Vmax\KotTtt$.
In contrast, when $\estmethod=\estopt$, the latter  can be bound by $n\sqrtST$
as we use data from all the rounds,
whereas when $\estmethod=\estee$, it can only be bound by $n\KotTtt$.
We will see this effect empirically as well,
with $\estmethod=\estopt$ performing significantly better than $\estmethod=\estee$.

Next we will consider the agent and seller regrets.
For this, we define $\valiexpl$ below, which can be used to bound the instantaneous
regret of agent $i$ during the exploration phase,
i.e. $\utiliopt - \utilit \leq \valiexpl$.
If the agent prefers any allocation $s\in\Scal$ for free than paying the
VCG price~\eqref{eqn:vcg} for the socially optimal outcome,
she will incur no regret during the exploration rounds,
and correspondingly, $\valiexpl=0$.
\begin{align*}
\numberthis
\label{eqn:valiexpl}
\valiexpl = \max\big(\vali(\allocopt) - \priceiopt - \min_s\vali(s), \;0\big).
\end{align*}
Proposition~\ref{prop:agentseller} bounds the agent and seller regrets for the
different $\estmethod,\pricemethod$ choices.


\insertprethmspacing
\begin{proposition}
\label{prop:agentseller}
Assume all agents are truthful.
Let $\RiT$ and $\RaT$ be as defined in~\eqref{eqn:regret}.
Let $\kappaiv=1$ if agent $i$ participates by rewards and $0$ if she participates by bids.
The following statements hold after $T>2K$ rounds for the
\emph{$\estmethod,\pricemethod$} choices specified.

\textbf{1.}
Let \emph{$\estmethod=\estee$}. Then, when \emph{$\pricemethod=\priceage$}, we have
\begingroup
\allowdisplaybreaks
\begin{align*}
 \EE[\RiT] &\leq 
       \big(3\valiexpl  + 10 \sigma \kappaiv \sqrtlogST \,\big) K^{\nicefrac{1}{3}} T^{\nicefrac{2}{3}} + 4n
\; \in \; \bigOtilde \left( K^{\nicefrac{1}{3}} T^{\nicefrac{2}{3}}\right)
\\
 \EE[\RmechT] &\leq 
       \big(3\Vmaxit  + 20\sigma n^2  \sqrtlogST \,\big) K^{\nicefrac{1}{3}} T^{\nicefrac{2}{3}} +
4\Vmaxit n
\; \in \; \bigOtilde \left( n^2 K^{\nicefrac{1}{3}} T^{\nicefrac{2}{3}}\right)
\end{align*}
\endgroup
If \emph{$\pricemethod=\pricesel$}, we have
\begingroup
\allowdisplaybreaks
\begin{align*}
 \EE[\RiT] &\leq 
       \big(3\valiexpl  + 20 \sigma n\sqrtlogST \,\big) K^{\nicefrac{1}{3}} T^{\nicefrac{2}{3}} + 4n
\; \in \; \bigOtilde \left( n K^{\nicefrac{1}{3}} T^{\nicefrac{2}{3}}\right)
,
\\
 \EE[\RmechT] &\leq 
       3\Vmaxit K^{\nicefrac{1}{3}} T^{\nicefrac{2}{3}} + 4\Vmaxit n
\; \in \; \bigOtilde \left( \Vmaxit K^{\nicefrac{1}{3}} T^{\nicefrac{2}{3}}\right)
\end{align*}
\endgroup

\textbf{2.}
Let \emph{$\estmethod=\estopt$}. Then, when \emph{$\pricemethod=\priceage$}, we have
\begingroup
\allowdisplaybreaks
\begin{align*}
  \EE[\RiT] &\leq 
       9 \sigma \kappaiv \sqrt{|\Scal|T \logST} + 3\valiexpl K^{\nicefrac{1}{3}}
T^{\nicefrac{2}{3}} + 6n
\; \in \; \bigOtilde \left(\sqrtST + \valiexpl\KotTtt\right)
\\[0.05in]
 \EE[\RmechT] &\leq 
       9\sigma n^2 \sqrt{|\Scal|T \logST}
        + \big(3\Vmaxit  + 10\sigma n^2  \sqrtlogST \,\big) K^{\nicefrac{1}{3}} T^{\nicefrac{2}{3}}
    + 6\Vmaxit n
\\
\; & \hspace{0.1in} \in \; \bigOtilde \left(n^2\sqrtST + n^2\KotTtt\right).
\end{align*}
\endgroup
If \emph{$\pricemethod=\pricesel$}, we have
\begingroup
\allowdisplaybreaks
\begin{align*}
  \EE[\RiT] &\leq 
       9\sigma n \sqrt{|\Scal|T \logST}  +  \big(3\valiexpl + 20 \sigma n\sqrtlogST\,\big) K^{\nicefrac{1}{3}}
T^{\nicefrac{2}{3}} + 6n
\\
&\hspace{0.1in}\in \; \bigOtilde \left(n\sqrtST +n\KotTtt\right),
\\[0.05in]
 \EE[\RmechT] &\leq 
       3\Vmaxit K^{\nicefrac{1}{3}} T^{\nicefrac{2}{3}} + 6\Vmaxit n
\; \in \; \bigOtilde \left(\Vmaxit\KotTtt\right).
\end{align*}
\endgroup
\end{proposition}

While, generally speaking, the agent and seller regrets scale at rate $\Ttwth$,
the dependence on other problem parameters are determined by the choices for $\estmethod$ and
$\pricemethod$.
First consider the case $\estmethod=\estee$.
If we choose $\pricemethod=\pricesel$, which, as we explained before, is favourable to the seller,
the seller's regret scales at rate $\Vmax\KotTtt$, with at most linear dependence on $n$.
However, this is disadvantageous for an agent---her regret and asymptotic individual rationality
bounds (Theorem~\ref{thm:ir}) scale linearly with $n$.
On the other hand, if we choose $\pricemethod=\priceage$, then the agent regret is the smallest,
but the seller suffers some disadvantageous consequences.
Since $\vlbjt\leq\vubjt$, in line~\ref{lin:pricing} of Algorithm~\ref{alg:vcglearn},
$\priceit$ could be negative,
i.e., the seller makes a payment to the customer.
This violates the
no-positive-transfers property which is considered desirable in mechanism design.
The seller's regret is also poor, with $n^2$ scaling.
We may draw similar conclusions when $\estmethod=\estopt$, with the main difference being
that some terms can be bounded by $\sqrtT$ rates.
It is also worth noting that when $\estmethod=\estopt$,
for agents for whom $\valiexpl=0$, we achieve $\sqrt{T}$ regret.

It is worth observing that
while the welfare regret is simply the sum of the agent and seller regrets
$\RT = \RmechT + \sum_{i=1}^n \RiT$ (see~\eqref{eqn:regret}),
the bounds for $\RT$ given in Proposition~\ref{prop:welfare} is smaller
than the sum of the agent and seller bounds in Proposition~\ref{prop:welfare}
for all $\estmethod$ choices.
For instance, when $\estmethod=\estee$, we have
$\RT \in \bigOtilde(n\KotTtt)$,
but naively summing the bounds on the agent and seller regrets would yield a bound
$\bigOtilde(n^2\KotTtt)$.
This discrepancy can be explained by the fact that
 the prices do not affect the welfare and therefore the error in estimating the
prices need not be accounted for in the welfare regret.
However, the agent and seller utilities depend on the price, and consequently their regret bounds
should account for this error.
As we explained in Section~\ref{sec:lowerbound}, estimating the prices is one of the main
bottlenecks in this set up.
We are able to bound the welfare regret separately and
obtain a better bound than the sum of individual regrets.
For example, in our simulation in Figure~\ref{fig:figsimulation}, the regret of the mechanism and
the first agent are fairly large while the regret of many of the other agents is negative.
This highlights the fact that the regret of any one agent or the seller might be large due to
the error in estimating the prices, even though the sum of these regret terms, which is the welfare
regret, is small.

\subsection{Discussion}

It is worth contextualising the above results with prior work on mechanism design with bandit
feedback in the online advertising setting.
As explained in Section~\ref{sec:related}, these settings, where an agent submits a single bid ahead
of time and the stochasticity is observed by the mechanism on each round, is different
from ours, where the mechanism needs to rely on
the agents to report their values on each round.
In a fixed-horizon version of this problem,~\citet{babaioff2014characterizing}
describe an almost surely truthful mechanism with $\Ttwth$ welfare regret and
~\citet{devanur2009price}
describe an almost surely truthful mechanism with $\Ttwth$ seller regret.
While they focus on a simpler problem and provide stronger truthfulness guarantees,
it is worth noting that both works
use an explore-then-commit style algorithm to guarantee truthfulness.
~\citet{babaioff2015truthful} describe a
truthful-in-expectation mechanism with $\sqrt{T}$ welfare regret.
However, they do not bound the agent and seller regrets.


Finally, 
we note that our algorithm and analysis assumes that seller values are known.
If this is unknown, one can define lower and upper confidence bounds for the seller similar
to~\eqref{eqn:learnterms} and use them in
Algorithm~\ref{alg:vcglearn} in place of $\valmech$, similar to those of the agents.
While $\Ttwth$ rates are still possible, there are additional considerations.
First, in many applications, it may not be reasonable to assume that this distribution
has the same sub-Gaussian constant $\sigma$ (e.g. PaaS);
the variance of the seller might scale with $n$ and this will invariably  be reflected in
the regret bounds, including that of the agents.
Second, since $\Alloc$ may be much larger than $\Scal$, this results in long exploration phases
and worse regret bounds reflected via the parameter $K$.

%% file: experiments.tex
\section{A Simulation}
\label{sec:experiments}

We present some simulation results in a single-parameter single-item environment.
Here, ten agents are competing for a single item and all of them are participating by rewards.
When an agent receives the item, her value is drawn stochastically from a $\Ncal(\mu, 0.5)$
distribution where $\mu$ is
chosen uniformly on a grid in the interval $(0.2, 0.9)$.
Agent 1 has a value of $0.9$ for receiving the item
(and will be the agent who receives the item if values are known)
and agent 10 has a value of $0.2$.
If an agent does not receive the item, their value is non-stochastically zero.
Observe that this is environment is rather noisy---the variance of the
reward distribution is large when compared to the range of the values of the agents.
The game is repeated for $3000$ rounds.

We have shown the pseudo-regrets for the welfare, the seller, and some of the agents in
Figure~\ref{fig:figsimulation} for all possible choices of the $\estmethod$ and $\pricemethod$
hyperparameters.
As we see, $\estmethod=\estopt$ performs better than $\estmethod=\estee$ on all plots as it uses
all the data.
Moreover, we see that the agents have lower regret when $\pricemethod=\priceage$
than when $\pricemethod=\pricesel$, and vice versa for the seller.
The regrets of agents 2 to 10 decrease indefinitely leading to negative regret
 since their utility at the socially optimal
outcome is zero,
but they occasionally get the item assigned to them during the exploration phase.

\insertFigSimulation

%% file: lb_app.tex
\section{Proof of Theorem~\ref{thm:lowerbound}}
\label{sec:lb_app}

In this section, we present our proof of the lower bound in Section~\ref{sec:lowerbound}.
We will first describe notation and definitions that will be used throughout
our proofs in Sections~\ref{sec:lb_app} and \ref{sec:vcglearnproofs}.

\paragraph{Notation:}
$\EE, \PP$ will denote expectations and probabilities.
$\EEt, \PPt$ will denote expectation and probability when conditioned on observations
up to time $t-1$; for example, $\PPt(\cdot) = \PP(\cdot|D_t)$, where
$D_t = \{\sil, \Yil\}_{i\leq n, \ell\leq t-1}$.

Recall that $\allocopt = \argmax_\alloc \Val(\alloc)$ is the socially optimal outcome.
Let $\siopt = \mapfuni(\allocopt)$ be the allocation for agent $i$ at the optimum.
Similarly, $\Valmi$ and $\allocmiopt$, defined below, will denote the welfare without
agent $i$ and its optimiser respectively.
\begin{align*}
\Valmi(\alloc) = \valmech(\alloc) + \sum_{j\neq i} \vali(\alloct),
\hspace{0.35in}
\allocmiopt = \argmax_{\alloc\in\Alloc}\Valmi(\alloc).
\numberthis
\label{eqn:Valmidefn}
\end{align*}

We will first state the following fact, which is straightforward
to verify, regarding agent and seller utilities in the
VCG mechanism.

\insertprethmspacing
\begin{fact}
\label{fac:vcgutils}
When the outcome and the prices are chosen according to the VCG mechanism,
\begin{align*}
\utiliopt &= \vali(\siopt) - \priceiopt = \Val(\allocopt) - \Valmi(\allocmiopt),
\\
\utilmechopt &= \valmech(\allocopt) + \sum_{i=1}^n\priceiopt =
            \sum_{i=1}^n\Valmi(\allocmiopt) - (n-1)\Val(\allocopt).
\end{align*}
\end{fact}

Our second result expresses the regret terms in~\eqref{eqn:regret}
in a way that is convenient for analysis.
For this, we define quantities $\HT, \WT$ below.
\begin{align*}
\numberthis
\label{eqn:HTWTdefn}
\HT = \frac{1}{T}\sum_{i=1}^n\sum_{t=1}^T\left(\priceit + \Valmi(\alloct)\right),
\hspace{0.5in}
\WT =  \HT - \sum_{i=1}^n \Valmi(\allocmiopt).
\end{align*}
$\HT$ is computed using observations from rounds $1$ to $T$, and can be thought of as
the algorithm's estimate of $\sum_i \Valmi(\allocopt)$ at the end of $T$ rounds.
The following lemma expresses $\RaT$ and $\RmechT$ in terms of $\RT$ and $\WT$.

\insertprethmspacing
\begin{lemma}
\label{lem:WTdecomposition}
Let $\RaT, \RmechT, \RT$ be as defined in~\eqref{eqn:regret}.
Then,
\[
\RaT = n\RT + T\WT,
\hspace{1in}
\RmechT = -(n-1)\RT - T\WT.
\]
\end{lemma}
\begin{proof}
Let $\hit=\priceit + \Valmi(\alloct)$ so that
$\HT =\frac{1}{T}\sum_{i=1}^n\sum_{t=1}^T \hit$.
For agent $i$, we can use Fact~\ref{fac:vcgutils} and the fact that
$\utilit = \vali(\alloct) - \priceit = \Val(\alloct) - \hit$ to obtain,
\[
\utiliopt - \utilit =  (\Val(\allocopt) - \Val(\allocmiopt)) - (\Val(\alloct) - \hit)
    = (\Val(\allocopt) - \Val(\alloct)) + (\hit - \Val(\allocmiopt)).
\]
Then, since $\RaT = \sum_i\sum_t (\utiliopt - \utilit)$, we have
\begin{align*}
\RaT &= \sum_{t=1}^T \sum_{i=1}^n 
\left((\Val(\allocopt) - \Val(\alloct)) + (\hit - \Val(\allocmiopt))\right)
\\
&= n \sum_{t=1}^T (\Val(\allocopt) - \Val(\alloct)) \;+\;
        T\left(\HT - \sum_{i=1}^n\Valmi(\allocmiopt)\right) \,.
\end{align*}
This proves the first claim.
For the seller, at time $t$, we observe
\begin{align*}
\utilmecht = \valmech(\alloct) + \sum_{i=1}^n\priceit 
    = \valmech(\alloct) + \sum_{i=1}^n \hit - \sum_{i=1}^n \Valmi(\alloct)
    = \sum_{i=1}^n \hit - (n-1)\Val(\alloct).
\end{align*}
As before, we can now use Fact~\ref{fac:vcgutils} to write,
\begin{align*}
\RmechT &= \sum_{t=1}^T\left(\utilmechopt - \utilmecht\right)
    = \sum_{t=1}^T\sum_{i=1}^n \left(\Valmi(\allocmiopt) - \hit\right)
    + (n-1) \sum_{t=1}^T(\Val(\alloct) - \Val(\allocopt)).
\end{align*}
The claim follows by observing that the first term in the RHS is $-T\WT$ and that the second term
is $-\RT$.
\end{proof}

Our proof of Theorem~\ref{thm:lowerbound} uses techniques from binary hypothesis
testing to establish a lower bound on the \valueregret.
For this, we begin by reviewing some facts about the KL divergence $\KL(\cdot\|\cdot)$.
Recall that for two probabilities $P, Q$ with
$Q$ absolutely continuous with respect to $P$, the KL divergence is
$\KL(P\|Q) = \EE_P\big[\big(\frac{\ud P}{\ud Q}(X)\big)\big]$.
For distributions $P, P', Q, Q'$ with $\supp(P)=\supp(Q)$ and $\supp(P')=\supp(Q')$,
the KL divergence between the product distributions satisfies
$\KL(P\times P' \|Q \times Q') = \KL(P\|Q) +  \KL(P'\|Q')$.
Additionally, for two univariate Gaussians $\Ncal(\mu_1, 1), \Ncal(\mu_2, 1)$,
we know $\KL(\Ncal(\mu_1, 1) \| \Ncal(\mu_2, 1)) = (\mu_1 - \mu_2)^2/2$.
The following result from~\citet{tsybakov08nonparametric} will be useful in our proof.

\insertprethmspacing
\begin{lemma}[\citet{tsybakov08nonparametric}, Lemmas 2.1 and 2.6]
\label{lem:hyptestlemma}
Let $P, Q$ be probabilities such that $Q$ is absolutely continuous with respect to $P$.
Let $A$ be any event. Then,
\[
P(A) + Q(A^c) \geq \frac{1}{2}\exp\big(-\KL(P\| Q)\big).
\]
\end{lemma}

We are now ready to prove the theorem.

\textbf{Proof of Theorem~\ref{thm:lowerbound}.}
Let $n> 1$.
Since the maximum is larger than an average, for any set of real numbers $\{a_i\}_i$, we have
$\max(\{a_i\}_i) \geq \sum_i\alpha_i a_i $ for any $\{\alpha_i\}_i$ such that
$\alpha_i \geq 0$, $\sum_i\alpha_i = 1$.
Using Lemma~\ref{lem:WTdecomposition} and the fact that $\RT$ is positive,
 we obtain the following two upper bounds
on $\max(n\RT, \RaT, \RmechT)$:
\begin{align*}
\max(n\RT, \RaT, \RmechT) &\geq \frac{4}{5} n\RT + \frac{1}{5}\RmechT
\geq \frac{2}{5} n\RT - \frac{1}{5}T\WT,
\\
\max(n\RT, \RaT, \RmechT) &\geq \frac{4}{5} n\RT + \frac{1}{5}T\RaT
=  n\RT + \frac{1}{5}T\WT
\geq \frac{2}{5} n\RT + \frac{1}{5}T\WT.
\end{align*}
The LHS should be larger than both of the above lower bounds.
Since $\max(a+b, a+c) = a + \max(b, c)$, we have,
\begin{align*}
\RmaxT = \max(n\RT, \RaT, \RmechT) &\geq
\frac{2}{5} n\RT + \frac{1}{5}T|\WT|
\;\, \defeq \QT.
\end{align*}
We will obtain a lower bound on $\inf_\Acal \sup_\Theta \EE\QT$ which translates to a lower bound
on the desired quantity.
Our strategy for doing so is to consider two problems in $\Theta$ and
show that any algorithm will not be able to distinguish between them.
Both problems will have the same set of outcomes
$\Alloc=\{0,1,\dots,\dots, |\Alloc|-1\}$ with $\Alloc=\Scal$ and
$|\Alloc| \geq n + 1$.
In the first problem, henceforth called $\lbprobone$, the optimal outcome is $0$ with
$\vali(0) = 1/2$ for all agents $i$.
For outcome $j\in \{1,\dots, n\}$, $\valj(j) = 0$ and $\vali(j) = 1/2$
for every other agent $i\neq j$.
For $j>n$, $\vali(j) < 1/4$ for all $i$.
When an outcome $\alloc$ is chosen, agent $i$ realises a value drawn from
$\Ncal(\vali(\alloc), 1)$.
Finally, the seller has $0$ value for all outcomes, $\valmech(j) = 0$ for all $j\in\Alloc$.
The following statements are true about problem $\lbprobone$:
\[
\Val(\allocopt) = \Val(0) = \frac{n}{2},
\hspace{0.4in}
\Valmi(\allocmiopt) = \Val(i) = \frac{n}{2} - \frac{1}{2},
\hspace{0.4in}
\sum_{i=1}^n\Valmi(\allocmiopt) = \frac{n^2}{2} - \frac{n}{2}.
\]
The second problem, henceforth called $\lbprobtwo$, is the same as $\lbprobone$ but differs in
outcomes $j\in \{1,\dots,n\}$, as shown below.
Here, the value of $\delta\in(0, 1/(2(n-1)))$ will be specified shortly. 
We have:
\[
\vali(j) = \begin{cases}
0 \hspace{0.4in} &\quad\text{if $i=j$}, \\
\frac{1}{2} + \delta \quad&\quad\text{if $i\neq j$}.
\end{cases}
\]
The following statements are true about problem $\lbprobtwo$:
\begin{align*}
&\Val(\allocopt) = \Val(0) = \frac{n}{2},
\hspace{0.4in}
\Valmi(\allocmiopt) = \Val(i) = \frac{n}{2} - \frac{1}{2} + (n-1)\delta,
\\
&\sum_{i=1}^n\Valmi(\allocmiopt) = \frac{n^2}{2} - \frac{n}{2} + n(n-1)\delta.
\end{align*}
In the above problems,
if $\delta$ is set to be too large, then it becomes easier
to distinguish between the values of different outcomes using stochastic observations
thus making the problem easy.
If $\delta$ is set to be too small, then the regret terms become small since all outcomes have
similar values.
The largest lower bound is obtained by careful choice of $\delta$ (dependent on both $n$ and $T$)
so as to balance between these two cases.

We will make the dependence of $\QT$ on the problem explicit and write $\QT(\lbprobone), \QT(\lbprobtwo)$
respectively.
Consider any algorithm in $\Acal$.
Expectations and probabilities when we execute this algorithm problem in $\lbprobone$ will
be denoted $\EEtheta, \PPtheta$, and in problem $\lbprobtwo$,
they will be denoted $\EEthetap, \PPthetap$.
Let $\Nt(\alloc)=\sum_{i=1}^{t-1} \indfone(\alloct=\alloc)$
denote the number of times outcome $\alloc\in\Alloc$ was chosen in the
first $t-1$ time steps.
With this notation, we can upper bound the welfare regret in problem
$\theta\in\{\lbprobone,\lbprobtwo\}$ as,
\[
\EE_{\theta}[\RT] = \sum_{j\geq 1} (\Val(0) - \Val(j)) \EE_\theta[\Nt(j)]
\geq \sum_{j=1}^n (\Val(0) - \Val(j)) \EE_\theta[\Nt(j)].
\]
Using the observation that the gap between the optimal and any other outcome  in
problem $\lbprobone$ is at least $1/2$,
and that when
$\HT > n^2/2 - n/2 + n(n-1)\delta/2$,  $|\WT|$ is at least $n(n-1)\delta/2$,
we obtain the following lower bound on $\EEtheta[\QT(\lbprobone)]$:
\begin{align*}
\EEtheta[\QT(\lbprobone)] &\geq\;
\frac{2n}{5}\sum_{k=1}^n \frac{1}{2} \EEtheta\left[\Ntpo(k)\right]
 \;+\; \frac{T}{5} \frac{n(n-1)\delta}{2}\,
\PPtheta\bigg(\underbrace{\vphantom{\Bigg(}
    \HT > \frac{n^2}{2} - \frac{n}{2} + \frac{1}{2}n(n-1)\delta}_{\text{event } A}\bigg), \\
&\geq\; \frac{n}{10}\Bigg(\;
\sum_{k=1}^n 2\EEtheta\left[\Ntpo(k)\right]
 \;+\; Tn(n-1)\delta\;\PPtheta(A)
\;\Bigg).
\label{eqn:lbprobonebound}
\numberthis
\end{align*}
By a similar argument regarding $\HT$ under the event $A^c$ in problem $\lbprobtwo$, we obtain the
following.
Here, we have dropped the $\EEthetap\left[\Ntpo(k)\right]$ terms which are
positive.
\begin{align*}
\EEthetap[\QT(\lbprobtwo)] &\geq\;
\frac{n}{10}
Tn(n-1)\delta\;\PPthetap(A^c).
\end{align*}
To combine these results we will apply Lemma~\ref{lem:hyptestlemma} on
$\PPtheta(A) + \PPtheta'(A^c)$ in a manner similar to~\citet{bubeck2013bounded}.
Letting $\lbproboneT, \lbprobtwoT$ denote the probability laws of the observed rewards up to round
$T$ in problems $\lbprobone, \lbprobtwo$ respectively, we obtain
\begin{align*}
\PPtheta(A) + \PPtheta'(A^c) \geq
 \frac{1}{2}\exp\left(-\KL(\lbproboneT\|\lbprobtwoT)\right)
= \frac{1}{2}\exp\left( - \frac{(n-1)\delta^2}{2}\sum_{j=1}^n\EEtheta[\NTpo(j)]  \right).
\end{align*}
For the first step we have used the fact that $A$ is measurable with respect to the 
$\sigma$-field generated by observations up to round $T$.
For the second step, observe that the outcomes $0, n+1, n+2, |\Alloc|-1$
have the same distributions under both $\lbprobone$ and $\lbprobtwo$.
For any outcome $i\in\{1,\dots, n\}$, the distribution of agent $i$ is also the same in
both problems.
For all other agents $j\neq i$,
the KL divergence between the corresponding distributions  in the two problems is $\delta^2/2$.
By combining the three previous bounds, we obtain an upper bound on 
$\EEtheta[\QT(\lbprobone)] + \EEthetap[\QT(\lbprobtwo)]$:
\begingroup
\allowdisplaybreaks
\begin{align*}
\frac{10}{n}\Big(\EEtheta&[\QT(\lbprobone)] + \EEthetap[\QT(\lbprobtwo)]\Big)
\\
&\geq\;   \sum_{k=1}^n 2\EEtheta\left[\Ntpo(k)\right]
 \;+\; T(n-1)\delta\;\Big(\PPtheta(A) + \PPtheta'(A^c)\Big), \\
&\geq\;  2 \sum_{k=1}^n \EEtheta\left[\Ntpo(k)\right]
 \;+\; \frac{1}{2}T(n-1)\delta\;
    \exp\left( - \frac{(n-1)\delta^2}{2}\sum_{k=1}^n\EEtheta[\Ntpo(k)]  \right) \\
&\geq\;  \min_{x} \left\{ 2x 
 \;+\; \frac{1}{2}T(n-1)\delta\;
    \exp\left( - \frac{(n-1)\delta^2}{2} x  \right) \right\} \\
&\geq \frac{4}{(n-1)\delta^2} \log\left( \frac{T(n-1)^2\delta^3}{8}  \right).
\end{align*}
\endgroup
Finally, we choose $\delta = \left(\frac{16}{T(n-1)^2} \right)^{1/3}$ so that the $\log$
term above can be upper bounded by a constant.
This results in the following bound:
\begin{align*}
\frac{1}{2}\Big(\EEtheta[\QT(\lbprobone)] + \EEthetap[\QT(\lbprobtwo)]\Big)
\geq\; \frac{\log(2)}{5\cdot 16^{\nicefrac{2}{3}}} \cdot \Ttwth (n-1)^{\nicefrac{4}{3}},
\end{align*}
where $\delta < \frac{1}{2(n-1)}$ is satisfied if $T>128n$. The claim follows
by observing $\;\sup_{\theta\in\Theta} \EE[\QT(\theta)]
\geq \max\left(\EEtheta[\QT(\lbprobone)], \EEthetap[\QT(\lbprobtwo)]\right)
\geq  \frac{1}{2}\EEtheta[\QT(\lbprobone)] + \frac{1}{2}\EEthetap[\QT(\lbprobtwo)]$.
\qedblack

%% file: proofs_vcglearn.tex
\section{Proofs of Results in Section~\ref{sec:analysis}}
\label{sec:vcglearnproofs}

In this section, we analyse Algorithm~\ref{alg:vcglearn}.
Section~\ref{sec:boundingpeitc} controls the probability that the confidence intervals
given in~\eqref{eqn:learnterms} capture the true values.
The proofs of Theorems~\ref{thm:truthfulness},~\ref{thm:ir}, and~\ref{thm:value}
are given in Sections~\ref{sec:prooftruthfulness},~\ref{sec:proofir}, and~\ref{sec:proofvalue}
respectively.
Sections~\ref{sec:proofwelfare} and~\ref{sec:proofagentseller} prove Propositions~\ref{prop:welfare}
and~\ref{prop:agentseller} respectively.
The bounds on $\RmechT$ and $\RwT$ will be useful in bounding $\RmaxT$.
In Section~\ref{sec:prooftechnical}, we state some technical results that are used in our proofs.
We begin with some notation and definitions.

\input{bounding_peitc}

\input{truthfulness_proof}

\subsection{Proof of Theorem~\ref{thm:ir}}
\label{sec:proofir}

In this Section, we prove the individual rationality properties of Algorithm~\ref{alg:vcglearn}.
In our proofs, we will only assume that agent $i$ is participating truthfully.
While the computed upper/lower confidence bounds of all agents
will appear in the analysis, we will not use the fact that $\vlbjt \leq \valj \leq \vubjt$
for $j\neq i$.
We will however use Lemma~\ref{lem:Eitcbound} to control the probability of
the event $\vlbit \leq \vali \leq \vubit$ for agent $i$.

\subsubsection{Proof of Theorem~\ref{thm:ir}.1}
\label{sec:proofiretc}

We will first consider the $\estmethod=\estee$ case.
For all agents, $\utilit\geq 0$ when $t\in\Expl$, so let us consider $t\notin\Expl$.
By Fact~\ref{fac:algoutils}, we have for $t\notin\Expl$,
\begin{align*}
\numberthis
\label{eqn:utilitirdecomp}
\utilit = \underbrace{\vali(\sit) - \git(\sit)}_{\ct} + \underbrace{\Gt(\alloct) -
\Fmit(\allocmit)}_{\dt}.
\end{align*}
We will first bound $\ct$.
If agent $i$ participates by bids truthfully, $\git = \vali$ and hence $\ct=0$ a.s.
To bound $\ct$ when she participates by rewards truthfully,
let $\ctildet = \max(0, \git(\sit) - \vali(\sit))$.
Clearly, $\ctildet\geq 0$ and $\ct \geq -\ctildet$.
Observing that  $\vlbit \leq \vali\leq \vubit$ under $\Eit$~\eqref{eqn:Etdefn},
and that $\git=\vubit$ when $\pricemethod=\priceage$ and
$\git=\vlbit$ when $\pricemethod=\pricesel$,
we have,
\[
\EE[\ctildet|\Eit] \leq 0 \hspace{0.10in}\text{if $\pricemethod=\pricesel$},
\hspace{0.4in}
\EE[\ctildet|\Eit] \leq 2 \betat \sigmait(\sit)
    \hspace{0.10in}\text{if $\pricemethod=\priceage$}.
\]
To bound $\dt$,
we first observe that
$\alloct=\argmax_\alloc \Gt(\alloc)$ since $\Vlbt,\Vbart,\Vubt$ are vertically shifted functions
(using the same argument used in Section~\ref{sec:prooftruthfulnessetc}).
Now, consider the case $\pricemethod=\priceage$.
Since, $\Vubt = \Vubmit + \vubit$ and $\vubit\geq\vbarit\geq 0$ (recall from~\eqref{eqn:learnterms}
that we clip $\vbarit$ between $0$ and $1$),
we have that  $\Vubit \geq \Vubmit$.
By observing $\Vlbmit\leq\Vubmit$, we have
\[
\dt = \max_\alloc \Gt - \max_\alloc\Fmit
= \max\Vubt - \max \Vlbmit \geq \max\Vubt - \max\Vubmit \geq 0.
\]
When $\pricemethod=\pricesel$, and therefore $\Gt=\Vlbt$ and $\Fmit=\Vubmit$,
 one no longer has $\dt \geq 0$ since $\Vubmit$ can be larger than
$\Vlbt$.
However, we can obtain a weaker bound of the form,
\[
\dt = 
\Vlbt(\alloct) - \Vubmit(\allocmit)
=  \max\Vubt - \max\Vubmit - 2\sum_j \sigmajt(\sjt)
\geq - 2\sqrt{2}\sigma\betat  n K^{\nicefrac{1}{3}} t^{\nicefrac{-1}{3}}.
\]
Above, the last step uses that $\Vubt \geq \Vubmit$ as before and~\eqref{eqn:sigmaitsbound} to
bound the $\sigmajt(\sjt)$ terms.

We can now bound the utilities for the various cases in the theorem for agent $i$.
First, when $\pricemethod=\pricesel$ and agent $i$ is participating by bids,
we have, $\utilit=\ct + \dt\geq 0$ a.s. for all $t\notin\Expl$.
Therefore, $\UiT \geq 0$ for all $T$ and the mechanism is individually rational for this agent.
That is, the algorithm is (almost surely) individually rational.
If agent $i$ participates by rewards truthfully, we have
\begin{align*}
\EE[-\UiT] \leq \sum_{t\notin\Expl} \EEt[-c_t|\Eit] +  \sum_{t\notin\Expl} \EEt[-d_t|\Eit]
+ \sum_{t\notin\Expl} \PPt(\Eitc).
\label{eqn:UiTirdecomp}
\end{align*}
By combining the results above and applying Lemma~\ref{lem:Eitcbound}, we obtain,
\begin{align*}
\EE[\UiT] \geq
\sum_{t\notin \Expl}\EE[\utilit] \geq
\begin{cases}
- 3\sqrt{2} \sigma\betat K^{\nicefrac{1}{3}} \Ttwth - 4, \;& \text{if $\pricemethod=\priceage$},
\\
- 3\sqrt{2} n \sigma\betat K^{\nicefrac{1}{3}} \Ttwth - 4, \;& \text{if $\pricemethod=\pricesel$}.
\end{cases}
\end{align*}
The claim follows by substituting for $\betaT$~\eqref{eqn:betat}.
\qedwhite

\subsubsection{Proof of Theorem~\ref{thm:ir}.2}
\label{sec:proofiropt}

Now we will consider the $\estmethod=\estopt$ case.
As in Section~\ref{sec:proofiretc},
we will write $\utilit = \ct + \dt$ where $\ct,\dt$ are as defined in~\eqref{eqn:utilitirdecomp},
and consider rounds $t\notin\Expl$.
First consider $\ct$.
If agent $i$ participates by bids truthfully, $\git = \vali$ and hence $\ct=0$.
To bound $\ct$ when she participates by rewards truthfully,
let $\ctildet = \max(0, \git(\sit) - \vali(\sit))\geq 0$.
Using a similar argument as above, we have
\[
\EE[\ctildet|\Eit] \leq 0 \hspace{0.10in}\text{if $\pricemethod=\pricesel$},
\hspace{0.5in}
\EE[\ctildet|\Eit] \leq 2 \betat \sigmait(\sit)
    \hspace{0.10in}\text{if $\pricemethod=\priceage$}.
\]
To bound $\dt$, first note that $\Vubt(\alloct) - \max\Vubmit \geq 0$ since
$\Vubt = \Vubmit + \vubit$ and $\vubit \geq 0$.
Therefore, when $\pricemethod=\priceage$,
\begin{align*}
&\dt = \Vubt(\alloct) - \Fmit(\allocmit) \geq \Vubt(\alloct) - \max\Vubmit \geq 0,
\end{align*}
and when $\pricemethod=\pricesel$,
\begin{align*}
\dt = \Vlbt(\alloct) - \Fmit(\allocmit) \geq \Vubt(\alloct) - \max\Vubmit -
2\betat\sum_j\sigmajt(\sjt) \geq - 2\betat\sum_j\sigmajt(\sjt).
\end{align*}
To bound the $\sum_t\sigmait(\sit)$ terms in $\ct$ and $\dt$ when $\estmethod=\estopt$, we use the
following argument which leads to a tighter upper bound bound.
\begin{align*}
\label{eqn:sqrtSTbound}
\numberthis
\sum_{t\notin\Expl} \frac{\sigmait(\sit)}{\sigma}
&=
\sum_{t\notin\Expl} \frac{1}{\sqrt{\Nit(\sit)}}
\leq
\sum_{t>K} \frac{1}{\sqrt{\Nit(\sit)}} \leq
\sum_{s\in\Scal}  \sum_{j=1}^{\NiT(s)} \frac{1}{\sqrt{j}}
\\
&\leq 2 \sum_{s\in\Scal} \sqrt{\NiT(s)}
\leq 2 \sqrt{|\Scal|(T-qK)}
\leq 2 \sqrt{|\Scal|T}.
\end{align*}
The first step simply adds more terms to the summation.
The second step observes that the summation can be written as $|\Scal|$ different summations,
one for each $s\in\Scal$.
The third step uses Lemma~\ref{lem:subounds}.
We will use the above bound in~\eqref{eqn:sqrtSTbound} elsewhere in our proofs for the
$\estmethod=\estopt$ case.

By the same argument for other agents $j$ and using the fact that $\betat \leq \betaT$ for all
$t<T$, we have
$\sum_t \betat\sigmajt(\sjt) \leq 2\sigma \betaT \sqrtST$ for all agents $j$.
Therefore, for an agent participating by bids,
$\UiT \geq 0$ if $\pricemethod=\priceage$ and
$\UiT \geq -4\betaT \sigma n \sqrtST$ if $\pricemethod=\pricesel$.
For an agent participating by rewards, we have:
\begin{align*}
\EE[\UiT] &\geq - \sum_{t\notin\Expl} \EE[\ctildet|\Eit] 
    + \sum_{t\notin\Expl} \EE[\dt]
        - \sum_{t\notin\Expl} \PP(\Eitc)
\geq \begin{cases}
-4\betaT \sigma \sqrtST  \,- 6\hspace{0.8in} &\text{if $\pricemethod=\priceage$}, \\
-4\betaT \sigma n \sqrtST  \,- 6 \hspace{0.8in} &\text{if $\pricemethod=\pricesel$}.
\end{cases}
\end{align*}
\qedwhite

\subsection{Proof of Proposition~\ref{prop:welfare}}
\label{sec:proofwelfare}

In this section, we bound the welfare regret $\RwT$.
The bounds we establish for the welfare regret here and for the seller regret in
Section~\ref{sec:proofagentseller} will be useful when we bound the VCG regret in
Section~\ref{sec:proofvalue}.
The following lemma provides a bound that will be useful in the proof of
Proposition~\ref{prop:welfare}.

\insertprethmspacing
\begin{lemma}
\label{lem:RTgenbound}
The welfare regret $\RwT$~\eqref{eqn:regret} satisfies the following bound.
\[
\EE[\RT] \espleq  3\Vmaxit K^{\nicefrac{1}{3}} T^{\nicefrac{2}{3}}
            +  2\betaT\sum_{t\notin\Expl} \sum_{i=1}^n \EEt[\sigmait(\sit)]
            + \Vmaxit\sum_{t\notin\Expl}\PPt(\Etc).
\]
\end{lemma}
\begin{proof}
Write $\RT = \sum_{t=1}^T \rt$ where $\rt = \Val(\allocopt) - \Val(\alloct)$.
Recall that $\Expl$ denotes time indices belonging to the explore phase.
We split the instantaneous regret terms to obtain,
\[
\RT = \sum_{t=1, t\in\Expl}^T \rt + \sum_{t=1, t\notin\Expl}^T \rt.
\]
First consider the second summation.
Using the notation in~\eqref{eqn:betat}, we obtain,
\begin{align*}
\EEt[\rt] \;&\leq\; 
    \EE[\rt|\Et] + \EE[\rt|\Etc]\PP(\Etc)
    \,=\,
    \EEt[\Val(\allocopt) - \Vubt(\alloct) + \Vubt(\alloct)  - \Val(\alloct)|\Et] 
                + \Vmax\PPt(\Etc)
 \\
         &\leq\; \EEt[\Val(\allocopt) - \Vubt(\allocopt) + \Vubt(\alloct)  - \Val(\alloct)| \Et]
                + \Vmax\PPt(\Etc)
        \\
         &\leq\, \EEt[\Vubt(\alloct)  - \Vlbt(\alloct)| \Etc] + \Vmax\PPt(\Etc) 
          \;\leq 2\betat\sum_{i=1}^n \EEt[\sigmait(\sit)] + \Vmax\PPt(\Etc).
\numberthis
\label{eqn:eertbreakdown}
\end{align*}
Here, the third step uses the fact that $\Vubt$ is maximised at $\alloct$.
The fourth step uses that $\Vubt \geq \Val$ and $\Vlbt \leq \Val$ under $\Et$.
Now summing over all $t$, we obtain
\begin{align*}
\EE[\RT] &\espleq \sum_{t\in\Expl}\EEt[\rt] \spplus \sum_{t\notin\Expl} \EEt[\rt]
          \espleq \sum_{t\in\Expl}\Vmax \spplus 
            \sum_{t\notin\Expl} 
                \Big(2\sum_{i=1}^n \betat\EEt[\sigmait(\sit)]  + \Vmax\PPt(\Etc)\Big).
\end{align*}
Now,
the number of terms in the first summation can be bound by
$\qT K\leq 3K^{\nicefrac{1}{3}} T^{\nicefrac{2}{3}}$ using Lemma~\ref{lem:qTbound}.
The claim follows by observing $\betat\leq\betaT$ for all $t\leq T$.
\end{proof}

\proofheader{Proof of Proposition~\ref{prop:welfare}}
We will first consider the case $\estmethod=\estee$,
and apply Lemma~\ref{lem:RTgenbound}.
By Lemma~\ref{lem:Etcbound}, we have $\sum_t\PPt(\Etc) \leq 4n$.
By following a similar argument to~\eqref{eqn:sigmaitsbound}, we obtain
$\sigmait(\sit) \leq \sqrt{2} K^{1/3} t^{-1/3}$.
Then, using Lemma~\ref{lem:subounds} to bound $\sum t^{-\nicefrac{1}{3}}$, we have
\begin{align*}
\numberthis
\label{eqn:eewelfareregbound}
\EE[\RT] \espleq  4n\Vmax + 3\Vmax K^{\nicefrac{1}{3}} T^{\nicefrac{2}{3}}
            +  3\sqrt{2}\betaT n \KotTtt.
\end{align*}
Next, consider $\estmethod=\estopt$.
In order to use Lemma~\ref{lem:RTgenbound},
we will use a similar argument as in~\eqref{eqn:sqrtSTbound} to obtain
$\sum_{t\notin\Expl} \sigmait(\sit) \leq 2\sigma\sqrtST$.
%
Next, by Lemma~\ref{lem:Etcbound}, we have $\sum_t\PPt(\Etc) \leq 6n$.
These results when applied with Lemma~\ref{lem:RTgenbound} yield:
\begin{align*}
\EE[\RT] \leq 6 \Vmax n + 4\sigma n \betaT \sqrtST + 3\KotTtt.
\numberthis
\label{eqn:ucbwelfareregbound}
\end{align*}
The claims follow by substituting for $\betaT$~\eqref{eqn:betat} in~\eqref{eqn:eewelfareregbound}
and~\eqref{eqn:ucbwelfareregbound}.
\qedwhite

\subsection{Proof of Proposition~\ref{prop:agentseller}}
\label{sec:proofagentseller}

In this section, we bound the agent and seller regrets.
First, in Lemma~\ref{lem:RiTgenbound} we provide an upper bound on the agent regret.
Recall 
that $\valiexpl = \max(\utiliopt - \min_s\vali(s), 0)$ from~\eqref{eqn:valiexpl}.
If the agent prefers receiving any item in $\Scal$ for free instead of the socially optimal outcome
at the VCG price, then this term will be $0$ and the agent does not incur any regret during
the exploration phase rounds.

\insertprethmspacing
\begin{lemma}
\label{lem:RiTgenbound}
Consider any agent $i$ and define $\at, \bt$ as follows for $t\geq 0$.
\begin{align*}
\at = \Fmit(\allocmit) - \Vlbt(\allocmiopt), 
\hspace{0.5in}
\bt &= \git(\sit) - \vlbit(\sit) + \Vubt(\alloct) - \Gt(\alloct).
\end{align*}
Then, the following bound holds on the regret of agent $i$.
\begin{align*}
\EE[\RiT] \;&\leq\; 
   3\valiexpl K^{\nicefrac{1}{3}} T^{\nicefrac{2}{3}} \;+\;
    \sum_{t\notin\Expl} \EEt[\at|\Et] \;+\;
    \sum_{t\notin\Expl} \EEt[\bt|\Et] \,+\,
    \sum_{t\notin\Expl}\PPt(\Etc).
\end{align*}
\end{lemma}
\begin{proof}
As above, we will write
$\RiT = \sum_{t\in\Expl} \rit + \sum_{t\notin\Expl} \rit$,
where $\rit = \utiliopt - \utilit$.
We will first bound the second summation  in expectation.
For $t\notin\Expl$, we use Facts~\ref{fac:vcgutils} and~\ref{fac:algoutils} to obtain,
\begin{align*}
\rit = \git(\sit) - \vali(\sit) + \Val(\allocopt) - \Gt(\alloct) + 
        \Fmit(\allocmit) - \Valmi(\allocmiopt).
\end{align*}
Under $\Et$, the following are true;
$\vali(\sit) \geq \vlbit(\sit)$;
$\Val(\allocopt) \leq \Vubt(\allocopt)\leq \Vubt(\alloct)$ since $\alloct$ maximises $\Vubt$,
and $\Valmi(\allocmiopt) \geq \Valmi(\allocmit) \geq \Vlbt(\allocmit)$ since $\allocmiopt$ maximises
$\Valmi$.
This leads us to,
\begin{align*}
\EEt[\rit] \;&\leq\; 
    \EEt[\underbrace{\Fmit(\allocmit) - \Vlbt(\allocmit)}_{\at}|\Et] \,+\,
    \EEt[\underbrace{\git(\sit) - \vlbit(\sit) + \Vubt(\alloct) - \Gt(\alloct)}_{\bt}|\Et]
    \,+\, \PPt(\Etc).
\end{align*}
Summing over all $t$ yields the following bound on the agent regret:
\begin{align*}
\label{eqn:RiTbound}
\numberthis
\EE[\RiT] \;&\leq\; 
            \sum_{t\in\Expl} \valiexpl \;+\; 
        \sum_{t\notin\Expl} \EEt[\rit|\Et]  \;+\; \sum_{t\notin\Expl} \PPt(\Etc) \\
   &\leq\; 3\valiexpl K^{\nicefrac{1}{3}} T^{\nicefrac{2}{3}} \;+\;
    \sum_{t\notin\Expl} \EEt[\at|\Et] \;+\;
    \sum_{t\notin\Expl} \EEt[\bt|\Et] \,+\,
    \sum_{t\notin\Expl} \PPt(\Etc).
\end{align*}
Here, for the first summation, we applied Lemma~\ref{lem:qTbound} to obtain
$\valiexpl\qT K \leq 3\valiexpl K^{\nicefrac{1}{3}} T^{\nicefrac{2}{3}}$.
\end{proof}

When applying the above lemma,
the value of hyperparameter $\pricemethod$ in Algorithm~\ref{alg:vcglearn} will
decide the bounds for $\at,\bt$ respectively.
Additionally, note that $\at, \bt$ are measurable with respect to the
sigma field generated by observations up to time $t-1$.
Hence, $\EEt[\at], \EEt[\bt]$ are deterministic quantities.
Our next lemma bounds the seller regret.
For this, we first
define $\Amit$, for $i\in\{1,\dots,n\}$ and $\Bt$ as follows for $t\geq 0$:
\begin{align*}
\Amit = \Valmi(\allocmiopt) - \Fmit(\allocmiopt), 
\hspace{0.5in}
\Bt &= \Gt(\alloct) - \Val(\allocopt).
\label{eqn:AtBtdefn}
\numberthis
\end{align*}

\insertprethmspacing
\begin{lemma}
\label{lem:RmechTgenbound}
Let $\Amit$ and $\Bt$ be as defined in~\eqref{eqn:AtBtdefn}.
Then, the following bound holds on the regret of the seller~\eqref{eqn:regret}:
\begin{align*}
\EE[\RmechT] \;&\leq\; 
   3\Vmaxit K^{\nicefrac{1}{3}} T^{\nicefrac{2}{3}} \;+\;
    \sum_{t\notin\Expl} \sum_{i=1}^n \EEt[\Amit|\Et] \;+\;
    (n-1)\sum_{t\notin\Expl} \EEt[\Bt|\Et] \,+\,
    \Vmaxit \sum_{t\notin\Expl}\PPt(\Etc).
\end{align*}
\end{lemma}
\begin{proof}
Write
$\RmechT = \sum_{t\in\Expl}^T \rmecht + \sum_{t\notin\Expl}^T \rmecht$,
where $\rmecht = \utilmechopt - \utilmecht$.
To bound the second summation,  we use Facts~\ref{fac:vcgutils} and~\ref{fac:algoutils} to obtain
the following expression for $\rmecht$ when $t\notin\Expl$:
\begin{align*}
\rmecht = \utilmechopt - \utilmecht \;=\;
        \sum_{i=1}^n\left(\Valmi(\allocmiopt) - \Fmit(\allocmit)\right)
         \,+\,    (n-1)\left(\Gt(\alloct) - \Val(\allocopt)\right).
\label{eqn:rmechtdecomp}
\numberthis
\end{align*}
Hence, for $t\notin\Expl$,
we have $\EEt[\rmecht] \leq \EEt[\rmecht|\Et] + \EEt[\rmecht|\Etc]\PP(\Etc)
\leq \sum_{i=1}^n \EE[\Amit|\Et] + (n-1)\EE[\Bt|\Et] + \Vmax\PPt(\Etc)$.
Summing over all $t$ yields the following bound on the seller regret:
\begin{align*}
\EE[\RmechT] \;&\leq\; 
            \sum_{t\in\Expl} \Vmax \;+\; 
        \sum_{t\notin\Expl} \EEt[\rmecht] 
    \\
   &\leq\; 3\Vmax K^{\nicefrac{1}{3}} T^{\nicefrac{2}{3}} \;+\;
    \sum_{t\notin\Expl} \sum_{i=1}^n \EEt[\Amit|\Et] \;+\;
    (n-1)\sum_{t\notin\Expl} \EEt[\Bt|\Et] \,+\,
    \Vmax \sum_{t\notin\Expl}\PPt(\Etc).
\end{align*}
Here, for the first summation, we applied Lemma~\ref{lem:qTbound} to obtain
$\qT K \leq 3 K^{\nicefrac{1}{3}} T^{\nicefrac{2}{3}}$.
\end{proof}

\subsubsection{Proof of Proposition~\ref{prop:agentseller}.1, agent regret}
\label{sec:agentetcproof}

Let us first consider $\RiT$, the regret for agent $i$, when $\estmethod=\estee$.
We will apply Lemma~\ref{lem:RiTgenbound} and proceed to control the $\at,\bt$ terms
for the two different choices for $\pricemethod$ when $\Et$ holds.
First consider $\at$.
When $\pricemethod=\priceage$, we have $\Fmit = \Vlbmit$ and therefore $\at = 0$ a.s.
When $\pricemethod=\pricesel$, we have $\Fmit = \Vubmit$ and therefore under $\Et$,
\begin{align*}
\at \espeq \Vubmit(\allocmit) - \Vlbmit(\allocmit) \espeq \sum_{i=1}^n
2\betat\sigmait(\mapfuni(\allocmit)) \;\leq\;
      2\sqrt{2}\betat n \sigma K^{\nicefrac{1}{3}} t^{\nicefrac{-1}{3}}.
\end{align*}
The last step uses an argument similar to~\eqref{eqn:sigmaitsbound}
 followed by Lemma~\ref{lem:qTbound}.
Along with Lemma~\ref{lem:subounds}, we have the following bounds on the sum of $\at$'s:
\begin{align*}
\numberthis
\label{eqn:atagent}
\sum_{t\notin\Expl} \EE[\at|\Et] \,\leq\;
    \begin{cases}
        \;0 &\text{if $\pricemethod=\priceage$}, \\
        \;3\sqrt{2}\betaT n \sigma K^{\nicefrac{1}{3}} T^{\nicefrac{2}{3}}
            \hspace{0.7in}&\text{if $f=\pricesel$}.
    \end{cases}
\end{align*}
Now consider $\bt$ and assume the agent participates by rewards.
When $\pricemethod=\priceage$, $\git=\vubit$ and $\Gmit=\Vubmit$.
We therefore have,
$\bt = \vubit(\sit) - \vlbit(\sit) = 2\betat\sigmait(\sit)
\leq 2\sqrt{2}\sigma \betat \Kottmot$ under $\Et$.
Similarly,
when $\pricemethod=\pricesel$, $\git=\vlbit$ and $\Gmit=\Vlbmit$, which results in
$\bt= \Vubt(\alloct) - \Vlbt(\alloct) = 2\betat\sum_i\sigmait(\alloct)
\leq 2\sqrt{2}\sigma \betat n \Kottmot$.
%
For an agent participating by bids $\vlbit = \git = \vali$.
The only change in the analysis is that now $\bt = \Vubt(\alloct) - \Gt(\alloct)$ which
can be bound in a similar fashion to above depending on the value of
$\pricemethod$.
Accounting for these considerations, and using
Lemma~\ref{lem:subounds}, we have the following bounds on the sum of $\bt$'s:
\begin{align*}
\numberthis
\label{eqn:btagent}
\sum_{t\notin\Expl} \EE[\bt|\Et] \,\leq\;
    \begin{cases}
        \;3\sqrt{2}\betaT \sigma \kappaiv K^{\nicefrac{1}{3}} T^{\nicefrac{2}{3}}
            \hspace{0.2in}&\text{if $\pricemethod=\priceage$}, \\
        \;3\sqrt{2}\betaT \sigma n K^{\nicefrac{1}{3}} T^{\nicefrac{2}{3}}
            \hspace{0.2in}&\text{if $\pricemethod=\pricesel$}.
    \end{cases}
\end{align*}
Recall that $\kappaiv=1$ if the agent participates by rewards and $0$ if she participates by bids.
%
Finally, an application of Lemma~\ref{lem:Etcbound}
leads to the following bounds for the agent regret:
\begin{align*}
 \EE[\RiT] &\leq
    \begin{cases}
       4n +  \big(3\valiexpl  + 3\sqrt{2}\kappaiv \sigma \betaT \big) K^{\nicefrac{1}{3}} T^{\nicefrac{2}{3}},
            \hspace{0.2in}&\text{if $\pricemethod=\priceage$}, \\
       4n +  \big(3\valiexpl  + 6\sqrt{2}\sigma \betaT n \,\big) K^{\nicefrac{1}{3}} T^{\nicefrac{2}{3}},
            \hspace{0.2in}&\text{if $\pricemethod=\pricesel$}.
    \end{cases}
\end{align*}

\subsubsection{Proof of Proposition~\ref{prop:agentseller}.1, seller regret}
\label{sec:selleretcproof}

Now, we will consider $\RmechT$, the seller regret, when $\estmethod=\estee$.
We will apply Lemma~\ref{lem:RmechTgenbound} and proceed to control the $\Amit, \Bt$ terms
for the two different choices for $\pricemethod$ under $\Et$.
First consider the $\Amit$ terms.
When $\pricemethod=\pricesel$, then $\Fmit=\Vubmit$ is an upper bound for $\Valmi$ under $\Et$.
Hence, $\Amit = \max\Valmi - \max\Vubmit \leq 0$.
When $\pricemethod=\priceage$, under $\Et$, we obtain the following uniform bound on
$\Valmi(\alloc) - \Fmit(\alloc)$:
\begin{align*}
\forall\alloc\in\Alloc, \;
\Valmi(\alloc) - \Fmit(\alloc)
&\leq \Vubmit(\alloc) - \Vlbmit(\alloc) =
\sum_{j\neq i} 2\betat\sigmajt(\mapfunj(\alloc))
\\
&\leq 2\sqrt{2}\sigma\betat (n-1) K^{\nicefrac{1}{3}} t^{\nicefrac{-1}{3}}.
\label{eqn:Amituniformbound}
\numberthis
\end{align*}
Here, the last step uses an argument similar to~\eqref{eqn:sigmaitsbound}.
Hence, by  Lemma~\ref{lem:boundmax}, we have
$\Amit = \max \Valmi - \max \Fmit
\leq 2\sqrt{2}\sigma\betat (n-1) K^{\nicefrac{1}{3}} t^{\nicefrac{-1}{3}}$ under $\Et$.
Along with Lemma~\ref{lem:subounds}, we obtain the following.
\begin{align*}
\label{eqn:Atmech}
\numberthis
\sum_{i=1}^n \sum_{t\notin\Expl} \EE[\Amit|\Et] \espleq
    \begin{cases}
        \;3\sqrt{2}\sigma\betaT n(n-1) K^{\nicefrac{1}{3}} T^{\nicefrac{2}{3}}
        \hspace{0.2in} &\text{if $\pricemethod=\priceage$}, \\
        \;0 &\text{if $\pricemethod=\pricesel$}.
    \end{cases}
\end{align*}
Now, we turn to $\Bt$.
For this, note that under $\Et$, $\Val(\allocopt) \geq \Val(\alloct) \geq \Vlbt(\alloct)$.
When $\pricemethod=\pricesel$, we have $\Gt=\Vlbt$ and therefore $\Bt \leq 0$.
When $\pricemethod=\priceage$, we have $\Gt=\Vubt$ and therefore,
\begin{align*}
\Bt \;\leq\; \Vubt(\alloct) - \Vlbt(\alloct)
        \;=\; 2\betat\sum_{i=1}^n\sigmait(\sit)
        \;\leq\; 2\sqrt{2}\sigma n \betat K^{\nicefrac{1}{3}} t^{\nicefrac{-1}{3}}.
\end{align*}
This yields the following bounds for the sum of $\Bt$'s.
\begin{align*}
\label{eqn:Btmech}
\numberthis
(n-1)\sum_{t\notin\Expl} \EE[\Bt|\Et] \espleq
    \begin{cases}
        \;3\sqrt{2}\sigma\betaT n(n-1) K^{\nicefrac{1}{3}} T^{\nicefrac{2}{3}}
        \hspace{0.2in} &\text{if $\pricemethod=\priceage$}, \\
        \;0 &\text{if $\pricemethod=\pricesel$}.
    \end{cases}
\end{align*}
Combining the above results with Lemma~\ref{lem:RmechTgenbound} and Lemma~\ref{lem:Etcbound}
leads to the following bounds for the seller regret:
\begin{align*}
 \EE[\RmechT] &\leq
    \begin{cases}
       4n\Vmax + \Big(3\Vmax  + 6\sqrt{2}\sigma \betaT n(n-1) \Big) K^{\nicefrac{1}{3}} T^{\nicefrac{2}{3}}.
            \hspace{0.2in}&\text{if $\pricemethod=\priceage$}, \\
       4n\Vmax + 3\Vmax K^{\nicefrac{1}{3}} T^{\nicefrac{2}{3}}.
            \hspace{0.2in}&\text{if $\pricemethod=\pricesel$}.
    \end{cases}
\end{align*}

\subsubsection{Proof of Proposition~\ref{prop:agentseller}.2, agent regret}
\label{sec:agentoptproof}

Next, we will consider agent $i$'s regret $\RiT$, when $\estmethod=\estopt$.
As in Section~\ref{sec:agentetcproof},
we will use Lemma~\ref{lem:RiTgenbound} and control the $\at,\bt$ terms
for the two different choices for $\pricemethod$ under $\Et$.
First, $\at$ is bounded identically to obtain the upper bound in~\eqref{eqn:atagent};
this uses the fact that even when $\estmethod=\estopt$, there will have been $\qT$ exploration
phases by round $T$.

Next, consider $\bt$.
When $\pricemethod=\priceage$, $\git=\vubit$ and $\Gmit=\Vubmit$.
We therefore have,
$\bt = \vubit(\sit) - \vlbit(\sit) = 2\betat\sigmait(\sit)$ if the agent is participating
by rewards and $\bt=0$ if she is participating by bids.
Similarly, when $\pricemethod=\pricesel$,
$\bt= \Vubt(\alloct) - \Vlbt(\alloct) \leq 2 \betat\sum_i\sigmait(\alloct)$ under $\Et$.
Using a similar argument to~\eqref{eqn:sqrtSTbound},
we obtain the following bounds on the sum of $\bt$'s when $t\notin\Expl$:
\begin{align*}
\numberthis
\label{eqn:btagentopt}
\sum_{t\notin\Expl} \EE[\bt|\Et] \,\leq\;
    \begin{cases}
        \;4 \kappaiv \betaT \sigma \sqrtST
            \hspace{0.2in}&\text{if $\pricemethod=\priceage$}, \\
        \;4\betaT \sigma n \sqrtST
            \hspace{0.2in}&\text{if $\pricemethod=\pricesel$}.
    \end{cases}
\end{align*}
Combining the above results with Lemma~\ref{lem:RiTgenbound} and Lemma~\ref{lem:Etcbound}
leads to the following bounds for the agent regret:
\begin{align*}
 \EE[\RiT] &\leq
    \begin{cases}
       6n  + 4\sigma\betaT\kappaiv\sqrtST + 3\valiexpl K^{\nicefrac{1}{3}} T^{\nicefrac{2}{3}},
            \hspace{0.2in}&\text{if $\pricemethod=\priceage$}, \\
       6n + 4\sigma\betaT n \sqrtST + (3\valiexpl + 3\sqrt{2}\sigma\betaT n)K^{\nicefrac{1}{3}} T^{\nicefrac{2}{3}}, 
            \hspace{0.2in}&\text{if $\pricemethod=\pricesel$}.
    \end{cases}
\end{align*}

\subsubsection{Proof of Proposition~\ref{prop:agentseller}.2, seller regret}
\label{sec:selleroptproof}

Finally, we will consider the seller regret $\RmechT$, when $\estmethod=\estopt$.
Following along the same lines as Section~\ref{sec:selleretcproof}, we will use
Lemma~\ref{lem:RmechTgenbound} to control the seller regret, and moreover,
use the expression in~\eqref{eqn:Atmech} to bound the $\EE[\Amit|\Et]$ terms.
The same bounding technique can be used since, even when $\estmethod=\estopt$,
 there will have been $\qT$ exploration phases by round $T$.

To bound the $\Bt$ terms, we first observe that under $\Et$, $\Val(\allocopt) \geq \Val(\alloct)
\geq \Vlbt(\alloct)$.
When $\pricemethod=\pricesel$, we have $\Gt=\Vlbt$, and therefore $\Bt \leq 0$.
When $\pricemethod=\priceage$, we have $\Gt=\Vubt$, and therefore  $\Bt \leq \Vubt(\alloct) - \Vlbt(\alloct) =
2\betat\sum_i\sigmait(\sit)$.
Putting these results together and using a similar argument to~\eqref{eqn:sqrtSTbound},
we obtain the following bounds on the sum of $\Bt$'s:
\begin{align*}
\label{eqn:Btmechopt}
\numberthis
(n-1)\sum_{t\notin\Expl} \EE[\Bt|\Et] \espleq
    \begin{cases}
        \;4\sigma\betaT n(n-1) \sqrtST
        \hspace{0.2in} &\text{if $\pricemethod=\priceage$}, \\
        \;0 &\text{if $\pricemethod=\pricesel$}.
    \end{cases}
\end{align*}
Combining the above results with Lemma~\ref{lem:RmechTgenbound} and Lemma~\ref{lem:Etcbound}
leads to the following bounds for the seller regret:
\begin{align*}
 \EE[\RmechT] &\leq
    \begin{cases}
       6n\Vmax + 4\sigma\betaT n(n-1) \sqrtST + \Big(3\Vmax  + 3\sqrt{2}\sigma \betaT n(n-1) \Big)
K^{\nicefrac{1}{3}} T^{\nicefrac{2}{3}},
            \hspace{0.2in}&\text{if $\pricemethod=\priceage$}, \\
       6n\Vmax + 3\Vmax K^{\nicefrac{1}{3}} T^{\nicefrac{2}{3}},
            \hspace{0.2in}&\text{if $\pricemethod=\pricesel$}.
    \end{cases}
\end{align*}

\subsection{Proof of Theorem~\ref{thm:value}}
\label{sec:proofvalue}

We now bound $\RmaxT$.
First, we provide a bound on $\RmaxT$ in terms of $\Amit$ and $\Bt$ defined in~\eqref{eqn:AtBtdefn}.

\insertprethmspacing
\begin{lemma}
\label{lem:MaxRegretgenbound}
Let $\Amit, \Bt$ be as defined in~\eqref{eqn:AtBtdefn}.
Then, the following bound holds on the \valueregret{} $\RmaxT$ defined in~\eqref{eqn:regret}.
\emph{
\begin{align*}
\EE[\RmaxT] \leq \;&
    2\sum_{i=1}^n \sum_{t\notin\Expl} \EE[|\Amit|\,|\Et] +
    2(n-1)\sum_{t\notin\Expl} \EE[|\Bt|\,|\Et] +
    6\Vmax\KotTtt
    \\
    & + (n+1)\RT +
    2\Vmax\sum_{t\notin\Expl}\PPt(\Etc).
\end{align*}
}
\end{lemma}
\begin{proof}
Recall that $\RT$ is always non-negative while $\RaT$ and $\RmechT$ may be positive or negative.
From Lemma~\ref{lem:WTdecomposition}, we have $\RaT + \RmechT = \RT$.
Since the maximum is smaller than the sum, we have
\begin{align*}
\RmaxT = \max(n\RT, \RaT, \RmechT)
&\leq n\RT + |\RaT| + |\RmechT|
\leq (n+1)\RT + 2|\RmechT|.
\numberthis
\label{eqn:maxregbound}
\end{align*}
By the triangle inequality, we obtain the following bound on $|\RmechT|$, similar
to Lemma~\ref{lem:RmechTgenbound}:
\begin{align*}
\EE[|\RmechT|] \;&\leq\; 
            \sum_{t\in\Expl} \EE[|\rmecht|] \;+\; 
        \sum_{t\notin\Expl} \EE[|\rmecht|\,|\Et]  \;+\; \sum_{t\notin\Expl} \EE[|\rmecht|\,|\Et] \\
   &\leq\;
            \Vmax K\qT \;+\; 
    \sum_{i=1}^n \sum_{t\notin\Expl} \EE[|\Amit|\,|\Et] \;+\;
    (n-1)\sum_{t\notin\Expl} \EE[|\Bt|\,|\Et] +
    \Vmax\sum_{t\notin\Expl}\PPt(\Etc).
\end{align*}
The claim follows by combining the above bound with~\eqref{eqn:maxregbound}
and then applying Lemma~\ref{lem:qTbound} for a bound on $K\qT$.
\end{proof}

We are now ready to prove Theorem~\ref{thm:value}.

\proofheader{Proof of Theorem~\ref{thm:value}}
First consider the $\estmethod=\estee$ case.
We will use Lemma~\ref{lem:MaxRegretgenbound} to control $\RmaxT$.
We already have an upper bound on $\RT$ from Section~\ref{sec:proofwelfare},
and upper bounds on $\Amit$ and $\Bt$ from Sections~\ref{sec:selleretcproof}
and~\ref{sec:selleroptproof}.
The lower bounds
for $\Amit, \Bt$ are obtained by simply reversing the argument.

First consider, $|\Amit|$.
If $\pricemethod=\pricesel$, we already saw $\Amit \leq 0$.
By using an argument similar to~\eqref{eqn:Amituniformbound}, we obtain
$-\Amit \leq 2\sqrt{2}\sigma\betat (n-1)\Kottmot$.
Similarly, if $\pricemethod=\priceage$, we already saw $\Amit \leq 2\sqrt{2}\sigma\betat (n-1)\Kottmot$.
Moreover,
\[
-\Amit = \Vlbmit(\allocmit) - \Valmi(\allocopt)
\leq \Valmi(\allocmit) - \Valmi(\allocopt) \leq 0.
\]
Therefore,  for both $\pricemethod$ values we  have
$|\Amit| \leq 2\sqrt{2}\sigma\betat (n-1)\Kottmot$.
After an application of Lemma~\ref{lem:subounds}, we obtain,
\begin{align*}
\numberthis
\label{eqn:Amitmodbound}
\sum_i\sum_{t\notin\Expl} \EE[|\Amit|\,|\Et] \leq 3\sqrt{2}\sigma\betat n(n-1)\KotTtt.
\end{align*}
By following a similar argument, we can obtain
$(n-1)\sum_{t\notin\Expl}\EE[|\Bt|\,|\Et] \leq 3\sqrt{2}\sigma\betat n(n-1)\KotTtt$.
The claim follows by combining the above with Lemma~\ref{lem:MaxRegretgenbound},
the bound on the welfare regret in~\eqref{eqn:eewelfareregbound},
and Lemma~\ref{lem:Etcbound} to control $\sum_{t\notin\Expl}\PPt(\Etc)$.

Now consider the $\estmethod=\estopt$ case.
Since there will have been $\qT$ exploration phases in $T$ rounds, we use the expression
in~\eqref{eqn:Amitmodbound} to bound the sum of $\EE[|\Amit||\Et]$ terms.
%
Next, let us turn to the $|\Bt|$ terms in the RHS of
Lemma~\ref{lem:MaxRegretgenbound}.
When $\pricemethod=\priceage$, we already saw in Section~\ref{sec:selleroptproof} that $\Bt \leq 0$.
Moreover, $-\Bt = \Val(\allocopt) - \Vlbt(\alloct) \leq \Vubt(\alloct) - \Vlbt(\alloct)
\leq 2\betat\sum_i\sigmait(\sit)$ under $\Et$.
Similarly, when $\pricemethod=\pricesel$, we already saw $\Bt \leq 2\betat\sum_i\sigmait(\sit)$.
Moreover, $-\Bt = \Val(\allocopt) - \Vubt(\alloct) \leq 0$.
In all cases, we have $|\Bt| \leq 2\betat\sum_i\sigmait(\sit)$ and therefore,
by following the same calculations in~\eqref{eqn:sqrtSTbound},
we have,
\[
(n-1)\sum_{t\notin\Expl} \EE[|\Bt|\,|\Et] \leq 4\sigma\betaT n(n-1)\sqrtST.
\]
The claim follows by combining the above with Lemma~\ref{lem:MaxRegretgenbound},
the bound on the welfare regret in~\eqref{eqn:ucbwelfareregbound},
and Lemma~\ref{lem:Etcbound} to control $\sum_{t\notin\Expl}\PPt(\Etc)$.
\qedwhite

\input{technical_lemmas}

%% file: bounding_peitc.tex
%


\paragraph{Notation \& Definitions:}
Recall that Algorithm~\ref{alg:vcglearn} proceeds in a sequence of brackets.
In our proofs,
$q_t$ will denote the bracket index round $t$ belongs to and
$T_q$ will be the number of rounds completed by $q$ brackets.
Then,
\begin{align*}
\numberthis
\label{eqn:qtdefn}
T_{\qt-1} < t \leq T_{\qt}.
\end{align*}

$\Eit$, defined below, will denote the event that agent $i$'s values are trapped by the lower and
upper confidence bounds at round $t$ when she participates truthfully.
$\Et$ denotes the same for all agents.
Here $\vubit, \vlbit$ are as defined in~\eqref{eqn:learnterms}.  We have:
\begin{align*}
\numberthis
\label{eqn:Etdefn}
\Eit=\Big(\forall\, s\in\Scal,\;\, \vali(s) \in [\vlbit(s), \vubit(s)]\Big),
\hspace{0.6in}
\Et=\bigcap_{i=1}^n \Eit.
\end{align*}

For the outcome $\alloct$ at time $t$, let $\sit = \mapfuni(\alloct)$ be the allocation for
agent $i$.
Hence, for instance, we can write $\Val(\alloct) = \valmech(\alloct) + \sum_{i=1}^n\vali(\alloct)
= \valmech(\alloct) + \sum_{i=1}^n \vali(\sit)$.
We will similarly use the following definitions for
the upper and lower bound on the welfare at time $t$, the functions $\Fit,\Git$
used in the pricing calclulation, and their optimisers.
Some of these terms have been defined before.
\begingroup
\allowdisplaybreaks
\begin{align*}
&\Vubt(\alloc) = \valmech(\alloc) + \sum_{i=1}^n \vubit(\alloct),
\hspace{0.15in}
\alloct = \argmax_{\alloc\in\Alloc}\Vubt(\alloc),
\hspace{0.25in}
\Vlbt(\alloc) = \valmech(\alloc) + \sum_{i=1}^n \vlbit(\alloct),
\\
&\Fit(\alloc) = \valmech(\alloc) + \sum_{j\neq i} \fit(\alloct),
\hspace{0.15in}
\allocmit = \argmax_{\alloc\in\Alloc}\Fit(\alloc),
\numberthis
\label{eqn:welfaredefns}
\\
&
\Gmit(\alloc) = \valmech(\alloc) + \sum_{j\neq i} \git(\alloct),
\hspace{0.25in}
\Gt(\alloc) = \valmech(\alloc) + \sum_{i=1}^n \git(\alloct).
\end{align*}
\endgroup

Next, we define some quantities related to the mean and confidence intervals defined
in~\eqref{eqn:learnterms}.
For brevity, we will denote the unclipped empirical mean in~\eqref{eqn:learnterms} by
$\vbarbarit(s)$. 
Next, we define $\betat, \sigmait$ as shown below.
With this, we can rewrite the upper and lower confidence
bounds in~\eqref{eqn:learnterms} as follows:
\begin{align*}
\numberthis
\label{eqn:betat}
&\betat = \sqrt{5\log(t-qK+1) + 2\log(|\Scal|)},
\hspace{0.35in}
\sigmait(s) =\begin{cases}
    \,0 \hspace{0.1in} &\text{if $i$ plays by bids} \\
    \,\frac{\sigma}{\sqrt{\Nit(s)}} \hspace{0.1in}&\text{otherwise}
    \end{cases},
\\
&\vubit(s) = \vbarit(s) + \betat\sigmait(s),
\hspace{1.10in}
\vlbit(s) = \vbarit(s) - \betat\sigmait(s).
\end{align*}

\subsection{Bounding $\PPt(\Etc)$}
\label{sec:boundingpeitc}

In this section, we control the probability that the upper and lower confidence
bounds do not trap the true values $\{\vali(s)\}_{i,s}$.
Recall that sub-Gaussian random variables satisfy the following concentration
property.
Let $\{X_i\}_{i=1}^n$ be $n$ i.i.d samples from a $\sigma$ sub-Gaussian
distribution and $\overline{X} = \frac{1}{n}\sum_i
X_i$ be its sample mean. Then,
\[
\PP( \overline{X} > \epsilon) \leq e^{\frac{-n\epsilon^2}{2\sigma^2}}, \hspace{0.45in}
\PP( \overline{X} < \epsilon) \leq e^{\frac{-n\epsilon^2}{2\sigma^2}}.
\]

\insertprethmspacing
\begin{lemma}
\label{lem:Eitcbound}
Assume that agent $i$ participates truthfully and let
$\Eit$ be as defined in~\eqref{eqn:Etdefn}.
When \emph{$\estmethod=\estee$},
for $t\notin\Expl$ in bracket $q$, $\PPt(\Eitc) \leq 2(t-qK)^{-\nicefrac{5}{2}}$.
Moreover, for all $T$, $\sum_{t=1, t\notin\Expl}^T\PPt(\Etc) \leq 4$.
When \emph{$\estmethod=\estopt$},
for $t\notin\Expl$ in bracket $q$, $\PPt(\Eitc) \leq 2(t-qK)^{-\nicefrac{3}{2}}$.
Moreover, for all $T$, $\sum_{t=1, t\notin\Expl}^T\PPt(\Etc) \leq 6$.
\end{lemma}
\begin{proof}
If the agent participates by bids truthfully, then $\vubit = \vlbit = \vali$ and the claim
is trivially true.
For agents participating by rewards, we will first prove this for $\estmethod=\estopt$.
Consider the event $\{\vali(s) > \vubit(s)\}$ and recall the
definitions in~\eqref{eqn:learnterms}.
Let $\vbarbarit(s)$ be the unclipped empirical mean in~\eqref{eqn:learnterms}.
Let $\vbbbarit(s) = \max(0, \vbarbarit(s))$ and $\vuuubit(s) = \vbbbarit(s) + \betat\sigmait(s)$.
Since $\vbarit(s) = \min(1, \vbbbarit(s))$, we have $\vuuubit(s) \geq \vubit(s)$.
However, the following calculations show that 
$\PP\left(\vali(s) > \vubit(s)\right) =  \PP(\vali(s) > \vuuubit(s))$.
\begin{align*}
&\PPt\left(\vali(s) > \vubit(s)\right)
\\
&= \PPt\left(\vali(s) > \vubit(s)| \vbarbarit(s) \geq 1\right) \PPt(\vbarbarit(s) \geq 1)
 + \PPt\left(\vali(s) > \vubit(s)| \vbarbarit(s) < 1\right) \PPt(\vbarbarit(s) < 1) \\
&= \PPt\left(\vali(s) > \vuuubit(s)| \vbbbarit(s) \geq 1\right) \PPt(\vbbbarit(s) \geq 1)
 + \PPt\left(\vali(s) > \vuuubit(s)| \vbbbarit(s) < 1\right) \PPt(\vbbbarit(s) < 1) \\
&= \PPt\left(\vali(s) > \vuuubit(s)\right).
\end{align*}
Here, the second step uses two arguments.
First, when $\vbarbarit(s)<1$, then $\vubit(s) = \vuuubit(s)$.
Second, when $\vbarbarit(s)\geq 1$, then $\PPt(\vali(s) > \vubit(s))
= \PPt(\vali(s) > \vuuubit(s)) = 1$ since $\vali(s) \leq 1 < \vubit(s) \leq \vuuubit(s)$.
We can now bound,
\begingroup
\allowdisplaybreaks
\begin{align*}
\PPt(\vali(s) > \vubit(s))
&= \PPt\big(\vali(s) > \max(0, \vbarbarit(s)) + \betat\sigmait(s)\big)
\\
&\leq \PPt\bigg(\vali(s) \,>\, \frac{1}{\Nit(s)}\sum_{\ell=1}^{t-1}\Xil\indfone(\sit = s) +
            \betat\frac{\sigma}{\sqrt{\Nit(s)}} \bigg) \\
&\leq \PPt\bigg(\exists\tau\in\{q,\dots,t-(K-1)q\},\;
        \vali(s) \,>\, \frac{1}{\tau}\sum_{\ell=1}^\tau X'_{i\ell} + \betat\frac{\sigma}{\sqrt{\tau}}
              \bigg) \\
&\leq
\sum_{\tau=q}^{t-qK +q} \PPt\bigg(
        \vali(s) \,>\, \frac{1}{\tau}\sum_{\ell=1}^\tau X'_{i\ell} + \betat\frac{\sigma}{\sqrt{\tau}}
              \bigg)
\;\leq\; 
(t-qK + 1) e^{-\betat^2/2}
\\
&\leq\; \frac{1}{|\Scal|(t-qK+1)^{\nicefrac{3}{2}}}
\end{align*}
\endgroup
In the second step,
if $\vbarbarit(s)$ was clipped below at $0$, then we can replace it with a smaller quantity.
In the third step, we have used the fact that $\Nit(s)$ would take a value
in $\{q,\dots,t-(K-1)q\}$ since there have been $qK$ exploration rounds thus far,
during which we have collected
rewards from agent $i$ for allocation $s$ exactly $q$ times.
$\{X'_{i\ell}\}_{\ell=1}^{\tau}$ denotes the rewards $\Xil$ collected when $\Nit(s)=\tau$.
The fourth step uses a union bound and the fourth step applies the sub-Gaussian condition.
A similar bound can be shown for the event $\{\vali(s) < \vlbit(s)\}$.
The first claim follows
by applying a union bound over these two events and over all $s\in\Scal$.
The second claim follows from the observation
$\sum_{t=1}^\infty t^{-3/2} \leq 1 + \int_1^\infty t^{-3/2} \leq 3$.

Now consider $\estmethod=\estee$.
The calculations above can be repeated, except  $\Nit(s)=\qt$~\eqref{eqn:qtdefn}
deterministically for all $i,t$.
(When $\estmethod=\estopt$, $\Nit(s)$ is random and depends
on the reward realised.)
Therefore, we will not need the sum over $\tau\in\{q,\dots, t-(K-1)q\}$,
resulting in the
bound $e^{-\betat^2/2}\leq \frac{1}{|\Scal|(t-qK)^{\nicefrac{5}{2}}}$.
The second claim follows from 
$\sum_{t} t^{-5/2} \leq  2$.
\end{proof}

\insertprethmspacing
\begin{lemma}
\label{lem:Etcbound}
Assume that all agents participate truthfully and let $\Et$ be as defined in~\eqref{eqn:Etdefn}.
When \emph{$\estmethod=\estee$},
for $t\notin\Expl$ in bracket $q$, $\PPt(\Etc) \leq 2n(t-qK)^{-\nicefrac{5}{2}}$.
Moreover, for all $T$, $\sum_{t=1, t\notin\Expl}^T\PPt(\Etc) \leq 4n$.
When \emph{$\estmethod=\estopt$},
for $t\notin\Expl$ in bracket $q$, $\PPt(\Etc) \leq 2n(t-qK)^{-\nicefrac{3}{2}}$.
Moreover, for all $T$, $\sum_{t=1, t\notin\Expl}^T\PPt(\Etc) \leq 6n$.
\end{lemma}
\begin{proof}
This follows by an application of the union bound over the agents $i\in\{1,\dots,n\}$
on the results of Lemma~\ref{lem:Eitcbound}.
\end{proof}


%% file: truthfulness_proof.tex
\subsection{Proof of Theorem~\ref{thm:truthfulness}}
\label{sec:prooftruthfulness}

We will first prove Theorem~\ref{thm:truthfulness}.
We begin with the following Lemma.
To state it, consider any strategy $\pi$ that agent $i$ may follow when reporting her rewards.
Let $\utilit$ be the utility of the agent on round $t$
when she reports truthfully on rounds $1$ to $t-1$ (recall that the allocation the agent receives on
round $t$ depends on the rewards $\{\Yil\}_{\ell=1}^{t-2}$
she reported on rounds in the first $t-1$ rounds),
let $\utilpiit$ be the utility of the agent when she follows strategy $\pi$ from rounds 1 through
$t-1$,
and let $\utilitmot$ be the utility of agent on round $t$ when she follows $\pi$ on rounds
1 thorough $t-2$ and then switches to truth-telling at the end of round $t-1$.
If participating by bids, this means it will change the bid function, and if participating by rewards,
it means it will replace the reported rewards $\Yil$ for rounds $1,\dots,t-2$ with the true
rewards $\Xil$ and then report truthfully at round $t-1$.
Agent $i$'s allocation at round $t$ when the agent replaces her rewards this way will be different to
the allocation chosen when simply reporting truthfully since  her past
untruthful behaviour will have affected the outcomes chosen by the mechanism in the previous rounds
(this is particularly the case when $\estmethod=\estopt$).
We should also emphasise that this behaviour of replacing the rewards is only for the purposes
of our proof below.
We have the following result.

\insertprethmspacing
\begin{lemma}
\label{lem:decomposition}
Let $\utilit,\utilpiit,\utilitmot$ be as defined above.
Then,
\begin{align}
\label{eqn:decomposition}
\UpiiT - \UiT = \sum_{t=1}^T \utilipit - \utilitmot  + \sum_{t=2}^T \utilitmot - \utilit.
\end{align}
\begin{proof}
The claim follows by adding and subtracting $\sum_{t=1}^T \utilitmot$, rearranging the terms,
and noting that $\utilipptt{0}{1} = \utiliitt{i}{1}$.
\end{proof}
\end{lemma}


When applying the above Lemma,
we will denote the strategy which follows $\pi$ up to round $t-2$ and switches to truth-telling
at the end of round $t-1$ as $\pi^{t-1}$.
We will denote the outcomes at round $t$
when following $\pi$, $\pi^{t-1}$ and truth-telling 
by $\allpit,\allpptt{t-1}{t}$ and $\alloct$ respectively,
and the allocations for agent $i$ by $\sipit,\sipptt{r}{t},\sit$ for respectively;
therefore, $\sipit=\mapfuni(\allpit)$, $\sipptt{r}{t}=\mapfuni(\allpptt{t-1}{t}$,
and $\sit=\mapfuni(\alloct)$.

\subsubsection{Proof of Theorem~\ref{thm:truthfulness}.1}
\label{sec:prooftruthfulnessetc}


We begin with Lemma~\ref{lem:decomposition}.
First, consider the second summation in its RHS, where we claim that each term inside the summation
is $0$.
To see this, note that $\pi^{t-1}$ is also participating truthfully at round $t$.
It has replaced its reported rewards with its true realised rewards in the previous rounds.
The mechanism only uses rewards reported in the exploration rounds to decide outcomes on
the exploitation rounds, and the outcomes in the exploration rounds are chosen independent
of the bids/rewards reported by the agent.
As the outcome and prices in round $t$ will be the same for both policies,
we have $\utilitmot = \utilit$.
(As we will see shortly in Section~\ref{sec:prooftruthfulnessopt}, this will not be
the case when $\estmethod=\estopt$, and the second sum will be non-zero.)

Now turn to the first summation in the RHS of~\ref{lem:decomposition}.
In the remainder of the proof,
$\git$ will denote the appropriate quantity, either $\vlbit$ or $\vubit$ depending
on the value of hyperparameter $\pricemethod$,
for agent $i$ when following $\pi^{t-1}$.
Since, at time $t$, she has switched to being truthful and only rewards from the exploration
phase are used in computing outcomes, 
this will be the same as had she been truthful throughout.
Similarly, let $\Gt$ denote either $\Vlbt$ or $\Vubt$ 
when agent $i$ follows  $\pi^{t-1}$.
Using these, we can write for $t\notin\Expl$,
\begingroup
\allowdisplaybreaks
\begin{align*}
\numberthis
\label{eqn:eetruthcalcs}
\utilipit - \utilitmot \;&=\; \bigg(\vali(\sipit) +
                                \Big(\valmech(\allpit) + \sum_{j\neq i}\gjt(\allpit) \Big) -
                                \max_\alloc\Fmit(\alloc) \bigg) \\
            &\hspace{0.8in} -   \bigg(\vali(\sitmot) +
                                \Big(\valmech(\alltmot) + \sum_{j\neq i}\gjt(\alltmot) \Big) -
                                \max_\alloc\Fmit(\alloc) \bigg), \\
    &=\;  \vali(\sipit) - \vali(\sitmot) + 
                \Big(\valmech(\allpit) + \sum_{j\neq i}\gjt(\allpit) \Big) -
                \Big(\valmech(\alltmot) + \sum_{j\neq i}\gjt(\alltmot) \Big), \\
    &=\;  (\vali(\sipit) - \git(\sipit)) + (\git(\sitmot) - \vali(\sitmot)) \,+\, \\
     &\hspace{0.8in}
       \Big(\underbrace{\valmech(\allpit) + \sum_{i=1}^n\gjt(\allpit) }_{\Gt(\allpit)}\Big) -
      \Big(\underbrace{\valmech(\alltmot) + \sum_{i=1}^n\gjt(\alltmot)}_{\Gt(\alltmot)} \Big),
        \\
    &\leq\;(\vali(\sipit) - \git(\sipit)) + (\git(\sitmot) - \vali(\sitmot)).
\end{align*}
\endgroup
Here, the first step substitutes expressions for $\utilipit,\utilitmot$ from
Fact~\ref{fac:vcgutils}.
The $\max_\alloc\Fmit(\alloc)$ terms are cancelled out in the second step; 
they will be the same for both policies $\pi, \pi^{t-1}$ since
it is computed using the rewards reported by other agents in rounds $1,\dots,t-1$
and hence does not depend on the fact that agent $i$ has switched policies in the current round.
The third step adds and subtracts $\git(\sipit) + \git(\sitmot)$ and observes
that the last two terms are $\Gt(\allpit), \Gt(\alltmot)$, where, recall
$\Gt$ is the appropriate quantity computed \emph{after}
agent $i$ switches to truthful reporting.

To obtain the last step,
recall that $\alloct=\alltmot = \argmax_\alloc \Vubt(\alloc)$ by line~\ref{lin:alloct} of
Algorithm~\ref{alg:vcglearn}.
Moreover,  when $\estmethod=\estee$,
$\vlbit, \vbarit, \vubit$ are vertically shifted functions;
for agents participating by bids, they are identical while for agents participating by rewards,
 we use only one observation per allocation
per agent in each exploration phase.
Therefore,
$\Vlbt, \Vbart, \Vubt$ are also vertically shifted functions and hence
 $\alloct = \alltmot = \argmax\Vubt = \argmax\Vlbt = \argmax\Vbart$.
Therefore, regardless of the value of $\pricemethod$, we have $\Gt(\alltmot) \geq \Gt(\allpit)$.
We emphasise that the above calculations do not use the fact that $\vubjt$ is an upper confidence
bound on $\valj$ for agents $j\neq i$; this may not be true since agent $j$ may not be truthful.
Instead, it is simply treated as a function of rewards reported by agent $j$ in previous rounds.

To complete the proof, we can use the fact that that the $\git$ terms
are computed under truthful reporting from agent $i$.
If the agent participates by bids, then $\git = \vali$ and hence
$\utilipit - \utilitmot \leq 0$ a.s.
Combining this with the fact that the utilities
for all policies are the same during $t\in\Expl$, we have
$\UpiiT - \UiT \leq 0$ a.s..
For an agent participating by rewards,
under $\Eit$,
\begin{align*}
\vali(\sipit) - \git(\sipit) + \git(\sitmot) - \vali(\sitmot)
&= 
\vali(\sipit) - \vubit(\sipit) + \vubit(\sitmot) - \vali(\sitmot)
\numberthis
\label{eqn:visipitbound}
\\
&\leq
2\betat\sigmait(\mapfuni(\allocmit)) \;\leq\;
      2\sqrt{2}\betat \sigma K^{\nicefrac{1}{3}} t^{\nicefrac{-1}{3}}.
\end{align*}
Above, we have used the fact
 the widths of the confidence intervals are all equal.
For the last step, we use the following argument to bound $\sigmait(s)$ for
any $s\in\Scal$.
It uses Lemma~\ref{lem:qTbound}
and the fact that at time $t$, agent $i$ will have experienced all allocations $s\in\Scal$
at least $q_t$ times.
\begin{align*}
\forall\; i\in\{1,\dots, n\},\;t\geq 1,\; s\in\Scal,\quad
\sigmait(s) = \sigma/\sqrt{\Nit(s)} \leq \sigma/\sqrt{q_t} \leq \sqrt{2} K^{1/3} t^{-1/3}.
\numberthis
\label{eqn:sigmaitsbound}
\end{align*}
We will use the bound in~\eqref{eqn:sigmaitsbound} repeatedly in our proofs.

Therefore,~\eqref{eqn:visipitbound}
leads us to $\EE[\utilipit - \utilitmot|\Eit] \leq 2\sqrt{2}\betat \sigma
K^{\nicefrac{1}{3}} t^{\nicefrac{-1}{3}}$
and consequently,
\[
\EE[\UpiiT - \UiT] = \sum_{t} \EE[\utilipit - \utilitmot|\Eit] + \sum_t\PP(\Eit)
\leq 3\sqrt{2}\betaT K^{\nicefrac{1}{3}} T^{\nicefrac{2}{3}} + 4.
\]
The last step uses Lemma~\ref{lem:Eitcbound} to bound $\sum_t\PP(\Eit)$.
The claim follows by substituting for $\betaT$~\eqref{eqn:betat}.
\qedwhite

\subsubsection{Proof of Theorem~\ref{thm:truthfulness}.2}
\label{sec:prooftruthfulnessopt}

The main difference in applying Lemma~\ref{lem:decomposition} in the $\estmethod=\estopt$ case
 is that now the mechanism uses all of the rewards reported by the agents,
and this needs to be accounted for when bounding the two summations.
Unlike in Section~\ref{sec:prooftruthfulnessetc}, we cannot take values such as $\vubit,\vlbit$
to be the same for $\pi^{t-1}$ and truth-telling because now the mechanism is using reported rewards
from all rounds to determine the outcome at round $t$; while we have swapped all false
reports with the true rewards in $\pi^{t-1}$, the outcomes in the rounds outside the exploration
phase will have been different, and therefore so are the rewards realised and the quantities
computed based on the rewards.
Therefore,
in this proof, we will annotate quantities related to strategy $\pi^{t-1}_t$ at time $t$ with a prime.
For example, $\vpubit:\Scal\rightarrow\RR$ (see~\eqref{eqn:learnterms})
will be the upper confidence bounds at time $t$ for agent $i$ when following $\pi^{t-1}$.
On the same note, $\Epit$ denotes the event that agent $i$'s true values fall within
the confidence interval at time $t$ when she follows $\pi^{t-1}$.

For the terms in the first summation in the RHS of Lemma~\ref{lem:decomposition},
by repeating the calculations in~\eqref{eqn:eetruthcalcs}, we obtain
(using our above notation),
\begin{align*}
\utilipit - \utilitmot = (\vali(\sipit) - \gpit(\sipit)) + (\gpit(\sitmot) - \vali(\sitmot)) +
                            \Gpt(\allpit) - \Gpt(\alltmot).
\end{align*}
Recall that $(\gpit, \Gpt)$ denote
either $(\vplbit, \Vplbt)$ or $(\vpubit, \Vpubt)$ as per the value of
$\pricemethod$ being $\pricesel$ or $\priceage$.
They are computed in round $t$ under truthful reporting.
If agent $i$ participates by bids, then $\gpit = \vali$ and hence $\vali(s) - \gpit(s) = 0$ for all
$s$.
If she participates by rewards, then for all choices of $\pricemethod$,
\begin{align*}
\label{eqn:gdiffreasoning}
\numberthis
\EE[(\vali(\sipit) - \gpit(\sipit)) + (\gpit(\sitmot) - \vali(\sitmot)|\Epit] \leq
    2\sqrt{2}\sigma \betat K^{\nicefrac{1}{3}} t^{\nicefrac{-1}{3}}.
\end{align*}
This follows from the observation that when $\pricemethod=\priceage$,
the first term is less than $0$ while the second is less than
$2\betat\sigmait(\sitmot) \leq 2\sqrt{2} \sigma\betat \Kottmot$ by Lemma~\ref{lem:qTbound}
and the fact that there have been $\qt$ exploration phases (see~\eqref{eqn:sigmaitsbound});
a similar argument holds for $\pricemethod=\pricesel$, but with the terms reversed.

Next, we use $\alltmot=\argmax_\alloc\Vpubt(\alloc)$
to bound the difference $\Gpt(\allpit) - \Gpt(\alltmot)$.
When $\pricemethod=\priceage$,
\begin{align*}
&\Gpt(\allpit) - \Gpt(\alltmot) = \Vpubt(\allpit) - \Vpubt(\alltmot) \leq  0.
\end{align*}
When $\pricemethod=\pricesel$, we can use the fact that at round $t\notin\Expl$
all allocations will have been experienced by each agent at least $\qt$ times~\eqref{eqn:qtdefn}
to obtain,
\begin{align*}
\Gpt(\allpit) - \Gpt(\alltmot) &= \Vplbt(\allpit) - \Vplbt(\alltmot) \\
&\leq 
\Vpubt(\allpit) - \Vpubt(\alltmot) + 2\betat\sum_i(\sigmapit(\sitmot) - \sigmapit(\allpit)) \\
&\leq 2\betat\sum_i\sigmapit(\alltmot)
 \leq 2\sigma \betat\sum_i \frac{1}{\sqrt{\qt}}
\leq 2\sqrt{2}\sigma \betat n \Kottmot.
\end{align*}
The last step uses Lemma~\ref{lem:qTbound}.
Now summing over all $t$ and using Lemma~\ref{lem:subounds}, we obtain
\begin{align*}
\label{eqn:optfirstbound}
\numberthis
\sum_{t\notin\Expl} \EE[\utilipit - \utilitmot | \Epit]
\leq
\begin{cases}
3\sqrt{2}\sigma \kappaiv \betaT \KotTtt \hspace{0.2in}&\text{if $\pricemethod=\priceage$}, \\
3\sqrt{2}\sigma (n + \kappaiv) \betaT \KotTtt  \hspace{0.2in}&\text{if $\pricemethod=\pricesel$}.
\end{cases}
\end{align*}

We now move to the second summation in the RHS of Lemma~\ref{lem:decomposition}.
To bound this term, we will use the fact that 
all agents except $i$ are adopting stationary policies.
Therefore,
the rewards reported by any agent $j\neq i$ for any $s\in\Scal$ concentrates around some mean,
and we can apply Lemma~\ref{lem:Eitcbound} for that agent.
For the remainder of this proof,
$\valj(s)$ will denote the mean of this distribution.
(This may not be equal to the true value of agent $j$ for allocation $s$ since she may not be
truthful.)
$\Et,\Ept$ denote the events that $\valj(s)$ falls within
the confidence intervals $(\vlbjt(s),\vubjt(s))$, $(\vplbjt(s),\vpubjt(s))$ respectively
for all agents $j$ at round $t$.
Here, recall, the former interval is obtained for agent $j$ when agent $i$ is being truthful
from the beginning
and the latter when $i$ is following $\pi^{t-1}$.
We now expand each term in the second summation as follows,
\begingroup
\allowdisplaybreaks
\begin{align*}
\numberthis
\label{eqn:opttruthsecondsum}
\utilitmot - \utilit \;&=\; \bigg(\vali(\sitmot) +
                                \Big(\valmech(\alltmot) + \sum_{j\neq i}\gpjt(\alltmot) \Big) -
                                \max_\alloc\Fpmit(\alloc) \bigg) \\
               &\hspace{1.0in}  - \bigg(\vali(\sit) +
                                \Big(\valmech(\alloct) + \sum_{j\neq i}\gjt(\alloct) \Big) -
                                \max_\alloc\Fmit(\alloc) \bigg), \\
    &=\;  \Big(\vali(\sitmot) - \gpit(\sitmot)\Big) + \Big(\git(\sit) - \vali(\sit)\Big) \\
       &\hspace{1in}     \,+\, 
            \Big(\Gpt(\alltmot) -\Gt(\alloct)\Big) + 
        \Big(\max_\alloc \Fmit(\alloc) -  \max_\alloc \Fpmit(\alloc)\big)\Big).
\end{align*}
The first step uses the expressions in Fact~\ref{fac:algoutils}, while the second step
adds and subtracts $\gpit(\sitmot) + \git(\sit)$ and rearranges the terms.
To bound all four terms in~\eqref{eqn:opttruthsecondsum},
we will use that $\git$, $\gpit$, $\Gt$, $\Gpt$, $\Fmit$, $\Fpmit$ are all
computed under truthful reporting from agent $i$,
that all other agents are adopting stationary policies,
 and that  each agent has experienced each allocation at least
$\qt$ times~\eqref{eqn:qtdefn} in round $t$.
The first two terms are $0$ for an agent participating by bids.
If participating by rewards,
via a similar reasoning to that used in~\eqref{eqn:gdiffreasoning},
\[
\EE[(\vali(\sitmot) - \gpit(\sitmot)) + (\git(\sit) - \vali(\sit))|\Epit, \Eit]
\leq 2\sqrt{2} \sigma\betat \Kottmot.
\]
%
%
To bound the third term in~\eqref{eqn:opttruthsecondsum},
observe that $\Vpubt-\Vubt$ is uniformly bounded under $\Et\cap\Ept$.
\begin{align*}
\Vpubt(\alloc) - \Vubt(\alloc)
&= \sum_i (\vpubit(\alloc) - \vubit(\alloc))
= \sum_i (\vpubit(\alloc) - \vali(\alloc)) + \sum_i(\vali(\alloc) - \vubit(\alloc))
\\
&\leq 2\sqrt{2} \sigma\betat n \Kottmot.
\end{align*}
\endgroup
Observing that $\alltmot = \argmax_\alloc\Vpubt(\alloct)$ and
$\alloct = \argmax_\alloc\Vubt(\alloct)$, we use Lemma~\ref{lem:boundmax} to obtain,
\begin{align*}
\Gpt(\alltmot) - \Gt(\alloct)
&= \Vpubt(\alltmot) - \Vubt(\alloct) + \betat\kappag\sum_j(\sigmajt(\alloct) - \sigmapjt(\alltmot))
\\
&\leq (2+\kappag)\sqrt{2} \sigma\betat n \Kottmot.
\end{align*}
Here $\kappag=0$ if $\pricemethod=\priceage$ and $\kappag=2$ if $\pricemethod=\pricesel$.
Above,
the first step rewrites the expression for $\Gt, \Gpt$ in terms of
$\Vubt$,  $\Vpubt$, and $\kappag$.
The second step drops the $\sigmapit(\alltmot)$ terms and bounds the $\sigmait(\alltmot)$ terms
using Lemma~\ref{lem:qTbound}.
To bound the last term, we observe that $\Fpmit-\Fmit$ is uniformly bounded under $\Et\cap\Ept$.
Using a similar reasoning to~\eqref{eqn:gdiffreasoning},
\begin{align*}
\Fpt(\alloc) - \Ft(\alloc) &= \sum_i (\fpit(\alloc) - \fit(\alloc))
= \sum_i (\fpit(\alloc) - \vali(\alloc)) + \sum_i(\vali(\alloc) - \fit(\alloc))
\\
&\leq 2\sqrt{2} \sigma\betat n \Kottmot.
\end{align*}
By Lemma~\ref{lem:boundmax}, we therefore have,
$
\max_\alloc \Fmit(\alloc) -  \max_\alloc \Fpmit(\alloc)
\leq 2\sqrt{2} \sigma\betat n \Kottmot.
$
Summing over all $t$ and using Lemma~\ref{lem:subounds}, we can now bound the second summation
in Lemma~\ref{lem:decomposition}.
\begin{align*}
\label{eqn:optsecondbound}
\numberthis
\sum_{t\notin\Expl} \EE[\utilitmot - \utilit | \Ept,\Et]
\leq
3\sqrt{2}\sigma \big(\kappaiv + n(4+\kappag/2)\big) \betaT \KotTtt.
\end{align*}
Finally,
we can combine the results in~\eqref{eqn:optfirstbound},~\eqref{eqn:optsecondbound} to
obtain
\begingroup
\allowdisplaybreaks
\begin{align*}
\EE[\UpiiT - \UiT] 
    &= \sum_{t\notin\Expl} \EE[\utilipit - \utilit]
    = \sum_{t\notin\Expl} \EE[\utilipit - \utilit|\Et,\Ept]
            + \sum_{t\notin\Expl}\PP(\Etc\cup\Eptc) \\
    &= \sum_{t\notin\Expl} \EE[\utilipit - \utilitmot|\Epit]  +
                    \sum_{t\notin\Expl} \EE[\utilitmot - \utilit|\Et,\Ept] + 12n \\
    &\leq 3\sqrt{2}\sigma \big(2\kappaiv + n(4+\kappag)\big) \betaT \KotTtt + 12n.
\end{align*}
\endgroup
The first step observes that the allocations and prices are the same during
the exploration phase rounds $\Expl$.
The third step uses Lemma~\ref{lem:Etcbound}, although  $\Ept$
now refers to an event when the strategy changes at each step.
The claim follows by substituting for $\betaT$~\eqref{eqn:betat} and then
observing $\kappag\leq 2$ and $\kappaiv\leq 1$.
\qedwhite

%% file: technical_lemmas.tex
\subsection{Some Technical Lemmas}
\label{sec:prooftechnical}

This section states some technical results that were used throughout our proofs.
The following fact, akin to Fact~\ref{fac:vcgutils}, is straightforward to verify.

\insertprethmspacing
\begin{fact}
\label{fac:algoutils}
In round $t$ of Algorithm~\ref{alg:vcglearn},
the agent and seller utilities satisfy the following for the given $\{\fit, \git\}_{i,t}$ choices.
\emph{
\begin{align*}
&\textrm{if $t\in\Expl$,}\hspace{0.1in}
\utilit = \vali(\sit),
\hspace{0.52in}
\textrm{if $t\notin\Expl$,}\hspace{0.1in}
\utilit 
    = \vali(\sit) - \git(\sit) + \Gt(\alloct) -
        \Fmit(\allocmit), \\
&\textrm{if $t\in\Expl$,}\hspace{0.1in}
\utilmecht = \valmech(\alloct),
\hspace{0.5in}
\textrm{if $t\notin\Expl$,}\hspace{0.1in}
\utilmecht = 
            \sum_{i=1}^n\Fmit(\allocmit) - (n-1)\Gt(\alloct).
\end{align*}
}
\end{fact}

The following result bounds the number of brackets
$\qT$~\eqref{eqn:qtdefn} after a given number of rounds $T$.

\insertprethmspacing
\begin{lemma}
\label{lem:qTbound}
Consider Algorithm~\ref{alg:vcglearn} on the $T$\ssth round, where $T > 2K$,
and let $\qT$ denote the current bracket index.
If $T$ is an exploration round, i.e. $T\in\Expl$, then 
$\qT \leq 3 K^{\nicefrac{-2}{3}} T^{\nicefrac{2}{3}}$.
If $T\notin\Expl$, then,
$
\frac{1}{2} K^{\nicefrac{-2}{3}} T^{\nicefrac{2}{3}}
\leq \qT \leq
3 K^{\nicefrac{-2}{3}} T^{\nicefrac{2}{3}}
$.
\end{lemma}
\begin{proof}
For brevity, write $q=\qT$, $c = 5/6$, $d=1/2$.
First let $T\notin\Expl$.
Using the notation in~\eqref{eqn:qtdefn}, we have
\begin{align*}
T_{q-1} + K < T \leq T_{q}, 
\hspace{0.2in}
\text{where,}
\hspace{0.1in}
T_{m} = Km + \sum_{t=1}^m \lfloor cKt^d \rfloor.
\end{align*}
To bound $T_m$, letting $S_m = \sum_{t=1}^m t^d$ and bounding the sum of an increasing function
by an integral we have,
\begin{align*}
\int_0^m t^d\ud t < S_m < \int_1^{m+1} t^d\ud t
\hspace{0.2in}
\implies
\hspace{0.2in}
\frac{m^{d+1}}{d+1} < S_m < \frac{ (m+1)^{d+1} - 1}{d+1}.
\label{eqn:Smbound}
\numberthis
\end{align*}
This leads to the following bounds on $T$,
\begin{align*}
\numberthis\label{eqn:qTlb}
T &\leq T_q \leq qK + \sum_{t=1}^q cKt^d \leq qK + \frac{cK}{d+1}\big((q+1)^{d+1} -1\big)
\\
 &\leq q^{d+1}K + \frac{cK}{d+1}(2q)^{d+1}
 \leq c_1 q^{\nicefrac{3}{2}}K, \\
\numberthis\label{eqn:qTub}
T &\geq T_{q-1} + K \geq qK + \sum_{t=1}^{q-1} (cKt^d - 1) \geq
    qK - (q-1) + \frac{cK}{d+1}(q-1)^{d+1}
\\
    &\geq  \frac{cK}{d+1}\left(\frac{q}{2}\right)^{d+1}
    = c_2q^{\nicefrac{3}{2}}K.
\end{align*}
In~\eqref{eqn:qTlb}, we have used the upper bound in~\eqref{eqn:Smbound} with $m=q$, and the
facts $q\leq q^{\nicefrac{3}{2}}$, $q+1 \leq 2q$.
In~\eqref{eqn:qTub}, we have used the lower bound in~\eqref{eqn:Smbound} with $m=q-1$, and
the facts $qK>q-1$, $q-1\geq q/2$; the last inequality holds when $q\geq2$ which is true
when $T\geq 2K$.
Now, 
by substituting the values for $c$ and $d$, we have $c_1 = 1 + 10\sqrt{2}/9$ and
$c_2 = 5\sqrt{2}/36$.
Thus,
\[
q \leq \left(\frac{T}{c_2K}\right)^{\nicefrac{2}{3}}
\leq 3\frac{T^{\nicefrac{2}{3}}}{K^{\nicefrac{2}{3}}},
\hspace{0.4in}
q \geq \left(\frac{T}{c_1K}\right)^{\nicefrac{2}{3}}
\geq \frac{1}{2}\frac{T^{\nicefrac{2}{3}}}{K^{\nicefrac{2}{3}}}.
\]
This proves the result for $T\notin\Expl$.
If $T\in\Expl$, by noting that $T\geq T_{q-1}$, we can  repeat the calculations
in~\eqref{eqn:qTub} to obtain the same bound.
\end{proof}

The following two results were used repeatedly throughout our proofs.

\insertprethmspacing
\begin{lemma}
\label{lem:boundmax}
Let $f_1, f_2:\Xcal\rightarrow\RR$ for some finite set $\Xcal$ such that $f_1(x)-f_2(x)\leq\epsilon$
for all $x\in\Xcal$ and a given $\epsilon\geq 0$.
Then $\max f_1 - \max f_2 \leq \epsilon$.
\begin{proof}
Let $x_i=\argmax f_i$. Then, $f_1(x_1) - f_2(x_2) \leq f_2(x_1) - f_2(x_2) + \epsilon \leq \epsilon$.
\end{proof}
\end{lemma}

\insertprethmspacing
\begin{lemma}
\label{lem:subounds}
\;$\sum_{t=1}^n t^{-\nicefrac{1}{2}} \leq 2n^{\nicefrac{1}{2}}$,
\quad $\sum_{t=1}^n t^{-\nicefrac{1}{3}} \leq \frac{3}{2}n^{\nicefrac{2}{3}}$.
\begin{proof}
By bounding the summation of a decreasing function by an integral we have for $r\in[0, 1]$,
$\sum_{t=1}^n t^{-r} \leq 1 + \int_1^n t^{-r}\ud t \leq 1 + \frac{n^{1-r}}{1-r} -\frac{1}{1-r}
\leq  \frac{n^{1-r}}{1-r}$.
Setting $r=1/2, 1/3$ yields the results.
\end{proof}
\end{lemma}

%% file: conclusion.tex
\section{Conclusion}
\label{sec:conclusion}

%
We have studied mechanism design in settings where agents may not know their values,
but can experience an allocation and report back a realised reward.
The goal of the mechanism is to learn the values of the users while
simultaneously finding the optimal outcome and satisfying game-theoretic
desiderata such as individual rationality and truthfulness.
We established a lower bound on the \valueregret{} for this problem, and presented an algorithm that
essentially achieves this rate.
The proposed framework allows a practitioner to control trade-offs between various
properties that they might be interested in, such as agent and seller regrets, individual
rationality, and truthfulness.
We conclude with two avenues for future work.


First, we have assumed that we
can maximise the upper confidence bound~\eqref{eqn:Vubt} exactly.
In many settings, this might be computationally prohibitive, and we might only be able to
obtain an approximate solution.
It will be instructive to study which of the desiderata carry through in this case.
If we have an $\alpha$-approximate solver ($\alpha<1$), it is straightforward to show
that sublinear welfare regret~\eqref{eqn:regret} is possible under truthful reporting,
if it is defined as $\RT = \alpha T \Val(\allocopt) - \sum_t\Val(\alloct)$.
However, bounding the agent and seller regrets requires more careful analysis
as their utility depends on the near-optimal outcome chosen by the solver.
Implications on truthfulness are even less clear,
especially as an agent can be strategic over multiple rounds.

Second, the
lower bound in Theorem~\ref{thm:lowerbound} only captures one of the two key difficulties in
this problem, namely pricing calculation for agent/seller trade-offs;
the other being truthfulness.
It is worth studying the implications of even  asymptotic truthfulness on learning.
While our algorithm is optimal with respect to the lower bound for the \valueregret,
in some applications, it is not necessary to minimise all three regret terms in $\RmaxT$.
For instance, the PaaS setting in Example~\ref{exm:paas} could occur within an organisation,
where the service provider is one team providing a service to other (agent) teams.
In such cases, the seller regret is not a meaningful quantity.
In this setting, it is possible to obtain $\sqrt{T}$ regret for both the welfare and the agents
\emph{if} the agents report truthfully: at \emph{all} time steps,
select the outcome which maximises the upper
confidence bound on the welfare and choose a favourable pricing
scheme to the agents, such as the one obtained by setting $\pricemethod=\priceage$.
However, this is not a truthful mechanism.
In situations like this, we believe that truthfulness will prevent obtaining $\sqrt{T}$ regret.
For instance, \citet{babaioff2013multi} and \citet{devanur2009price} show that $\Ttwth$ regret is unavoidable
for deterministic truthful algorithms in their online advertising problem.
Their proof relies heavily on a necessary and sufficient condition for truthfulness in
single-parameter auctions where agents submit
bids~\citep{myerson1981optimal,archer2001truthful}.
Extensions of this condition to multi-parameter auctions exist~\citep{rochet1987necessary},
but only in instances where agents submit bids ahead of time.
This characterisation does not apply in our problem where the agent does not know her
value and reports a reward at the end of the round.